\documentclass{article}
\usepackage{fullpage}

\usepackage[utf8]{inputenc} %
\usepackage[T1]{fontenc}    %
\usepackage{url}            %
\usepackage{booktabs}       %
\usepackage{amsfonts}       %
\usepackage{nicefrac}       %
\usepackage{microtype}      %
\usepackage{xcolor}         %

\usepackage[titletoc]{appendix}
\usepackage{amsmath}
\usepackage{amsthm,amsfonts,amssymb}
\usepackage{xspace}
\usepackage{bbm}
\usepackage{graphicx}
\usepackage{bbm}
\usepackage{etoolbox}%
\usepackage{enumitem}

\usepackage[noend]{algorithmic}
\usepackage[ruled,vlined]{algorithm2e}
\usepackage{comment}
\usepackage[numbers]{natbib}
\usepackage{float}

\usepackage{accents}

\usepackage{graphicx} %
\usepackage{booktabs} %

\usepackage{caption}

\usepackage{xcolor}

\usepackage{tikz}
\usepackage{enumitem}
\usepackage{subcaption}
\usepackage{mathtools}

\usepackage{mathrsfs}

\usepackage{makecell}

\makeatletter
\DeclareFontFamily{OMX}{MnSymbolE}{}
\DeclareSymbolFont{MnLargeSymbols}{OMX}{MnSymbolE}{m}{n}
\SetSymbolFont{MnLargeSymbols}{bold}{OMX}{MnSymbolE}{b}{n}
\DeclareFontShape{OMX}{MnSymbolE}{m}{n}{
    <-6>  MnSymbolE5
   <6-7>  MnSymbolE6
   <7-8>  MnSymbolE7
   <8-9>  MnSymbolE8
   <9-10> MnSymbolE9
  <10-12> MnSymbolE10
  <12->   MnSymbolE12
}{}
\DeclareFontShape{OMX}{MnSymbolE}{b}{n}{
    <-6>  MnSymbolE-Bold5
   <6-7>  MnSymbolE-Bold6
   <7-8>  MnSymbolE-Bold7
   <8-9>  MnSymbolE-Bold8
   <9-10> MnSymbolE-Bold9
  <10-12> MnSymbolE-Bold10
  <12->   MnSymbolE-Bold12
}{}

\let\llangle\@undefined
\let\rrangle\@undefined
\DeclareMathDelimiter{\llangle}{\mathopen}%
                     {MnLargeSymbols}{'164}{MnLargeSymbols}{'164}
\DeclareMathDelimiter{\rrangle}{\mathclose}%
                     {MnLargeSymbols}{'171}{MnLargeSymbols}{'171}
\makeatother

\newtheorem{theorem}{Theorem}[section]
\newtheorem{lemma}{Lemma}[section]
\newtheorem{corollary}{Corollary}[section]
\newtheorem{definition}{Definition}[section]
\newtheorem{remark}{Remark}[section]

\newtheorem{claim}{Claim}[section]

\newtheorem{ass}{Assumption}[section]
\newtheorem{fact}{Fact}[section]

\usepackage{xcolor}

\newcommand{\ind}{\mathbbm{1}}
\newcommand{\R}{\mathbb{R}}
\newcommand{\E}{\mathop{\mathbb{E}}}
\newcommand{\argmin}{\mathop{\mathrm{argmin}}}

\newcommand{\cX}{\mathcal{X}}
\newcommand{\cY}{\mathcal{Y}}

\newcommand{\cP}{\mathcal{P}}

\newcommand{\cG}{\mathcal{G}}
\newcommand{\cF}{\mathcal{F}}
\newcommand{\cS}{\mathcal{S}}

\newcommand{\round}{\text{Round}}
\newcommand{\regret}{\text{regret}}

\newcommand{\cA}{\mathcal{A}}
\newcommand{\tcA}{\tilde{\cA}}

\newcommand{\Ber}{\text{Ber}}
\newcommand{\Tau}{\mathcal{T}}

\newcommand{\cL}{\mathcal{L}}

\newcommand{\cD}{\mathcal{D}}

\newcommand{\hp}{\hat{p}}
\newcommand{\hy}{\hat{y}}

\newcommand{\bK}{\overline{K}}
\newcommand{\bsK}{\overline{sK}}
\newcommand{\barf}{\bar{f}}

\newcommand{\sign}{\text{sign}} 
\newcommand{\ts}{\mathbf{s}}
\newcommand{\tx}{\mathbf{x}}

\newcommand{\thp}{\mathbf{\hat{p}}}
\newcommand{\ti}{\mathbf{i}}
\newcommand{\tj}{\mathbf{j}}
\newcommand{\tz}{\mathbf{z}}
\newcommand{\cV}{\mathcal{V}}

\newcommand{\forecast}{\mathscr{F}}

\newcommand{\adv}{\text{Adv}}
\newcommand{\forecastSwap}{\forecast_{\text{ContextualSwap}}}

\newcommand{\cO}{\mathcal{O}}
\newcommand{\scO}{s\mathcal{O}}

\newcommand{\cZ}{\mathcal{Z}}

\newcommand{\pb}{PB}
\newcommand{\cPq}{\cP_{\text{quantile}}}

\newcommand{\cLcvx}{\cL_{\text{convex}}}
\newcommand{\cLv}{\cL_{\mathcal{V}}}
\newcommand{\tcLv}{\tilde{\cL}_{\mathcal{V}}}

\newcommand{\tpi}{\tilde{\pi}}
\newcommand{\fat}{\text{fat}}

\newcommand{\lips}{\gamma}
\newcommand{\el}{\text{Trunc}}

\usepackage{thmtools} 
\usepackage{thm-restate}

\usepackage{amsmath}
\usepackage{pgfplots}
\usepackage{authblk}

\usepackage[hyperindex,breaklinks]{hyperref}

\title{Oracle Efficient Online Multicalibration and Omniprediction}

 \author[1]{Sumegha Garg\thanks{Supported by NSF CCF-1763299 and the Simons Collaboration on the Theory of Algorithmic Fairness}}
 \author[1]{Christopher Jung\thanks{Supported by NSF Award IIS-1908774 and the Simons Collaboration on the Theory of Algorithmic Fairness.}}
 \author[1]{Omer Reingold\thanks{Supported by Simons Foundation Investigators Award 689988 and the Sloan Foundation Investigator Grant 2020-13941}}
 \author[2]{Aaron Roth\thanks{Supported in part by the Simons Collaboration on the Theory of Algorithmic Fairness, and NSF grants FAI-2147212 and CCF-2217062.}}
 \affil[1]{Stanford University}
 \affil[2]{University of Pennsylvania}

\begin{document}

\maketitle

\begin{abstract}
A recent line of work has shown a surprising connection between multicalibration, a multi-group fairness notion, and omniprediction, a learning paradigm that provides simultaneous loss minimization guarantees for a large family of loss functions \cite{GopalanKRSW22, GopalanHKRW23, gopalan2023characterizing, globus2023multicalibration}. Prior work studies omniprediction in the batch setting. We initiate the study of omniprediction in the online adversarial setting. Although there exist algorithms for obtaining notions of multicalibration in the online adversarial setting \cite{gupta2021onlinevalid}, unlike batch algorithms, they work only for small finite classes of benchmark functions $\cF$, because they require enumerating every function $f \in \cF$ at every round. In contrast, omniprediction is most interesting for learning theoretic \emph{hypothesis classes} $\cF$, which are generally continuously (or at least exponentially) large.

We develop a new online multicalibration algorithm that is well defined for infinite benchmark classes $\cF$ (e.g. the set of all linear functions), and is oracle efficient --- i.e. for any class $\cF$, the algorithm has the form of an efficient reduction to a no-regret learning algorithm for $\cF$.  The result is the first efficient online omnipredictor --- an oracle efficient prediction algorithm that can be used to simultaneously obtain no regret guarantees to all  Lipschitz convex loss functions. For the class $\cF$ of linear functions, we show how to make our algorithm efficient in the worst case (i.e. the ``oracle'' that we need is itself efficient even in the worst case). We show how our results extend beyond mean multicalibration to quantile multicalibration, with applications to oracle efficient multivalid conformal prediction.  Finally, we show upper and lower bounds on the extent to which our rates can be improved: our oracle efficient algorithm actually promises a stronger guarantee called ``swap-omniprediction'', and we prove a lower bound showing that obtaining $O(\sqrt{T})$ bounds for swap-omniprediction is impossible in the online setting. On the other hand, we give a (non-oracle efficient) algorithm which can obtain the optimal $O(\sqrt{T})$ omniprediction bounds without going through multicalibration, giving an information theoretic separation between these two solution concepts. We leave the problem of obtaining $O(\sqrt{T})$ omniprediction bounds in an oracle efficient manner as our main open problem.

\end{abstract}

\thispagestyle{empty} \setcounter{page}{0}
\clearpage

 \tableofcontents
 \thispagestyle{empty} \setcounter{page}{0}
 \clearpage

\section{Introduction}
\paragraph{Omniprediction} An omnipredictor, informally, is a prediction algorithm that predicts a sufficient statistic for optimizing a wide range of loss functions in a way that is competitive with some benchmark class of models $\cF$. Gopalan et al. introduced the concept of  omnipredictors in \cite{GopalanKRSW22} and showed that a regression model that is appropriately \emph{muticalibrated} with respect to some benchmark class of models $\cF$ is an omnipredictor with respect to all Lipschitz convex loss functions and the benchmark class $\cF$. In other words, for any Lipschitz convex loss function $\ell$, the regression function can be efficiently and locally post-processed (in a way specific to $\ell$) to obtain loss that is competitive with the best model $f \in \cF$ for that loss function. Here ``locally'' means that the post-processing depends only on the prediction made for a particular point $x$, independently of the rest of the model, and so can be done much more efficiently than training a new model specifically for the loss function $\ell$. They also gave an oracle efficient algorithm for training such a regression function to satisfy the requisite notion of multicalibration in the batch setting, that efficiently reduced to agnostic learning over the class of models $\cF$. Again in the batch setting, Globus-Harris et al. \cite{globus2023multicalibration} gave a simplified algorithm that efficiently reduced to squared error regression over $\cF$. 

\paragraph{Multicalibration} Informally, multicalibration as introduced by \cite{hebert2018multicalibration}, is a requirement on a regression function that it be statistically unbiased conditional both on its own prediction and on membership in any one of a large collection of intersecting subsets of the data space. The notion of a ``subset'' was generalized to a class of real-valued functions $\cF$ by \cite{kim2019multiaccuracy,GopalanKRSW22}, which yields the following requirement. A model $m:\cX\rightarrow \mathbb{R}$ is (exactly) multicalibrated on a distribution $\cD$ over labelled examples $\cX \times \{0,1\}$ with respect to a class of real valued functions $\cF$ if for every $v$ in the range of $m$, and for every $f \in \cF$:
$$\E_{(x,y) \sim \cD}[f(x)(y-v)|m(x) = v] = 0$$
 
There are a variety of ways to define approximate multicalibration that we discuss in Section \ref{sec:prelim}
. In the batch setting, obtaining approximate multicalibration with respect to $\cF$ is reducible to solving learning problems over $\cF$ --- either classification problems \cite{hebert2018multicalibration,GopalanKRSW22} or squared error regression problems \cite{globus2023multicalibration}.

Gupta et al. \cite{gupta2021onlinevalid} and subsequent work \cite{noarov2021online,bastani2022practical} have given algorithms to obtain multicalibration in the online adversarial setting, in which there is no data distribution, and the goal is to make predictions $m_t(x_t)$ at each round $t$ such that the resulting predictions are multicalibrated in hindsight on the empirical distribution on realized examples $(x_t,y_t)$ that can be chosen adaptively by an adversary. But compared to their batch analogues, these online algorithms have a major deficiency that is especially relevant to the omniprediction application: They are defined only for finite classes $\cF$ and have per-round running time that is linear in $|\cF|$. That is, unlike batch algorithms for multicalibration, they are not reductions to learning problems over $\cF$, and hence are only efficient for small finite classes $\cF$. In the context of omniprediction, $\cF$ is a benchmark class of models (e.g. linear functions or more complex models), and so is typically continuously large (and any reasonable discretization would be exponentially large). These prior algorithms also promise approximate multicalibration in the $L_\infty$ metric (rather than in the $L_1$ metric in which  multicalibration error bounds translate to ominprediction bounds, and so prior work does not give online omniprediction at optimal rates.) Thus prior work on online multicalibration has at best limited implications for online omniprediction.  

\subsection{Our Results and Techniques}
\paragraph{Oracle Efficient Online Multicalibration and Omniprediction} Our main contribution is to define the problem of online omniprediction, and to give an oracle efficient algorithm that satisfies it. Informally speaking, an online ominpredictor is an algorithm that ingests a sequence of adversarially chosen contexts $x_t$, and makes a sequence of forecasts $\hat p_t$ about the unknown binary label $y_t$. This single stream of forecasts $\hat p_t$ will be simultaneously used by a large collection of learners who are each concerned with optimizing a different loss function $\ell$. For each $\ell$, the $\ell$-learner will transform the prediction $\hat p^t$ into an action $a^\ell_t = k^\ell(\hat p_t)$ via some one-dimensional post-processing function $k^\ell:[0,1]\rightarrow \mathbb{R}$. The $\ell$-learner then experiences loss $\ell(y_t,k^\ell(\hat p_t))$. The forecasting algorithm is an omnipredictor with respect to the class of loss functions $\ell$ and a benchmark class $\cF$ if simultaneously, all of the $\ell$-learners have diminishing regret to the best model in $\cF$ (which might be different for each $\ell$-learner):
$$\sum_{t=1}^T \ell(y^t,k^\ell(\hat p_t)) \leq \min_{f \in \cF} \sum_{t=1}^T \ell(y_t,f(x_t)) + o(T)$$

To solve this problem, we give a new online multicalibration algorithm that is oracle efficient ---  it is an efficient reduction to the problem of no-regret learning over $\cF$ with respect to the squared error loss. Our algorithm relies on a characterization of multicalibration recently shown in the batch setting independently by Globus-Harris et al. \cite{globus2023multicalibration} and Gopalan et al \cite{gopalan2023characterizing}. Informally, the characterization states that a model $m$ is multicalibrated with respect to $\cF$ if and only if it satisfies the following ``swap-regret'' like condition with respect to $\cF$ for all $v$ in the range of $m$\footnote{The connection holds as stated only if $\cF$ is closed under affine transformations --- but this is the case for many natural classes $\cF$ like linear and polynomial functions, regression trees, etc.}:
$$\E_{(x,y)\sim \cD}[(m(x)-y)^2 | m(x) = v] \leq \min_{f \in \cF}\E_{(x,y)\sim \cD}[(f(x)-y)^2 | m(x)=v]$$

We generalize this characterization to the sequential prediction setting, and use it to give an algorithm for sequential multicalibration via reduction to the standard no-regret learning problem of obtaining no (external) regret with respect to the squared error loss to the best model in $\cF$. That is to say, whenever we have an efficient algorithm for solving online squared error regression over $\cF$, we also have an efficient algorithm to obtain online  multicalibration with respect to $\cF$. Efficient algorithms for online squared error regression exist for the class of linear functions \cite{foster1991prediction,vovk2001competitive,azoury2001relative}, and online gradient descent empirically solves this problem quite well over parametric families of models even when the problem is hard in the worst case \cite{foster2018practical,foster2020beyond}. Our algorithmic reduction is based on a tight reduction from external to swap regret recently given by \cite{ito2020tight}, which is itself based on an earlier reduction of Blum and Mansour \cite{blum2007external}. When instantiated with the online regression algorithm for linear functions of \cite{azoury2001relative}, our algorithm efficiently obtains $L_2$-multicalibration at a rate of $O(T^{-1/4})$ and is an online omnipredictor for the class of all Lipschitz convex loss functions with a regret bound of $O(T^{-1/8})$. 

\begin{theorem}[Informal, See Corollary~\ref{cor:swap-forcast-bounds-multical} and \ref{cor:swap-omni-oracle-rate}] Given an online  regression oracle for $\cF$, we have an efficient algorithm that achieves sublinear multicalibration error with respect to $\cF$ and omniprediction with respect to all Lipschitz convex loss functions, against any adaptive adversary.
\end{theorem}
\begin{theorem}[Informal, See Corollary~\ref{cor:linearcalibration} and \ref{cor:linearomniprediction}] For the set of linear functions $\cF_{\text{lin}}$ with bounded norm, we have an efficient algorithm that achieves $O(T^{-1/4})$ multicalibration error and $O(T^{-1/8})$-omniprediction with respect to all Lipschitz convex loss functions, against any adaptive adversary.
\end{theorem}

\paragraph{Separations and Lower Bounds}
To what extent can we hope for online omnipredictors with better rates? Is $O(T^{-1/2})$ possible? We give several answers. First, by combining a recent characterization of 1 dimensional proper scoring rules for binary outcomes \cite{kleinberg2023u} with the ``Adversary Moves First'' multi-objective no regret framework introduced by \cite{noarov2021online}, we give a (non-oracle-efficient) algorithm for obtaining online omniprediction at a rate of $O(T^{-1/2})$ for finite boolean classes $\cF$. Second, we note that our oracle efficient online multicalibration algorithm actually obtains the stronger guarantee of \emph{swap}-multicalibration---which implies the correspondingly stronger guarantee of swap omniprediction (see Section \ref{sec:prelim} for definitions). Using a lower bound of \cite{qiao2021stronger} on obtainable online calibration bounds in the $L_1$ metric, we show that no algorithm (whether or not it is oracle efficient) can obtain $O(T^{-1/2})$ rates for swap omniprediction, which results in an information-theoretic separation between the best obtainainable rates for omniprediction and swap-omniprediction in the online setting. Whether or not there exist \emph{oracle efficient} algorithms that can obtain omniprediction at the optimal $O(T^{-1/2})$ rate is our main open problem.

\begin{theorem}[Informal, See Corollary~\ref{cor:swap-lb}] 
There exists an adaptive adversary such that against any forecaster, the swap-omniprediction rate is $\Omega(T^{-0.472})$ in expectation. 
\end{theorem}
\begin{theorem}[Informal, See Corollary~\ref{cor:monotone-losses} and \ref{cor:swap-lb}] 
For any finite family of boolean predictors $\cF: \cX \to \{0,1\}$, we have an algorithm that achieves omniprediction at 
 a rate of $O(T^{-1/2})$ against all bounded bi-monotone loss functions.
\end{theorem}

\paragraph{Extension to Quantile Multicalibration and Multivalid Conformal Prediction}
Finally, we observe that our core techniques are not specific to (mean) multicalibration and online squared error regression. In Section \ref{sec:conformal-extension} and Appendix \ref{app:conformal-extension}, we show how to extend our oracle efficient algorithm for online mean multicalibration to an oracle efficient algorithm for online \emph{quantile} multicalibration \cite{gupta2021onlinevalid,bastani2022practical}. Just as our algorithm for mean multicalibration over $\cF$ reduces to an online regression oracle for the squared error loss over $\cF$, it is possible to derive an algorithm for online quantile multicalibration via reduction to an online regression oracle for pinball loss over $\cF$.  An application of oracle efficient quantile multicalibration is an oracle efficient online multivalid conformal prediction algorithm. Multivalid conformal prediction, introduced and studied in \cite{gupta2021onlinevalid,bastani2022practical,jung2022batch} gives a method for attaching prediction sets to arbitrary black-box prediction models that cover the true labels with some target probability (say 95\%). The coverage guarantees must hold not just marginally, but also conditionally in membership in an arbitrary collection of intersecting groups specified by functions $\cF$. \cite{bastani2022practical} gave a simple algorithm for obtaining multivalid coverage in the online setting, with per-round running time scaling linearly with $|\cF|$: we give the first oracle efficient algorithm for multivalid conformal prediction.

\subsection{Additional Related Work}
The modern framing of multicalibration was introduced by Hebert-Johnson et al \cite{hebert2018multicalibration}. Similar ideas (without the concern for computational efficiency) date back to Dawid \cite{dawid1985calibration}. Dawid proposed  a notion of computable calibration as a foundation for empirical probability, that required sequential predictions to be calibrated not just overall, but also on every computable subsequence; he showed that any two methods for producing an infinite sequence of computably calibrated forecasts must agree with each other on all but a finite subsequence. 

The existence of an algorithm capable of producing calibrated forecasts in an online adversarial setting was first established by Foster and Vohra \cite{foster1998asymptotic}, which started the large  literature in economics on ``expert testing''---see \cite{olszewski2015calibration} for a survey. We highlight several results from this literature. Lehrer shows the existence of an online forecasting rule that can guarantee calibration on all computable subsequences defined independently of the forecasters predictions (which is the analogue of \emph{multi-accuracy}, as defined in \cite{hebert2018multicalibration} and \cite{kim2019multiaccuracy}). Sandroni et al. \cite{sandroni2003calibration} extend this result to show the existence of an online forecasting rule that can guarantee calibration on all computable subsequences, even those that can depend on the forecaster's predictions, which allows it to capture multicalibration. Sandroni gave a very general result showing that \emph{any} test of a sequential forecaster (not just multicalibration style tests) that could be passed on every distribution by correct probability forecasts could also be passed in a sequential adversarial setting \cite{sandroni2003reproducible}. All of these results are proven using minimax theorems, and hence are not constructive or computational. In fact, Fortnow and Vohra show that  (subject to cryptographic assumptions), there do exist tests that can be passed by nature but cannot be passed in an adversarial setting by any polynomial time forecaster \cite{fortnow2009complexity}.

Foster and Kakade \cite{foster2006calibration} give the first constructive multicalibration style algorithm that we are aware of: they give an algorithm that can produce ``smoothly\footnote{They achieve a notion of calibration against weighted subsequences, in which the weighting must be a smooth function of the forecaster's prediction. Thus they do not achieve the standard notion of calibration we study here, in which calibration must hold conditional on the exact value of the forecasters prediction.}'' multicalibrated predictions against in a sequential setting against an adversary with respect to some benchmark class $\cF$, with both running time and calibration error scaling polynomially with $|\cF|$. Gupta et al. \cite{gupta2021onlinevalid} give an online algorithm for obtaining multicalibration in the sequential prediction setting for a collection of benchmarks $\cF$ that has calibration error scaling only logarithmically with $|\cF|$ --- but still has running time scaling linearly with $|\cF|$. They also give online sequential algorithms for variants of multicalibration generalized from means to moments (c.f. \cite{jung2021moment}) and quantiles (c.f. \cite{bastani2022practical,jung2022batch}). We remark that obtaining multicalibration in the online adversarial setting is a strictly harder problem than obtaining multicalibration in the batch setting: online multicalibration algorithms can be generically converted into batch multicalibration algorithms using e.g. the online-to-batch reduction presented in \cite{gupta2021onlinevalid}, but the reverse is not  true. Kleinberg et al. \cite{kleinberg2023u} recently studied ``U-Calibration'', which can be viewed as a $0$-dimensional  analogue of omniprediction in which predictions must be made without any available features $x$ --- and hence need only to compete with constant benchmark functions.  They gave algorithms that obtain $O(\sqrt{T})$ rates in the online setting; we use one of their structural characterizations of proper scoring rules to give (non-oracle-efficient) $O(\sqrt{T})$ rates for omniprediction. 

Several recent papers \cite{GopalanKRSW22,kim2022universal} have found surprising applications of multicalibration. In particular, Gopalan et al. \cite{GopalanKRSW22} defined omniprediction and showed that approximately multicalibrated predictors are omnipredictors, and Kim et al. \cite{kim2022universal} showed that approximate multicalibration with respect to a class of functions $\cF$ encoding likelihood ratios between distributions implies a kind of out-of-distribution generalization. Gopalan et al. \cite{GopalanKRSW22} gave an algorithm based on agnostic boosting for classification for obtaining a sufficiently strong notion of multicalibration in the batch setting via reduction to agnostic learning, and Globus-Harris et al. \cite{globus2023multicalibration} gave a characterization of this form of multicalibration and its connection to squared error accuracy guarantees in terms of boosting for regression, and gave an algorithm (in the batch setting) via reduction to squared error regression. Gopalan, Kim, and Reingold \cite{gopalan2023characterizing} defined stronger notions of multicalibration and omniprediction called \emph{swap multicalibration} and \emph{swap omniprediction}, and showed an equivalence between them. We make use of a connection proved in both \cite{globus2023multicalibration} and \cite{gopalan2023characterizing} connecting multicalibration to a contextual notion of (squared error) swap regret. 

Multicalibration is concerned with mean estimations: analogous notions of multicalibration for other distributional properties have been studied as well: Jung et al. \cite{jung2021moment} define and give algorithms for moment multicalibration, and \cite{gupta2021onlinevalid,bastani2022practical,jung2022batch} define and give algorithms for quantile multicalibration with applications to conformal prediction. Recently Noarov and Roth \cite{noarov2023scope} gave a complete characterization of which distributional properties (multi)calibrated predictors can be learned for, and which they cannot: Multicalibration is possible for a property if and only if the property is the minimizer of some regression loss function. We use this characterization to give oracle efficient algorithms for online quantile multicalibration, by reduction to no regret algorithms for the corresponding regression loss.  We point the reader to \cite{rothuncertain} for an introductory treatment of much of this work. 

Our focus is on oracle efficient algorithms for an online learning problem by reduction to squared-error online regression algorithms; there has been a similar recent focus in the contextual bandits literature \cite{foster2018practical,foster2020beyond} based on the observation that online regression is often a practically solvable problem (even in settings in which it is hard in the worst case). A more ambitious goal would be to reduce to batch learning oracles --- but it is not yet understood when it is possible to reduce from online learning to batch learning in an oracle efficient way even to obtain standard external regret guarantees. It is known that this is possible under certain restrictive structural conditions --- e.g. when the benchmark class has a small ``separator set'' \cite{syrgkanis2016efficient,dudik2020oracle} --- but it is also known that it is not possible to do this in general, even when the benchmark class is online learnable \cite{hazan2016computational}. By reducing to online learning oracles (as \cite{foster2018practical} and \cite{foster2020beyond} do), we circumvent this difficulty.

There are various notions of regret in the online learning literature that optimize for several objectives simultaneously: for example, adaptive regret \cite{adaptiveregretvovk,adaptiveregret2} seeks to simultaneously minimize regret over all contiguous subsequences, regret in the sleeping experts setting seeks to minimize regret to each expert in the subsequence of rounds for which it is ``active'' \cite{sleeping,sleeping2,kleinberg2010sleepingexperts}, internal and swap regret seek to minimize regret on subsequences on which the learner played each action \cite{foster1999regret,blum2007external}, multigroup regret seeks to minimize regret on subsequences corresponding to datapoints in particular groups \cite{multigroup1,multigroup2} and all of these notions can be subsumed into general notions of subsequence regret which seeks to minimize regret over all of an arbitrary collection of subsequences that can be defined as a function of the predictions made by the algorithm \cite{phiregret,noarov2021online}. All of these notions seek to minimize the same loss function on different (sub)sequences of examples: in contrast, our goal is to simultaneously minimize \emph{different} loss functions on the same sequence of examples. We make use of techniques developed in the context of swap regret \cite{blum2007external,ito2020tight}. 

\section{Preliminaries}
\label{sec:prelim}
\subsection{Notation}
Let $\cX$ denote the feature domain and $\cY$ the label domain: we focus on the binary label setting $\cY = \{0,1\}$. Let $f:\cX \to \R$ denote some predictor and $\cF$ to denote a family of predictors. We write
\[
    \cF_B = \{f \in \cF: f(x)^2 \le B\quad \forall x \in \cX\}
\]
to denote the set of predictors in $\cF$ whose squared value is at most $B$. Let $\ell: \cY \times \R \to \R^{\ge 0}$ be a loss function that takes in the true label $y \in \cY$ and an action $\hat{y} \in \R$, and  write $\cL$ to denote a family of loss functions. We write $\Ber(p)$ to denote the Bernoulli distribution with parameter $p$. Given a positive integer $T \in \mathbb{N}^{>0}$, we write $[T]$ to denote $\{1, \dots, T\}$. We overload the notation and write $[\frac{1}{m}]$ for some positive integer $m \in \mathbb{N}^{>0}$ to denote a discretization of the unit interval $\{0, \frac{1}{m}, \dots, \frac{m-1}{m}, 1\}$. We write $\Delta A$ to denote the probability simplex over the elements in A: for example, we write $\Delta[m]= \{ q \in [0,1]^m: \sum_{j=1}^m q_j \le 1 \}$ to denote the simplex over $[m]$. Given a vector or a list $v$, we write $v[i]$ to denote the $i$th coordinate or element of $v$. 

We use a bold variable to refer to a random variable and non-bold variable to refer to its realization: for example, we sample realization $x$ from its random variable $\tx$.

\subsection{Setting}
Fix a class of loss functions $\cL$. For each loss function $\ell \in \cL$, we have a corresponding $\ell$-learner who is equipped with a post-processing function $k^{\ell}(\hp)$ that determines which action the $\ell$-learner will choose given a forecast $\hp \in \cP$ where $\cP=[\frac{1}{m}]$ denotes the finite set of possible forecasts.
\begin{definition}[Post-Processing]\label{def:post-process}
Given some loss function $\ell$, a post-processing function $k^\ell:\cP \to [0,1]$ chooses the optimal action $a$ according to the belief that distribution over $y$ is $\Ber(\hp)$:
\[
    k^\ell(\hp) = \argmin_{a \in [0, 1]} \E_{y \sim \Ber(\hp)}[\ell(y, a)] = \argmin_{a} \left(1-\hp \right) \ell(0, a) + \hp \ell(1, a).
\]
\end{definition}

Online prediction proceeds in rounds which we index by $t$ for a finite horizon $T$. In each round, an interaction between a forecaster, adversary, and learner proceeds as follows. In each round $t \in [T]$,
\begin{enumerate}
    \item The adversary chooses the feature vector $x_t \in \cX$ and $p_t$ with which the label $y_t$ is chosen to be a positive label: i.e. $y_t \sim Ber(p_t)$. 
    \item Upon observing the feature vector $x_t$, the forecaster makes a forecast $\hp_t \in \cP$ \emph{randomly}. We sometimes refer to $\cP$ as the level sets of the forecaster. 
    \item For each loss $\ell \in \cL$, $\ell$-learner uses the post-processing function $k^\ell: \cP \to \R$ to choose its action $a^\ell_t := k^\ell(\hp_t)$ and suffers loss $\ell(y_t, k^\ell(\hp_t))$.
    \item The realized label $y_t$ is revealed to the forecaster. 
\end{enumerate}

The forecaster $\forecast$'s interaction with the adversary from round $1$ to $T$ results in a \emph{history} $\psi_{1:T} = \{(x_t, y_t, \hp_t)\}_{t=1}^T$. We write $\psi_{1:t-1} \circ (x_t, y_t, \hp_t)$ to denote a concatenation of $(x_t, y_t, \hp_t)$ to the previous history $\psi_{1:t-1}$. We denote the domain of history between the forecaster and the adversary as $\Psi^{*}$. 

In order to allow the forecaster to keep track of some internal state from round to round, we write $\theta_t$ to denote its additional internal state at round $t$. We write $\Theta$ to denote the domain of the forecaster's internal state. And we write the forecaster's internal \emph{transcript} as $\pi_{1:T} = \{(x_t, y_t, \hp_t, \theta_t)\}_{t=1}^T$ and its domain as $\Pi^*$. 

Forecaster $\mathscr{F}: \Pi^* \times \cX \to \Delta\cP$ is a mapping from a tuple of (internal transcript, internal state, new feature vector) to a distribution over possible forecasts. Similarly, Adversary $\adv: \Psi^* \to (\cX \times \Delta\cY)$ is a mapping from a history to a feature vector and a distribution over $\cY$. 

For any $p \in \cP$ and $y \in \cY$, we write
\begin{align*}
    S(\pi_{1:T}, p) &= \{t \in [T]: \hp_t = p\} \\
    S(\pi_{1:T}, p, y) &= \{t \in [T], \hp_t = p, y_t = y\}
\end{align*}
to denote all the days on which the forecast was $p \in \cP$ and all the days on which the forecast was $p$ and the realized label was $y \in \cY$. When the transcript $\pi_{1:T}$ is clear from the context, we write $S(p)$ and $S(p,y)$. Similarly, we write 
\begin{align*}
    n(\pi_{1:T}, p) &= |S(\pi_{1:T}, p)| &&\text{the number of rounds the forecast was $p$}\\
    \mu(\pi_{1:T}, p) &= \frac{1}{|S(\pi_{1:T}, p)|}\sum_{t \in S(\pi_{1:T}, p)} y_t&&\text{the empirical average of $y$'s over $S(p)$}\\
    \barf(\pi_{1:T}, p) &= \frac{1}{|S(\pi_{1:T}, p)|}\sum_{t \in S(\pi_{1:T}, p)} f(x_t)&&\text{the empirical average of $f(x)$'s over $S(p)$}\\
    \barf(\pi_{1:T}, p, y) &= \frac{1}{|S(\pi_{1:T}, p, y)|}\sum_{t \in S(\pi_{1:T}, p, y)} f(x_t)&&\text{the empirical average of $f(x)$'s over $S(p, y)$}\\
\end{align*}
As before when the transcript is clear from the context, we write $n(p), \mu(p), \barf(p)$, and $\barf(p, y)$.

\subsection{Omniprediction}

If the forecast $\hp_t$ exactly matches the true underlying $p_t$ in each round $t \in [T]$, each $\ell$-learner would experience no regret to every comparison strategy $f:\cX\rightarrow [0,1]$ because it would be  minimizing the loss pointwise in each round: more formally, if $\hp_t = p_t$, then for any predictor $f: \cX \to [0,1]$ and loss function $\ell$,  
\[
    \sum_{t=1}^T \E_{y \sim \Ber(p_t)}[\ell(y_t, k^\ell(\hp_t))] \le \sum_{t=1}^T \E_{y \sim \Ber(p_t)}[\ell(y_t, f(x_t))]. 
\]

Ensuring that no $\ell$-learner has regret to \emph{any} comparison strategy $f: \cX \to [0,1]$ and loss function $\ell$ is too much to hope for. Instead we restrict our attention to a family of benchmark predictors $\cF$ and a family of loss functions $\cL$. Given $\cL$ and $\cF$, we want the forecaster to make forecasts such every $\ell$-learner for $\ell \in \cL$ can use the above post-processing function $k^\ell$ to achieve low regret with respect to any comparison strategy $f \in \cF$. 

We consider two flavors of regret: (1) one in which we compare our performance to that of a single comparison strategy $f$ fixed throughout the entire horizon $T$ (essentially external regret) and (2) one where our comparison is to a comparison function that can depend on which forecast is made at each round --- i.e.   a baseline $f_p$ for every forecast $p \in \cP$. This is a notion of swap regret. Swap omniprediction was first defined in \cite{gopalan2023characterizing}.

\begin{definition}[Omnipredictor]
For any transcript $\pi_{1:T}$ that is generated by forecaster $\forecast$ and adversary $\adv$, the forecaster's swap-omniprediction regret with respect to $\{f_{p}\}_{p \in \cP} \in \cF^m$ and loss functions $\{\ell_p\}_{p \in \cP} \in \cL^m$ is defined as
\begin{align*}
    \scO(\pi_{1:T}, \{\ell_p\}_{p \in \cP}, \{f_p\}_{p \in \cP}) = \sum_{p \in \cP} \frac{n(\pi_{1:T}, p)}{T} \left(\frac{1}{n(\pi_{1:T}, p)}\sum_{t \in S(p)} \ell_p(y_t, k_{\ell_p}(\hp_t)) - \frac{1}{n(\pi_{1:T}, p)}\sum_{t \in S(p)} \ell_p(y_t, f_p(x_t)) \right).
\end{align*}

Similarly, we define omniprediction regret as 
\begin{align*}
     \cO(\pi_{\pi_{1:T}}, \ell, f) = \scO(\pi_{1:T}, \{\ell_p\}_{p\in\cP}, \{f_p\}_{p \in \cP}) = \frac{1}{T}\left(\sum_{t=1}^T \ell(y_t, k^\ell(\hp_t)) - \sum_{t=1}^T \ell(y_t, f(x_t))\right)
\end{align*}
where $f_p = f$ and $\ell_p = \ell$ for each $p \in \cP$.

Finally, we use the following notations:
\begin{align*}
    \noalign{\text{Swap-Omniprediction regret with respect to ($\{\ell_p\}_{p \in \cP}, \cF$) and ($\cL, \cF$)}}\\
    \scO(\pi_{1:T}, \{\ell_p\}_{p \in \cP}, \cF) &= \max_{\{f_p\}_{p \in \cP} \in \cF^m } \scO(\pi_{1:T}, \{\ell_p\}_{p \in \cP}, \{f_p\}_{p \in \cP})\\
    \scO(\pi_{1:T}, \cL, \cF) &= \max_{\{\ell_p\}_{p \in \cP} \in \cL^m, \{f_p\}_{p \in \cP} \in \cF^m} \scO(\pi_{1:T}, \{\ell_p\}_{p \in \cP}, \{f_p\}_{p \in \cP})\\
    \noalign{\text{Omniprediction regret with respect to ($\ell, \cF$) and $(\cL, \cF)$}}
    \cO(\pi_{1:T}, \ell, \cF) &= \max_{f \in \cF} \cO(\pi_{1:T}, \ell, f) \quad \cO(\pi_{1:T}, \cL, \cF) = \max_{\ell \in \cL, f \in \cF} \cO(\pi_{1:T}, \ell, f).
\end{align*}
\end{definition}

In words, for each forecast $p \in \cP$, we consider the subsequence of rounds $S(p)$ in which the forecast was $p$ and compare the $\ell$-learner's accumulated loss over these rounds to what the $\ell$-learner could have obtained instead by choosing actions using the benchmark predictor $f_p$. We allow this comparison to be made against a different predictor $f_p$ and loss function $\ell_p$ for each possible forecast $p \in \cP$. Swap omniprediction regret corresponds to the worst case regret as measured over the choice of comparison predictors $\{f_p\}_{p \in \cP}$ and loss functions $\{\ell_p\}_{p \in \cP}$ of the regret averaged over forecasts $p \in \cP$. The non-swap version corresponds to requiring that the comparison benchmark is constrained to use the same $f$ and the same $\ell$ for every forecast $p \in \cP$. Since swap omniprediction regret compares to a richer class of baseline predictors and loss functions than the non-swap version, we have $\cO(\pi_{1:T}, \cL, \cF) \le \scO(\pi_{1:T}, \cL, \cF)$.

\subsection{Multicalibration}
\cite{GopalanKRSW22, gopalan2023characterizing, GopalanHKRW23} have shown a quite close connection between omniprediction and multicalibration, which was originally introduced as a multi-group fairness concept by \cite{hebert2018multicalibration}. Since its introduction, various definitions of approximate multicalibration have been studied, but unfortunately, all of these variants have been referred to only as ``multicalibration'' leading to some confusion in the literature.

To bring clarity to the landscape of approximate  multicalibration definitions, we define various multicalibration error measures and give them distinguishing names, following the exposition in \cite{rothuncertain}. All of them are defined with respect to a class of predictors $\cF$ for which we make the following assumption.
\begin{ass}
\label{ass:all-one}
    We assume $\cF$ contains the constant function $I$: $I(x) = 1$ for all $x \in \cX$. 
\end{ass}

Just as how we define two versions of omniprediction (a standard and a ``swap'' variant), we define standard and ``swap'' variants of multicalibration (swap multicalibration was first defined in \cite{gopalan2023characterizing} in the batch setting). Informally, swap multicalibration allows the function $f$ with which calibration is being measured with respect to, to vary with the forecast $p \in \cP$ of the forecaster.   
\begin{definition}[Multicalibration Errors (of Various Flavors)]
For any transcript $\pi_{1:T} = \{(x_t, y_t, \hp_t)\}_{t=1}^T$ generated by some forecaster $\forecast$ and adversary $\adv$, we define the forecaster's multicalibration error with respect to $f \in \cF$ for forecast $p \in P$ as 
\begin{align*}
  K(\pi_{1:T}, p, f) = \frac{1}{n(\pi_{1:T}, p)}\left(\sum_{t \in S(\pi_{1:T},p)} f(x_t) \cdot (y_t - \hp_t) \right)
\end{align*}
if $n(\pi_{1:T}, p) \ge 1$ and 0 otherwise.

We define $L_1$-, $L_2$-, and $L_\infty$-swap-multicalibration error with respect to $\{f_p\}_{p \in \cP} \in \cF^m$ respectively as
\begin{align*}
    \bsK_1(\pi_{1:T}, \{f_p\}_{p \in \cP}) &= \sum_{p \in \cP} \frac{n(\pi_{1:T}, p)}{T} \left|K(\pi_{1:T}, p, f_p)\right|  \\
    \bsK_2(\pi_{1:T}, \{f_p\}_{p \in \cP}) &= \sum_{p \in \cP} \frac{n(\pi_{1:T}, p)}{T} \left(K(\pi_{1:T}, p, f_p)\right)^2  \\
    \bsK_\infty(\pi_{1:T}, \{f_p\}_{p \in \cP}) &= \max_{p \in \cP} \frac{n(\pi_{1:T}, p)}{T} |K(\pi_{1:T}, p, f_p)|.
\end{align*}

We define the non-swap versions of $L_1$-, $L_2$-, and $L_\infty$-multicalibration error as follows, in which the functions $f_p$ must be the same function $f$ for all $p$: 
\begin{align*}
    \bK_1(\pi_{1:T}, f) &=  \bsK_1(\pi_{1:T}, \{f_p\}_{p \in \cP}) \\
    \bK_2(\pi_{1:T}, f) &= \bsK_2(\pi_{1:T}, \{f_p\}_{p \in \cP})  \\
    \bK_\infty(\pi_{1:T}, f) &= \bsK_\infty(\pi_{1:T}, \{f_p\}_{p \in \cP})
\end{align*}
where $f_p = f$ for each $p \in \cP$.

Finally, we define multicalibration errors with respect to a family of predictors $\cF$ respectively as 
\begin{align*}
    \bsK_1(\pi_{1:T}, \cF) &= \max_{\{f_p\}_{p \in \cP} \in \cF^m} \bsK_1(\pi_{1:T}, \{f_p\}_{p \in \cP})\quad\text{and}\quad \bK_1(\pi_{1:T}, \cF) = \max_{f \in \cF} \bK_1(\pi_{1:T}, f) \\
    \bsK_2(\pi_{1:T}, \cF) &= \max_{\{f_p\}_{p \in \cP} \in \cF^m} \bsK_2(\pi_{1:T}, \{f_p\}_{p \in \cP})\quad\text{and}\quad \bK_2(\pi_{1:T}, \cF) = \max_{f \in \cF} \bK_2(\pi_{1:T}, f) \\
    \bsK_\infty(\pi_{1:T}, \cF) &= \max_{\{f_p\}_{p \in \cP} \in \cF^m} \bsK_\infty(\pi_{1:T}, \{f_p\}_{p \in \cP}) \quad\text{and}\quad \bK_\infty(\pi_{1:T}, \cF) = \max_{f \in \cF} \bK_\infty(\pi_{1:T}, f).
\end{align*}

\end{definition}

As with omniprediction regret, the swap variant always upper-bounds the non-swap variant: e.g. $\bK_1(\pi_{1:T}, \cF) \le \bsK_1(\pi_{1:T}, \cF)$ and so forth. The following relationship between $L_2$ and $L_1$-(swap-)multicalibration error will be useful for us:
\begin{restatable}{lemma}{lemmulticalrelationship}
\label{lem:multical-relationship}    
\begin{align*}
 \bsK_1(\pi_{1:T}, \cF) \le \sqrt{\bsK_2(\pi_{1:T}, \cF)} \quad\text{and}\quad
\bK_1(\pi_{1:T}, \cF) \le \sqrt{\bK_2(\pi_{1:T}, \cF)}.
\end{align*}
\end{restatable}

\section{Oracle-efficient $L_2$-Swap-Multicalibration}
\label{sec:oracle-efficient-multical-alg}

In this section, we derive our main technical result: an oracle efficient algorithm for obtaining $L_2$-swap-multicalibration with respect to a class of functions $\cF$. By ``oracle efficient'', we mean that our algorithm is an efficient reduction to an assumed algorithm (an ``oracle'') for obtaining the standard notion of external regret (with respect to the squared error loss function) to the benchmark class $\cF$. For example, for $\cF$ consisting of all linear functions, there are existing efficient algorithms that obtain regret scaling only logarithmically with $T$ \cite{foster1991prediction,vovk2001competitive,azoury2001relative}. Even for classes $\cF$ for which no regret algorithms for worst-case regret bounds are not known (or do not exist), heuristics like online gradient descent often succeed in practice. 

At a high level, our approach will proceed as follows:
\begin{enumerate}
    \item First we define a notion of contextual swap regret with respect to the squared error loss: informally, a predictor has contextual swap regret relative to a benchmark class of functions $\cF$ if \emph{restricted to any one of its level-sets}, the predictor has higher squared error than the best function in $\cF$, on the sequence of features and labels realized in hindsight. 
    \item Next, we give an algorithm for obtaining diminishing contextual swap regret for any benchmark class $\cF$ by reducing to algorithms for obtaining diminishing external regret to $\cF$. The reduction maintains a copy of the external-regret algorithm for each level-set of the predictor, and advances only one of the simulated copies at each round. We adapt techniques from \cite{ito2020tight,blum2007external} which have been used for similar reductions from  (non contextual) swap regret to external regret. 
    \item Finally, we adapt a characterization theorem of \cite{globus2023multicalibration} and \cite{gopalan2023characterizing} from the batch to the online setting, which relates contextual swap regret to $L_2$-swap multicalibration. 
\end{enumerate}

In Section \ref{subsec:oracle} we define the online squared loss regression oracles that we reduce to, and observe that discretizing their predictions results in only a small increase in their regret bounds. In Section \ref{subsec:no-contextual-swap-regret} we give our main algorithmic construction for obtaining diminishing contextual swap regret. In Section \ref{subsec:swapMC} we show how bounds on diminishing contextual swap regret imply bounds on $L_2$-swap-multicalibration. These will be key ingredients for us in Section \ref{sec:omnipredictor}, in which we observe that online $L_2$-swap-multicalibration implies online $L_1$-swap multicalibration, which in turn implies online (swap-)omniprediction. All missing proofs in this section can be found in Appendix~\ref{app:miss-sec-oracle-efficient}.

\subsection{Online Regression Oracles and Modifications}
\label{subsec:oracle}
In this section, we define online regression oracles, which are online prediction algorithms with the guarantee that they have diminishing squared error regret with respect to a class of predictors $\cF$ in adversarial settings.

\begin{definition}[Online Squared Loss Regression Oracle]
    In each round $t \in [T]$, an online squared loss regression oracle $\cA: (\cX \times \cY)^* \times \cX \to [0,1])$ maps a sequence of (feature label) pairs $\{(x_\tau, y_\tau)\}_{\tau=1}^{t-1}$ and some new feature vector $x_t$ to a \emph{single} number $\hy_t = \cA(\{(x_{\tau}, y_\tau)\}_{\tau=1}^{t-1}, x_t)$. Its regret with respect to $f \in \cF$ is \[
        \regret(\{(x_t, y_t, \hy_t)\}_{t=1}^T, f) := \sum_{t=1}^T (\hy_t - y_t)^2 - \sum_{t=1}^T (f(x_t) - y_t)^2. 
    \]
    An oracle $\cA$ has a regret guarantee $r_{\cA}(T, \cF)$, if it can guarantee against any adaptively chosen sequence $\{(x_t,y_t)\}_{t=1}^T$ (i.e. $(x_t, y_t)$ can be chosen as a function of $\{(x_\tau, y_\tau, \hy_\tau)\}_{\tau=1}^{t-1}$):
    \[
    \max_{f \in \cF} \regret(\{(x_t, y_t, \hy_t)\}_{t=1}^T, f) \le r_\cA(T, \cF).
    \]
\end{definition}

\begin{remark}
\label{rem:convex-deterministic}
We can assume with out loss of generality that an online regression oracle $\cA$ is deterministic because both its action space and squared loss are convex. More formally, if an oracle $\cA$ were to output a distribution over the unit interval $q_t \in \Delta([0,1])$ in each round $t$, it is always better (incurs smaller squared error) to instead deterministicly output the expected prediction $\E_{\hy_t \sim q_t}[\hy_t]$. This follows from Jensen's inequality:
\begin{align*}
    \sum_{t=1}^T \left(\E_{\hy_t \sim q_t}[\hy_t] - y_t\right)^2  
 \le \sum_{t=1}^T \E_{\hy_t \sim q_t}\left[(\hy_t - y_t)^2\right] 
\end{align*}
\end{remark}

In order to produce predictions that are contained within the finite forecasting space $\cP$, we project and round the output of the original oracle $\cA$ to $[\frac{1}{m}]$. Projecting it to be within $[0,1]$ doesn't hurt at all because $y_t \in \{0,1\}$. Also, because the squared error is Lipschitz, this does not increase the regret of the oracle by much. We write $\cA^m: (\cX \times \cY)^* \to \left[\frac{1}{m}\right]$ to denote an oracle that rounds the output of $\cA$ to the nearest multiple of $\frac{1}{m}$: \[
\cA^m(\{(x_\tau, y_\tau)\}_{\tau=1}^{t-1}, x_t) = \round\left(\cA(\{(x_\tau, y_\tau)\}_{\tau=1}^{t-1}, x_t), \frac{1}{m}\right)
\]
where $\round(x, \frac{1}{m}) = \argmin_{v \in [\frac{1}{m}]} |x - v|$.

\begin{restatable}{lemma}{lemregretbyrounding}
\label{lem:regret-by-rounding}
Suppose $m \ge 1$. For any $\cA$, the regret guarantee of its rounded version $\cA^m$ can be bounded as follows: 
\[
    r_{\cA^m}(T, \cF) \le r_{\cA}(T, \cF) + \frac{3T}{m}.
\]
\end{restatable}

As noted in Remark~\ref{rem:convex-deterministic}, we have so far assumed without loss of generality that $\cA^m$ outputs a deterministic prediction $\hy_t \in [\frac{1}{m}]$. However, in Section~\ref{subsec:no-contextual-swap-regret} we will we need a randomized oracle $\tcA: (\cX \times \cY
)^* \times \cX \to \Delta([\frac{1}{m}])$ that maps a sequence of (feature, label) pairs and a new feature vector to a \emph{distribution} over $[\frac{1}{m}]$ that has full support. We use tilde to distinguish the randomized oracle $\tcA$ that we will construct from deterministic oracles $\cA$ and $\cA^m$.

In order to construct a randomized oracle that makes full-support predictions from a deterministic oracle $\cA^m$, we select the output of $\cA^m$ with probability $1-\frac{1}{T}$ and otherwise choose a prediction from $[\frac{1}{m}]$ uniformly at random. More formally, we define $\tcA^m$ to be the randomized oracle such that for every $j \in [m]$,
\[
    \Pr\left[\tilde{\cA}^m(\{x_\tau, y_\tau\}_{\tau=1}^{t-1}, x_t) = \frac{j}{m}\right] = \frac{1}{T} \cdot \frac{1}{m} + \left(1-\frac{1}{T}\right)\cdot \ind\left[\cA^m\left(\{x_\tau, y_\tau\}_{\tau=1}^{t-1}, x_t\right)  = \frac{j}{m} \right].
\]

The expected regret of $\tcA^m$ can be bounded in the following manner:
\begin{restatable}{lemma}{lemroundedrandomoracle}
\label{lem:rounded-random-oracle}
Suppose $m \ge 1$. For any sequence $\{(x_t, y_t)\}_{t=1}^T$ that is chosen adversarially and adaptively against $\tcA^m$, the expected regret of $\tcA^m$ over its own randomness can be bounded as follows: for any $f \in \cF$, it results in $\pi_{1:T}$ such that
\[
    \sum_{t=1}^T \E_{\hy'_t \sim \tcA^m(\{(x_\tau, y_\tau)\}_{\tau=1}^{t-1}, x_t)}[(y'_t - \hy_t) - (y'_t - f(x_t))^2 | \pi_{1:t-1}] \le r_\cA(T, \cF) + \frac{3T}{m} + 1
\]
\end{restatable}

\subsection{Contextual Swap Regret}
\label{subsec:no-contextual-swap-regret}
By analogy to \emph{swap regret} \citep{blum2007external}, we define a new notion of regret called \emph{contextual swap regret}. Informally, swap regret in a $k$-expert learning problem is the requirement that the learner have no regret on any of the subsequences on which they played each of the $k$ experts. By analogy, we define contextual swap regret as the requirement that the learner have no squared-error regret to any model in $\cF$, not just overall, but also on each of the subsequences on which the learner played any of the $m$ values defining its levelsets.

\begin{definition}[Contextual Swap Regret]
    Given a transcript $\pi_{1:T} \in \Psi^*$ generated by forecaster $\forecast: \Pi^* \times \cX \to \Delta \R$, its contextual swap regret with respect to $\{f_p\}_{p \in \cP}$ is
    \begin{align*}
        \sum_{p \in \cP} \sum_{t \in S(\pi_{1:T}, p)} (\hp_t - y_t)^2 - (f_p(x_t) - y_t)^2.
    \end{align*}
\end{definition}
Observe that contextual swap regret is a strict generalization of the standard notion of swap regret: we recover the standard notion of swap regret if $\cF$ consists of the constant functions $c_p$ for each $p \in \cP$ where $c_p(x) = p$ for all $x \in \cX$. We remark that our notion of contextual swap regret  is the online analogue of the definition of swap agnostic learning in \cite{gopalan2023characterizing} and what is referred to as the ``swap-regret'' like condition in \cite{globus2023multicalibration}---both of which are in the batch setting.

To design an oracle efficient algorithm for minimizing contextual swap regret, we adapt ideas from \cite{blum2007external} and \cite{ito2020tight}, which design reductions from the standard notion of swap reget to external regret in the expert learning setting. Before describing how to construct forecaster $\forecast_{\text{ContextualSwap}}$ in Algorithm~\ref{alg:externalToSwapReduction}, we introduce more notation and describe the algorithm at a high level. We instantiate $m$ oracles  $\{\tcA^m_i\}_{i \in [m]}$ where each $\tcA^m_i: (\cX \times \cY)^* \times \cX \to \Delta[m]$ is the randomized oracle that we show how to construct given some online squared loss regression oracle $\cA$ in Section~\ref{subsec:oracle}. At any round $t \in [T]$, the sequence of (feature, label) pairs given to each oracle $\tcA^m_i$ may be different. Hence, we write $s_{t}^i \in (\cX \times \cY)^*$ to denote the sequence of (feature, label) pairs that is fed into the $i$th oracle $\tcA^m_i$ in round $t$. It will sometimes be convenient for us to describe the oracle as outputting the weights over $\cP$ that correspond to the distribution it is sampling from (note that by our construction in Section~\ref{subsec:oracle}, these weights are explicitly defined, so this is without loss of generality): we write $q^i_t = \tcA^{m}_i(s^t_i, x_t)$ where
\[
    q^i_t[j] = \Pr\left[\tcA^{m}_i(s^t_i, x_t) = \frac{j}{m}\right]
\]
for every $j \in [\frac{1}{m}]$ to denote the mixing weights over possible forecasts $\cP=[\frac{1}{m}]$ for the $i$th oracle in round $t$.

\begin{algorithm}[H]
\begin{algorithmic}[1]
\STATE Initialize the randomized oracles $\tilde{\cA}^m_i$ for each $i \in [m]$ (by doing so, we implicitly initialize $\cA^m_i$).
\STATE Initialize the sequence $s^i_t = \{\}$ for each $i \in [m]$.
\FOR{$t=1, \dots, T$}
    \STATE Forecaster $\forecast_{\text{ContextualSwap}}$ observes the feature vector $x_t$
    \STATE $q^t_i = \tilde{\cA}^m_i(s^i_t, x_t)$ for each $i \in [m]$.
    \STATE Form $a_t \in \Delta([m])$ such that \begin{align}
        a_t = \sum_{i=1}^m q^i_t a_t[i].\label{eqn:stat-dist}\end{align}
    \STATE Let $\ti_t$ be distributed according to $a_t$ and sample $i_t \sim \ti_t$ to decide which oracle to use. 
    \STATE Let $\tj_t$ be distributed according to $q^{i_t}_t$, and sample $j_t \sim \tj_t$.
    \STATE Forecaster $\forecastSwap$ forecasts $\hp_t = \frac{j_t}{m}$.
    \STATE Forecaster $\forecastSwap$ observes $y_t$.
    \STATE Update the sequence fed into $i_t$th oracle: $s^{i_t}_{t+1} = s^{i_t}_t \cup \{(x_t, y_t)\}$, and set $s^{i}_{t+1} = s^i_{t}$ for all other $i \neq i_t$.
    \STATE Adversary $\adv$ chooses $(x_{t+1}, y_{t+1})$.
\ENDFOR
\end{algorithmic}
\caption{$\forecast_\text{ContextualSwap}(\cF, \cA, m, T)$}
\label{alg:externalToSwapReduction}
\end{algorithm}

The reason why we need to use the randomized version of the oracle we construct---which mixes in a uniform distribution over $\cP$---rather than the deterministic oracle we start with, is to ensure that there exists a distribution $a_t$ that satisfies equation \eqref{eqn:stat-dist}. In that equation, $a_t$ may be thought of as a stationary distribution for the Markov chain whose transition probabilities are described by $q^i_t$'s: the probability of going from state $i$ to $j$ is $q^i_t[j]$. Mixing in the uniform distribution ensures that $q^i_t[j] > 0$ for all $i, j \in [m]$, which in turn guarantees that there exists a unique stationary distribution $a_t$ in every round $t \in [T]$---thus making sure that equation \eqref{eqn:stat-dist} has a solution.   

\begin{fact}[\cite{levin2017markov}]
For any Markov chain where the transition probability from state $i$ to $j$ is non-zero $q^i[j]>0$ for every $i, j \in [m]$ (i.e. the resulting graph representation is strongly connected), there exists a unique stationary distribution $a \in \Delta([m])$ such that 
\[
    a = \sum_{i=1}^m q^i a[i].
\]
\end{fact}

Note that the additional internal state that the forecaster is keeping around is $\theta_t = (i_t, j_t)$ in this case. In other words, the transcript looks like \[
\pi_{1:t} = \{(x_\tau, y_\tau, \hp_t, (i_\tau, j_\tau))\}_{\tau=1}^{t}.
\]

Adapting the argument from \cite{ito2020tight} directly here would show that the contextual swap regret of forecaster $\forecastSwap$ is bounded in expectation over $\{i_t \sim \ti_t, j_t \sim \tj_t(i_t)\}_{t=1}^T$. More specifically, against any adversary $\adv$, the following quantity 
\begin{align}
    &\sup_{\{f_j\}_{j \in [m]} \in \cF^m_B}\E_{\{(i_t \sim \ti_t, j_t \sim \tj_t(i_t))\}_{t=1}^T}\left[\sum_{t=1}^T (j_t/m - y_t)^2 - (f_{j_t}(x_t) - y_t)^2\right] \nonumber \\
    &= \sup_{\{f_j\}_{j \in [m]} \in \cF^m_B}\sum_{t=1}^T \E_{\{(i_\tau \sim \ti_\tau, j_\tau \sim \tj_t(i_\tau))\}_{\tau=1}^{t}}\left[(j_t/m - y_t)^2 - (f_{j_t}(x_t) - y_t)^2\right]
    \label{eqn:context-swap-direct}
\end{align}
would be bounded where in each round, it's taking the expectation over the previous rounds' realizations as opposed with respect to a single realized transcript. However, later in Section~\ref{subsec:swapMC}, we need to use a concentration argument to show that with high probability, the realized contextual swap regret is bounded, and such concentration argument requires forming a martingale difference where we need the the per-round regret to be bounded in expectation conditional on the realized transcript. Hence, what we actually need is that forecaster $\forecastSwap$ obtains $\pi_{1:T}$ such that
\begin{align}
    \sup_{\{f_j\}_{j \in [m]} \in \cF^m_B}\sum_{t=1}^T \E_{i'_t \sim \ti_t, j'_t \sim \tj_t}\left[(j'_t/m - y_t)^2 - (f_{j'_t}(x_t) - y_t)^2| \pi_{1:t-1} \right].
    \label{eqn:context-swap-want}
\end{align}
Once again, note that quantity \eqref{eqn:context-swap-want} is not the same as \eqref{eqn:context-swap-direct} because in each round $t$, \eqref{eqn:context-swap-direct} takes expectation over all possible realization of the randomness of the algorithm in previous rounds, but \eqref{eqn:context-swap-want} is with respect to specific realization of some $\pi_{1:t-1}$.

Suppose using the original reduction of \cite{blum2007external}: the forecast $j_t/m$ is chosen according to $a_t$ in each round and every oracle is updated so $\{(x_\tau, y_\tau, \hp_\tau)\}_{\tau=1}^{t-1}$ is sufficient to calculate the distribution for $\thp_t$. However, when updating every oracle via $s^i_t$, this approach requires re-scaling the loss vector fed to $i$th oracle by $a[i]$, but it is not clear how to do such rescaling in our setting where we provide feedback to oracle $\tcA_i$'s via feature-label pairs $(x,y)$.

Instead, we use \cite{ito2020tight}'s approach and use the following some concentration lemma to bound the quantity \eqref{eqn:context-swap-want} in Theorem~\ref{thm:forecast-swap-context-regret}. For Forcaster $\forecastSwap$, because the random variable $\tj_t$ necessarily depends on $i_t$, we make their dependence explicit by writing $\tj_t(i_t)$ --- i.e. the distribution of $\tj_t$ when the realization of the random variable $\ti_t$ is $i_t$. 

\begin{restatable}{lemma}{lemswapregretconcentrationtwo}
\label{lem:swap-regret-concentration2}
Suppose the sequential shattering dimension of $\cF_B$ is finite at any scale $\delta$: $\fat_{\delta}(\cF_B) < \infty$. With probability $1-4m\rho$, $\forecastSwap$ results in $\pi_{1:T}$ such that
    \begin{align*}
         &\max_{\{f_j\}_{j \in [m]}}\bigg|\frac{1}{T} \sum_{t=1}^T \E_{j'_t \sim \tj_t(i_t)}[(j'_t/m- y_t)^2 - (f_{i_t}(x_t) - y_t)^2|\pi_{1:t-1}] \\
         &\;\;\;\;\;\;\;\;\;\;\;\;\;\;\;\;\;\;\;\;\;\;\;- \E_{i'_t \sim \ti_t, j''_t \sim \tj_t(i'_t)}\left[(j''_t/m - y_t)^2 - (f_{i'_t}(x_t) - y_t)^2 | \pi_{1:t-1}\right]\bigg| \\
         &\le \max(8B, 2\sqrt{B})m C_{\cF_B} \sqrt{\frac{\log(\frac{1}{\rho})}{T}}
    \end{align*}
    where $C_{\cF_B}$ is a finite constant that depends on the sequential fat shattering dimension of $\cF_B$.
\end{restatable}
Because $\cF$ is infinite, we could not have merely taken a union bound over all such possibilities and used Azuma's inequality for each $f$. Instead, to prove Lemma~\ref{lem:swap-regret-concentration2}, we borrow tools from \cite{block2021majorizing} to argue concentration over all possible $\{f_p\}_{p \in \cP}$. 

Note how the lemma requires the sequential fat shattering dimension of $\cF_B$ to be bounded at any scale $\delta$ and the final bound also depends on this complexity measure. \cite{rakhlin2015online} shows that online learnable $\cF_B$ must have a finite sequential fat shattering dimension at any scale $\delta$ and since we are assuming a regression oracle for $\cF_B$, the assumption below on the sequential fat shattering dimension is innocuous. See Appendix~\ref{app:concentration} to see the definition of sequential fat shattering dimension and more details. 

Equipped with the above concentration lemma, we can show the following contextaul swap regret bound for forecaster $\forecastSwap$:
\begin{theorem}
\label{thm:forecast-swap-context-regret}
Fix some $B>0$. Suppose the sequential shattering dimension of $\cF_B$ is finite at any scale $\delta$: $\fat_{\delta}(\cF_B) < \infty$. Suppose oracle $\cA$'s regret bound $r_\cA(T, \cF_B)$ is concave in time horizon $T$. Fix any adversary $\adv$ that forms $(x_t, y_t)$ as a function of $\psi_{1:t-1}= \{(x_\tau, y_\tau, \hp_\tau)\}_{\tau=1}^{t-1}$. With probability $1-\rho$ over the randomness of $\{ i_t \sim \ti_t \}_{t=1}^T$, Forecaster $\forecastSwap(\cF_B, \cA, m, T)$ results in $\pi_{1:T} = \{(x_t, y_t, \hp_t, (i_t, j_t))\}_{t=1}^T$ such that
\begin{align*}
    &\sup_{\{f_p\}_{p \in \cP} \in \cF_B^m}  \sum_{t=1}^T \E_{\hp'_t}\left[(\hp'_t/m - y_t)^2 - (f_{\hp'_t}(x_t) - y_t)^2|\pi_{1:t-1} \right]     \\
    &=\sup_{\{f_j\}_{j \in [m]} \in \cF_B^m}  \sum_{t=1}^T \E_{i'_t \sim \ti_t, j''_t \sim \tj_t(i'_t)}\left[(j''_t/m - y_t)^2 - (f_{j''_t}(x_t) - y_t)^2|\pi_{1:t-1}\right] \\
    &\le \left(m r_\cA\left(\frac{T}{m}, \cF_B \right) + \frac{3T}{m} + m\right) + \max(8B, 2\sqrt{B})m C_{\cF_B} \sqrt{\frac{\log(\frac{4m}{\rho})}{T}}.
\end{align*}
\end{theorem}

\begin{proof}

Fix any $\{f_j\}_{j \in [m]}$. Fix any round $t \in [T]$ and its transcript up to previous round $\pi_{1:t-1}$. Then, we have for any $x_t, y_t$
\begin{align}
    &\E_{i'_t \sim \ti_t, j''_t \sim \tj_t(i_t)}\left[\left(j''_t/m - y_t\right)^2 - (f_{j''_t}(x_t) - y_t)^2 \middle| \pi_{1:t-1}\right]\nonumber\\
    &=\sum_{i,j \in [m]} a_t[i] q^t_{i}[j] \cdot \left(\frac{j}{m} - y_t\right)^2 - \sum_{i,j \in [m]} a_t[i] q^t_{i}[j] \cdot (f_{j}(x_t) - y_t)^2\nonumber\\
    &=\sum_{i,j \in [m]} a_t[i] q^t_{i}[j] \left(\frac{j}{m} - y_t\right)^2 - \sum_{i \in [m]} a_t[i]  (f_{i}(x_t) - y'_t)^2\nonumber\\
    &=\E_{i'_t \sim \ti_t, j''_t \sim \tj_t(i_t)}\left[\left(j''_t/m - y_t\right)^2 - (f_{i'_t}(x_t) - y_t)^2 | \pi_{1:t-1}\right]\label{eqn:swap-via-station}
\end{align}
where the second equality follows from \eqref{eqn:stat-dist}.

For convenience, we define $L_{y_t} \in [0,1]^m$ given $y_t \in \cY$ such that for each $j \in [m]$
\begin{align*}
    L_{y_t}[j] = \left(\frac{j}{m}-y_t\right)^2.
\end{align*}

Then with probability $1-\rho$,
\begin{align*}
    &\sum_{t=1}^T \E_{i'_t \sim \ti_t, j''_t \sim \tj_t(i'_t)}\left[(j_t/m - y_t)^2 - (f_{j''_t}(x_t) - y_t)^2| \pi_{1:t-1}\right]\\
    &=\sum_{t=1}^T \E_{i'_t \sim \ti_t, j''_t \sim \tj_t(i'_t)}\left[(j_t/m - y_t)^2 - (f_{i'_t}(x_t) - y_t)^2 | \pi_{1:t-1}\right]\\
    &\le \sum_{t=1}^T \E_{j''_t \sim \tj_t(i_t)}\left[(j''_t/m - y_t)^2 - (f_{i_t}(x_t) - y_t)^2 | \pi_{1:t-1}\right] + \max(8B, 2\sqrt{B})m C_{\cF_B} \sqrt{\frac{\log(\frac{4m}{\rho})}{T}}\\
    &=\sum_{i \in [m]} \sum_{t \in [T]: i_t = i} \E_{j' \sim \tj_t(i_t)}[(j''/m - y_t)^2 - (f_i(x_t) - y_t)^2 | \pi_{1:t-1}] + \max(8B, 2\sqrt{B})m C_{\cF_B} \sqrt{\frac{\log(\frac{4m}{\rho})}{T}}\\
    &=\sum_{i \in [m]} \sum_{t \in [T]: i_t = i} \tcA^m_i(s^i_t) \cdot L_{y_t} - (f_i(x_t) - y_t)^2 + \max(8B, 2\sqrt{B})m C_{\cF_B} \sqrt{\frac{\log(\frac{4m}{\rho})}{T}}
\end{align*}
where the first equality follows \ref{eqn:swap-via-station} and the inequality follows from Lemma~\ref{lem:swap-regret-concentration2}

Lemma~\ref{lem:rounded-random-oracle} tells us that for every $i \in [m]$, $i$th oracle's overall regret is bounded as follows:
\begin{align*}
    \sum_{t\in [T]: i_t = i}\tcA^{m}_{i}(s^{i}_{t}) \cdot L_{y_t} - (f_{i}(x_t) - y_t)^2 \le r_\cA(T_i) + \frac{4T_i}{m} + \sqrt{T_i}
\end{align*}
where $T_i = |\{t \in [T]: i_t = i\}|$. For convenience, write 
\[
    r_{\tcA^m}(T, \cF) = r_\cA(T, \cF) + \frac{4T}{m} + \sqrt{T}.
\]

In other words, we now need only bound the sum of expected regret over $m$ oracles where the randomness is taken over $j_t$'s. Note that the sum of the expected regret over $m$ oracles can be bounded as
\begin{align*}
    &\sum_{i \in [m]} \sum_{t \in [T]: i_t = i} \tcA^m_i(s^i_t) \cdot L_{y_t} - (f_i(x_t) - y_t)^2 \\
    &\le \sum_{i \in [m]} r_{\tcA^m_i}(|t \in [T]: i_t=i|, \cF) \\
    &\le \max_{\{T_i\}_{i \in [m]}: \sum_{i \in [m]}T_i = T} \sum_{i \in [m]} r_{\tilde{\cA}_m}(T_i, \cF)\\
    &\le \max_{\{T_i\}_{i \in [m]}: \sum_{i \in [m]}T_i = T} m 
    \cdot r_{\tilde{\cA}^m}\left(\frac{1}{m}  \sum_{i \in [m]} T_i, \cF \right)\\
    &=  m \cdot r_{\tilde{\cA}^m}\left(\frac{T}{m}, \cF \right)\\
    &= m \cdot \left(r_\cA\left(\frac{T}{m}, \cF\right) + \frac{3T}{m^2} + 1 \right)\\
    &= m r_\cA\left(\frac{T}{m}, \cF \right) + \frac{3T}{m} + m
\end{align*}
where the second inequality follows from applying Jensen's inequality to a concave function $r_{\tcA^m}(\cdot, \cF)$.
\end{proof} 

We remark that although the algorithm maintains $m$ copies of the underlying regression oracle $\cA$, it only needs to \emph{update} one of them per-round just as in \cite{ito2020tight}. Thus, the per-round run-time of our algorithm is equal to the per-round run-time of the regression oracle $\cA$ that we reduce to, plus the additional additive overhead of solving equation \ref{eqn:stat-dist}.

\subsection{From Contextual Swap Regret to $L_2$-Swap-Multicalibration}
\label{subsec:swapMC}
Below, we extend a characterization proven in \cite{globus2023multicalibration, gopalan2023characterizing} within a batch framework to our online setting. The characterization allows us to show that a forecaster's $L_2$-swap-multicalibration error with respect to $\cF$ can be bounded by its contextual swap regret with respect to $\cF$. The characterization rests on a mild assumption:

\begin{ass}
    We assume that $\cF$ is closed under affine transformation: if $f \in \cF$, then $f'(x) := af(x) + b$ for every $a, b \in \R$ also belongs to $\cF$.
\end{ass}

\begin{remark}
Many natural classes of regression functions are closed under affine transformations: for example the set of all linear functions, the set of all polynomial functions of any bounded degree, and the set of all regression trees of any bounded depth. Other classes (like neural networks with softmax outputs) are not already closed under affine transformation, but can be made so by introducing two new parameters ($a$ and $b$) while maintaining differentiability. Thus we view this as a mild assumption, that is enforceable if it is not already satisfied. 
\end{remark}

We  show that if there is a function $f \in \cF$ and forecast $p \in \cP$ that witnesses the forecaster's calibration error being large, then it must be the case that the forecasters contextual swap regret is also large. Thus, by contrapositive, if the forecaster's contextual swap regret is small, then it also cannot have large calibration error. 
\begin{lemma}\label{lem:cal-error-implies-swap-error}
Fix any $\pi_{1:T}$, $p \in \cP$, and $B \ge 1$. Suppose there exists $f \in \cF_B$ such that the forecaster's multicalibration error with respect to $f$ and $p$ is at least $\alpha$ for some $\alpha \in [0,1]$:
\[
    K(\pi_{1:T}, p, f) \ge \alpha.
\]
Then there exists $f' \in \cF_{(1+\sqrt{B})^2}$ that witnesses that the forecaster's contextual swap regret for forecast $p \in \cP$ also scales with $\alpha^2$:
    \[
        \frac{1}{n(\pi_{1:T}, p)}\sum_{t \in S(\pi_{1:T}, p)} (\hp_t - y_t)^2 - (f'(x_t) - y_t)^2 \ge \frac{\alpha^2}{B}
    \]
\end{lemma}
\proof
Fix any $\pi_{1:T}$. Let $f'(x) = p + \eta f(x)$ where \[\eta = \min\left(1, \frac{\alpha}{\frac{1}{n(p)} \sum_{t \in S(\pi_{1:T}, p)} f(x_t)^2}\right).
\]
    \begin{align*}
        &\frac{1}{n(p)}\sum_{t \in S(p)} \left((\hp_t - y_t)^2 - (f'(x_t) - y_t)^2\right) \\
        &=\frac{1}{n(p)}\sum_{t \in S(p)} \left((p^2-2py_t +y_t^2) - (f'(x_t)^2 - 2y_tf'(x_t) + y^2_t) \right)\\
        &=\frac{1}{n(p)|}\sum_{t \in S(p)} \left(p^2-2py_t - (p + \eta f(x_t))^2 + 2y_t(p + \eta f(x_t)) \right)\\
        &=\frac{1}{n(p)}\sum_{t \in S(p)} \left(- 2p\eta f(x_t) - \eta^2 f^2(x_t))^2 + 2y_t\eta f(x_t) \right)\\
        &=\frac{1}{n(p)}\sum_{t \in S(p)} \left(2\eta f(x_t)(y_t - p)  - \eta^2 f^2(x_t))^2 \right)\\
        &\ge 2\eta \alpha -  \frac{\eta^2}{|S(p)|} \sum_{t \in S(p)} f(x_t)^2
    \end{align*}

    For convenience, write 
    \[
        \tau = \frac{1}{n(p)} \sum_{t \in S(p)} f(x_t)^2
    \]
    where $\tau \le B$ because $f \in \cF_B$.
    
    If $\alpha \ge \tau$ meaning $\eta=1$, then
    \begin{align*}
        &\frac{1}{n(p)}\sum_{t \in S(p)} \left((\hp_t - y_t)^2 - (f'(x_t) - y_t)^2\right)\\
        &\ge \eta \left(2\alpha - \eta \tau \right)\\
        &\ge \alpha \\
        &\ge \frac{\alpha^2}{B}
    \end{align*}
    In the other case when $\alpha < \frac{1}{n(p)} \sum_{t \in S(p)} f(x_t)^2$, plugging in $\eta = \frac{\alpha}{\tau}$ immediately gives us
    \begin{align*}
        &\frac{1}{n(p)}\sum_{t \in S(p)} \left((\hp_t - y_t)^2 - (f'(x_t) - y_t)^2\right)\\
        &\ge \frac{\alpha^2}{B}.&\qed
    \end{align*}

The contrapositive of the above lemma can be used to show that the forecaster's $L_2$-swap-multicalibration error can be bounded with its contextual swap regret.

\begin{theorem}
\label{thm:swap-regret-to-multical}
Fix some $B \ge 1$ and transcript $\pi_{1:T}$ generated by forecaster $\forecast$ and adversary $\adv$. Suppose the forecaster's average contextual-swap regret with respect to $\cF_{(1+\sqrt{B})^2}$ is bounded as follows: for any $\{f_p\}_{p \in \cP} \in \cF_{(1+\sqrt{B})^2}^m$, we have
\[
    \frac{1}{T} \sum_{p \in \cP} \sum_{t \in S(\pi_{1:T}, p)} (\hp_t - y_t)^2 - (f_p(x_t) - y_t)^2 \le \frac{\alpha}{B^2}.
\]
Then the forecaster's $L_2$-swap-multicalibration error with respect to $\cF_B$ is at most $\alpha$:
\[\bsK_{2}(\pi_{1:T}, \cF_B) \le \alpha.\] 
\end{theorem}
\begin{proof}
Fix $\pi_{1:T}$ and $\alpha \in [0,1]$. For the sake of contradiction, suppose the forecaster's $L_2$-swap-multicalibration error with respect to $\cF_B$ is at least $\alpha$, meaning there exists some $\{f_p\}_{p \in \cP} \in \cF^m_{B}$ such that 
\[
    \sum_{p \in P} \frac{n(p)}{T} (K(\pi_{1:T}, p, f_p))^2 > \alpha.
\]
For each $p \in \cP$, write
\[
    \alpha_p = \frac{n(p)}{T} (K(\pi_{1:T}, p, f_p))^2
\] 
to denote the multicalibration error with respect to forecast $p \in \cP$. In other words, we have
\begin{align*}
    K(\pi_{1:T}, p, f^*_p) = \sqrt{\frac{\alpha_p T}{n(p)}} 
\end{align*}
where $f^*_p$ is either $f_p \in \cF_B$ or $-f_p \in \cF_B$ for each $p \in \cP$. 
Because $|K(\pi_{1:T}, p, f)| \le \sqrt{B}$ for any $f \in \cF_B$, we have $\frac{1}{\sqrt{B}}K(\pi_{1:T}, p, f^*_p) \in [0,1].$ In other words, 
\[
    \sqrt{\frac{\alpha_p T}{B n(p)}} \le 1
\]
and $K(\pi_{1:T}, p, f^*_p) \ge \sqrt{\frac{\alpha_p T}{B n(p)}}$ as $B \ge 1$.

Then, for each $p \in \cP$, Lemma~\ref{lem:cal-error-implies-swap-error} yields that there exists $f'_p \in \cF_{(1+\sqrt{B})^2}$ such that
\begin{align*}
    \frac{1}{n(p)}\sum_{t \in S(p)} (\hp_t - y_t)^2 - (f'_p(x_t) - y_t)^2 &\ge \frac{1}{B} \left(\sqrt{\frac{\alpha_p T}{B n(p)}}  \right)^2 = \frac{1}{B^2} \frac{\alpha_p T}{n(p)}
\end{align*}

Adding over $p \in P$ gives us 
\begin{align*}
    \sum_{p \in P} \sum_{t \in S(p)} (\hp_t - y_t)^2 - (f'_p(x_t) - y_t)^2 &\ge \sum_{p \in P}\frac{\alpha_p T}{B^2} \ge \frac{\alpha T}{B^2}.
\end{align*}
This is a contradiction to our assumption about the contextual swap regret.
\end{proof}

\begin{remark}
Theorem \ref{thm:swap-regret-to-multical} bounds the forecaster's multicalibration error with respect to $\cF_B$, a subset of $\cF$ with bounded magnitude. This is necessary, since we have assumed that $\cF$ is closed under affine transformation. Observe that if $\cF$ is closed under affine transformations, and a forecaster has non-zero calibration error with respect to $\cF$, then it must actually have unboundedly large calibration error with respect to $\cF$, because for any $f \in \cF$ that witnesses a failure of $\alpha$-appoximate multicalibration, the functions $100f$, $1000f$, etc. are also in $\cF$. Thus for such function classes, statements of approximate multicalibration must always be with respect to an upper bound on the magnitude of the functions we are considering. 
\end{remark}

In this section so far, we have shown a connection between the contextual swap regret and $L_2$-multicalibration for a fixed transcript $\pi_{1:T}$. However, the forecaster $\forecastSwap$ we have constructed in Section~\ref{subsec:no-contextual-swap-regret} only bounds the contextual swap regret in expectation for some fixed $\{f_p\}_{p \in \cP} \in \cF^m_B$ --- see Theorem~\ref{thm:forecast-swap-context-regret}. And it is not immediate how to go from expected contextual swap regret to expected $L_2$-multicalibration. Therefore, as before, we can use the same argument as in Lemma~\ref{lem:swap-regret-concentration2} where we borrowed tools from \cite{block2021majorizing} to prove the following concentration bound.

\begin{restatable}{lemma}{lemswapregretconcentration}
\label{lem:swap-regret-concentration}
Suppose the sequential shattering dimension of $\cF_B$ is finite at any scale $\delta$: $\fat_{\delta}(\cF_B) < \infty$. With probability $1-4m\rho$, $\forecastSwap$ results in $\pi_{1:T}$ such that
    \begin{align*}
         &\max_{\{f_p\}_{p \in [m]}}\left|\frac{1}{T} \sum_{t=1}^T (\hp_t- y_t)^2 - (f_{\hp_t}(x_t) - y_t)^2 - \E_{\hp'_t}\left[(\hp'_t - y_t)^2 - (f_{\hp'_t}(x_t) - y_t)^2 | \pi_{1:t-1}\right]\right| \\
         &\le \max(8B, 2\sqrt{B})m C_{\cF_B} \sqrt{\frac{\log(\frac{1}{\rho})}{T}}
    \end{align*}
    where $C_{\cF_B}$ is a finite constant that depends on the sequential fat shattering dimension of $\cF_B$.
\end{restatable}

\begin{corollary}
\label{cor:swap-forcast-bounds-multical}
Fix some $B \ge 1$. Assume the sequential fat shattering dimension of $\cF_B$ is bounded at any scale $\delta$.  Against any adversary $\adv$, $\forecastSwap(\cF_{(1+\sqrt{B})^2}, \cA, m, T)$ guarantees that the $L_2$-swap-multicalibration error with respect to $\cF_B$ is bounded with high probability as follows: with probability $1-\rho$ over the randomness of $\{\hp_{t}\}_{t=1}^T$, $\forecastSwap$ results in $\pi_{1:T}$ such that
\[
    \bsK_2(\pi_{1:T}, \cF_B) \le B^2 \cdot \left(\frac{1}{T}\left(m r_\cA\left(\frac{T}{m}, \cF_{B'} \right) + \frac{3T}{m} + m\right) + 16B'm C_{\cF_{B'}} \sqrt{\frac{\log(\frac{8m}{\rho})}{T}} \right).
\]
where $B' = (1+\sqrt{B})^2$ and $C_{\cF_B}$ is as defined in Lemma~\ref{lem:swap-regret-concentration}.
\end{corollary}
\begin{proof}
With probability $1-2\rho$, $\forecastSwap(\cF_{B'}, \cA, T)$ guarantees that 
\begin{align*}
    &\sup_{\{f_p\}_{p \in \cP}}\frac{1}{T}\sum_{t=1}^T (\hp_t- y_t)^2 - (f_{\hp_t}(x_t) - y_t)^2 \\
    &\le \sup_{\{f_p\}_{p \in \cP} \in \cF_{B'}^m} \frac{1}{T}\sum_{t=1}^T \E_{\hp_t}\left[(\hp_t- y_t)^2 - (f_{\hp_t}(x_t) - y_t)^2 | \pi_{1:t-1}\right] + \max(8B', 2\sqrt{B'})m C_{\cF_B'} \sqrt{\frac{\log(\frac{4m}{\rho})}{T}}\\
    &\le\frac{1}{T}\left(m r_\cA\left(\frac{T}{m}, \cF_{B'} \right) + \frac{3T}{m} + m\right) + 16B'm C_{\cF_B} \sqrt{\frac{\log(\frac{4m}{\rho})}{T}}
\end{align*}
where the first inequality follows from Lemma~\ref{lem:swap-regret-concentration} and the second from Theorem~\ref{thm:forecast-swap-context-regret}.

Now, appealing to Theorem~\ref{thm:swap-regret-to-multical}, which tells us that if contextual swap regret with respect to $\cF_{B'}$ is bounded, then $L_2$-multicalibration error with respect to $\cF_B$ is bounded, gives us 
\begin{align*}
    \bsK_2(\pi_{1:T}, \cF_B) \le B^2 \cdot \left(\frac{1}{T}\left(m r_\cA\left(\frac{T}{m}, \cF_{B'} \right) + \frac{3T}{m} + m\right) + 16B'm C_{\cF_{B'}} \sqrt{\frac{\log(\frac{4m}{\rho})}{T}} \right).
\end{align*}
\end{proof}

We end this section by instantiating Corollary \ref{cor:swap-forcast-bounds-multical} with concrete rates. Azoury and Warmuth \cite{azoury2001relative} give an efficient algorithm for online squared error linear regression that has a regret bound scaling logarithmically with $T$ (\cite{foster1991prediction} and \cite{vovk2001competitive} give similar bounds):
\begin{theorem}[\cite{azoury2001relative}]
\label{thm:azoury}
There exists an efficient online forecasting algorithm such that for all sequences $\{(x_t,y_t)\}_{t=1}^T$ with $x_t \in \R^d$ with  $||x_t||_{2} \leq 1$ and $|y_t| \leq 1$, and for all parameter vectors $\theta$:
$$\frac{1}{T}\left(\sum_{t=1}^T(\hat p_t - y_t)^2 - (\langle \theta, x_t \rangle - y_t)^2 \right)\leq \frac{||\theta||^2}{T} + \frac{2d\ln(T+1)}{T}$$
\end{theorem}

Here $||\theta||$ corresponds to our bound $B$ on the  magnitude of the comparison functions $\cF$. Thus if $\cF_{(1+\sqrt{B})^2}$ is taken to be the set of all linear functions whose bound is less than $(1+\sqrt{B})^2$, Theorem \ref{thm:azoury} is giving us a regret bound for $\cF_{B}$. We can therefore instantiate Corolary \ref{cor:swap-forcast-bounds-multical} to obtain the following bound for  efficient (swap) multicalibration with respect to linear functions:
\begin{corollary}
\label{cor:linearcalibration}
Fix $B \ge 1$. Let $\cF$ be the set of all $d$-dimensional linear functions, and let $\cF_B$ we the set of all such functions with parameter norm $||\theta||^2 \leq (1+\sqrt{B})^2$. Then against any adversary who chooses a sequence of $d$ dimensional examples $(x_t,y_t)_{t=1}^T$ with $||x_t|| \leq 1$,  letting $\cA$ be the online forecasting algorithm for linear functions from Theorem \ref{thm:azoury} and letting $m = T^{1/4}$, 
    $\forecastSwap(\cF_{(1+\sqrt{B})^2}, \cA, m, T)$ guarantees that the $L_2$-swap-multicalibration error with respect to $\cF_B$ is bounded with probability $1-\rho$ over the randomness of the forecaster:
    \begin{align*}
    \bsK_{2}(\pi_{1:T}, \cF_B) &\le B^2 \cdot \left(\frac{1}{T}\left(m \left(\frac{m||\theta||^2}{T} + \frac{2dm\ln(T+1)}{T} \right) + \frac{3T}{m} + m\right) + 16B'm C_{\cF_{B'}} \sqrt{\frac{\log(\frac{8m}{\rho})}{T}} \right)\\
    &=\frac{B^2 m^2 {B'}}{T^2} + \frac{2dB^2m^2\ln(T+1)}{T^2} + \frac{3}{m} + \frac{m}{T} + 16B'm C_{\cF_{B'}} \sqrt{\frac{\log(\frac{8m}{\rho})}{T}} \\
    &= \tilde{O}\left(dB^3\sqrt{\ln\left(\frac{1}{\rho}\right)} T^{-1/4}\right).
\end{align*}
\end{corollary}

Thus we have a computationally efficient algorithm (for linear functions, we need not say ``oracle efficient'' since the oracle is actually implemented in polynomial time) for making predictions that are multicalibrated with respect to linear functions with error bounds tending to zero at the rate of $O\left(\frac{1}{T^{1/4}}\right)$.

Is a rate of $O\left(\frac{1}{T^{1/2}}\right)$ for $L_2$-multicalibration possible? In Section \ref{sec:AMF} we show how to obtain this rate for \emph{finite} classes with a non-oracle-efficient algorithm with running time scaling linearly with $|\cF|$ and error rates scaling logarithmically with $|\cF|$. Achieving this rate in an oracle efficient manner is left as an open question.

\section{Online (Swap-)Multicalibration to Online (Swap-)Omniprediction}
\label{sec:omnipredictor}
Finally we arrive at the last step of our argument, connecting our bounds for online multicalibration to bounds on online omniprediction. 
We adapt similar arguments of \cite{GopalanKRSW22, gopalan2023characterizing} which connect multicalibration to omniprediction in the batch setting. All the missing proofs in this section can be found in Appendix~\ref{app:miss-omnipredictor}.

First, we show that the conditional expected value of $f(x_t)$ over the empirical distribution of $\pi_{1:T}$ does not vary substantially when conditioning on $y_t = 1$ or $y_t = 0$ --- in particular, an approximate version of the following equalities:
\begin{align*}
     \left|\E[y_t f(x_t)] - \E[y_t] \E[f(x_t)]\right| &= \left|\Pr[y_t = 1] \cdot (\E[f(x_t)] - \E[f(x_t) | y_t=1])\right| \\
     &= \left|\Pr[y_t = 0] \cdot (\E[f(x_t)] - \E[f(x_t) | y_t=0])\right|. 
\end{align*}
The lemma below corresponds to Corollary 5.1 of \cite{GopalanKRSW22} and Claim 6.3 of \cite{gopalan2023characterizing}. 
\begin{restatable}{lemma}{lemconditionylittlechange}
\label{lem:condition-y-little-change}
Fix $\pi_{1:T}$. If $\bsK_1(\pi_{1:T}, \cF) \le \alpha$, then the following holds for any $y \in \cY$ and $\{f_p\}_{p \in \cP}$:
\[
    \sum_{p \in \cP}\frac{|S(\pi_{1:T}, p, y)|}{T} \left|\barf_p(\pi_{1:T}, p, y)  - \barf_p(\pi_{1:T}, p)\right| \le 2\alpha.
\]
Similarly, if $\bK(\pi_{1:T}, \cF) \le \alpha$, then we have for any $y \in \cY$ and $f \in \cF$
\[
    \sum_{p \in \cP}\frac{|S(\pi_{1:T}, p, y)|}{T} \left|\barf(\pi_{1:T}, p, y)  - \barf(\pi_{1:T}, p)\right| \le 2\alpha.
\]    
\end{restatable}

We show that for any forecast $p \in \cP$, the difference in loss over rounds in $S(p)$ between the the post-processing function $k^\ell$ defined in \ref{def:post-process} and any other post-processing function can be bounded in terms of the calibration error. The lemma below essentially corresponds to arguments presented in Corollary 6.2 of \cite{GopalanKRSW22} and Lemma 6.5 of \cite{gopalan2023characterizing}.

\begin{restatable}{lemma}{lemoptimalitypostprocess}
\label{lem:optimality-post-process}
Fix $\pi_{1:T}$ and loss function $\ell$. For any $p \in \cP$, $k: \cP \to [0,1]$, we have
    \[
        \frac{1}{n(\pi_{1:T}, p)}\sum_{t \in S(\pi_{1:T}, p)} \ell(y_t, k^\ell(p)) \le \frac{1}{n(\pi_{1:T}, p)}\sum_{t \in S(\pi_{1:T}, p)} \ell(y_t, k(p)) + C_\ell |K(\pi_{1:T}, I, p)|    
    \]
where $I$ is the constant function defined as in Assumption~\ref{ass:all-one}, $k^\ell$ is defined in Definition~\ref{def:post-process}, and $C_\ell = \max_{a} \left|\ell(0, a) - \ell(1, a)\right|$.    
\end{restatable}

Using the two helper lemmas above, we can show that $L_1$-swap multicalibration error approximately bounds the swap omniprediction regret.
Write the family of convex loss functions as $\cL_{\text{convex}}$ --- i.e. $\ell(y, \cdot)$ is convex for each $y \in \cY$. The theorem below is an adaptation of Theorem 6.1 from \cite{gopalan2023characterizing} to the online setting.
\begin{restatable}{theorem}{thmmulticaltoomni}
\label{thm:multical-to-omni}
Fix $\pi_{1:T}$. Then for any $\{\ell_p\}_{p \in \cP} \in \cLcvx^m$, we have 
\[
    \scO(\pi_{1:T}, \{\ell_p\}_{p \in \cP}, \cF) \le \left(\max_{p \in \cP} (C_{\ell_p} + 4D_{\ell_p})\right) \cdot \bsK_1(\pi_{1:T}, \cF)\quad\text{and}\quad \cO(\pi_{1:T}, \ell, \cF) \le (C_\ell + 4D_\ell) \bK_1(\pi_{1:T}, \cF)
\]
where $C_\ell$ is as defined in Lemma~\ref{lem:optimality-post-process} and $D_\ell = \max_{y \in \cY, t \in [0,1]} |\ell'(y, t)|$ is the bound on the derivative of loss.

Therefore, for any $\cL \subseteq \cLcvx$
\[
    \scO(\pi_{1:T}, \cF, \cL) \le (C_\cL + 4D_\cL) \bsK_1(\pi_{1:T}, \cF)\quad\text{and}\quad\cO(\pi_{1:T}, \cF, \cL) \le (C_\cL + 4D_\cL) \bK_1(\pi_{1:T}, \cF)
\]
where $C_\cL = \max_{\ell \in \cL} C_\ell$ and $D_\cL = \max_{\ell \in \cL} D_\ell$.   
\end{restatable}

Finally, we can apply our online multicalibration bounds to get concrete bounds for online omniprediction. First, we apply our bound for oracle efficient multicalibration which we derived in Corollary \ref{cor:swap-forcast-bounds-multical}.

\begin{corollary}
\label{cor:swap-omni-oracle-rate}
Fix $B \ge 1$, some family of convex loss functions $\cL \subseteq \cLcvx$, and family of predictors $\cF$. Suppose $\cF_{B'}$'s sequential fat shattering dimension at any scale $\delta$ is bounded where $B' = (1+\sqrt{B})^2$. Against any adversary $\adv$, $\forecastSwap(\cF_{B'}, \cA, m, T)$ makes forecasts such that its swap omniprediction regret with respect to $\cF_B$ and $\cL$ is bounded with probability $1-\rho$:
\[
    \scO(\pi_{1:T}, \cF_B, \cL)  \le (C_\cL + 4D_\cL) \sqrt{B^2 \cdot \left(\frac{1}{T}\left(m r_\cA\left(\frac{T}{m}, \cF_{B'} \right) + \frac{3T}{m} + m\right) + 16B'm C_{\cF_{B'}} \sqrt{\frac{\log(\frac{8m}{\rho})}{T}} \right)}
\]
where we recall that $C_\ell = \max_{a} \left|\ell(0, a) - \ell(1, a)\right|$ bounds the scale of the loss function $\ell$, $C = \max_{\ell \in \cL} C_\ell$ is a uniform bound overthe scale of all loss functions $\ell \in \cL$, $D_\ell = \max_{y \in \cY, t \in [0,1]} \ell'(y, t)$ is a bound on the derivative of $\ell$ in its second argument, and $D = \max_{\ell \in \cL} D_\ell$ is a uniform bound on the derivative. $C_{\cF_B}$ is a finite constant that depends on the sequential fat shattering dimension of $\cF_{B'}$.
\end{corollary}
\begin{proof}
Lemma~\ref{lem:multical-relationship} and Theorem~\ref{thm:multical-to-omni} give us
    \begin{align*}
        \scO(\pi_{1:T}, \cF_B, \cL_{convex}) \le (C + 4D)\bsK_1(\pi_{1:T}, \cF) \le (C + 4D)\sqrt{\bsK_2(\pi_{1:T}, \cF)}. 
    \end{align*}

    Applying Corollary~\ref{cor:swap-forcast-bounds-multical} which states that $\forecastSwap$ bounds $\bsK_2(\pi_{1:T}, \cF_B)$ gives us that with probability $1-\rho$,
    \begin{align*}
       &\scO(\pi_{1:T}, \cF_B, \cL_{convex})\\
        &\le \left(\max_{p \in \cP} (C_{\ell_p} + 4D_{\ell_p})\right) \sqrt{\bsK_2(\pi_{1:T}, \cF)}\\
        &\le \left(\max_{p \in \cP} (C_{\ell_p} + 4D_{\ell_p})\right)\sqrt{B^2 \cdot \left(\frac{1}{T}\left(m r_\cA\left(\frac{T}{m}, \cF_{B'} \right) + \frac{3T}{m} + m\right) + 16B'm C_{\cF_{B'}} \sqrt{\frac{\log(\frac{8m}{\rho})}{T}} \right)}.
    \end{align*}
    where the second to last inequality follows from the fact that $\sqrt{\cdot}$ is a concave function.
\end{proof}

We now have a concrete bound for oracle efficient online omniprediction, in terms of the regret bound of a black box online squared error regression algorithm for $\cF$. We can plug in the bound of \cite{azoury2001relative} for online linear regression, quoted in Theorem \ref{thm:azoury}, to get rates specifically for omniprediction with respect to linear functions.

\begin{corollary}
\label{cor:linearomniprediction}
    Fix some family of convex loss functions $\cL \subseteq \cL_{\text{convex}}$. Let $\cF$ be the set of all $d$-dimensional linear functions, and let $\cF_B$ we the set of all such functions with parameter norm $||\theta||^2 \leq B'$ where $B' = (1+\sqrt{B})^2$. Then against any adversary who chooses a sequence of $d$ dimensional examples $(x_t,y_t)_{t=1}^T$ with $||x_t|| \leq 1$, then letting $\cA$ be the online forecasting algorithm for linear functions from Theorem \ref{thm:azoury} and letting $m = T^{1/4}$, 
    $\forecastSwap(\cF_{B'}, \cA, m, T)$ guarantees with probability $1-\rho$ that the swap-omniprediction regret with respect to $\cF_B$ and $\cL$ is bounded:
\[
     \scO(\pi_{1:T}, \cF_B, \cL) \le  \tilde{O}\left((C_\cL + 4D_\cL) \cdot \sqrt{dB^3}\left(\ln\left(\frac{1}{\rho}\right)\right)^{1/4} T^{-1/8}\right).
\]
\end{corollary}

We contrast this with the bound that we can obtain for finite classes $\cF$ with our (non-oracle-efficient) online multicalibration algorithm from Theorem \ref{thm:amfmultical}. 
\begin{corollary}
Fix some family of convex loss functions $\cL \subseteq \cLcvx$ and family of predictors $\cF_B \subseteq \cF$ whose output's squared value is bounded by $B$. Against any adversary $\adv$, $\forecast_{AMF}(\cF_B)$ makes forecasts such that its omniprediction regret with respect to $\cF$ and $\cL$ is bounded in expectation as:
\[
    \E_{\pi_{1:T}}\left[ \cO(\pi_{1:T}, \cF_B, \cL) \right] \le (C_\cL + 4D_\cL)\sqrt{\frac{3B \log(T) + 4\sqrt{B \ln(|\cF_B|)} + \sqrt{B}}{\sqrt{T}}} = O\left(\frac{\sqrt{\log(T)}}{T^{1/4}}\right).
\]
where we recall that $C_\ell = \max_{a} \left|\ell(0, a) - \ell(1, a)\right|$ bounds the scale of the loss function $\ell$, $C_\cL = \max_{\ell \in \cL} C_\ell$ is a uniform bound over the scale of all loss functions $\ell \in \cL$, $D_\ell = \max_{y \in \cY, t \in [0,1]} \ell'(y, t)$ is a bound on the derivative of $\ell$ in its second argument, and $D_\cL = \max_{\ell \in \cL} D_\ell$ is a uniform bound on the derivative.
\end{corollary}
\begin{proof}
  Fix $\pi_{1:T}$.  Lemma~\ref{lem:multical-relationship} and Theorem~\ref{thm:multical-to-omni} give us
    \begin{align*}
        \cO(\pi_{1:T}, \cF_B, \cL_{convex}) \le (C_\cL + 4D_\cL)\bK_1(\pi_{1:T}, \cF_B) \le (C_\cL + 4D_\cL)\sqrt{\bsK_2(\pi_{1:T}, \cF_B)}. 
    \end{align*}

    Applying Theorem~\ref{thm:amfmultical} which states that $\forecast_{AMF}$ bounds $\bK_2(\pi_{1:T}, \cF_B)$ in expectation gives us
    \begin{align*}
        \E_{\pi_{1:T}}[\cO(\pi_{1:T}, \cF_B, \cL_{convex})] &\le \E_{\pi_{1:T}}\left[ (C_\cL + 4D_\cL)\sqrt{\bK_2(\pi_{1:T}, \cF_B)}\right] \\
        &\le (C_\cL+4D_\cL) \sqrt{\E_{\pi_{1:T}}[\bK_2(\pi_{1:T}, \cF_B)]}\\
        &\le (C_\cL + 4D_\cL)\sqrt{\frac{3Bm \log(T) + 4\sqrt{B \ln(|\cF_B|)} + \sqrt{B}}{\sqrt{T}}}.
    \end{align*}
    where the second to last inequality follows from the fact that $\sqrt{\cdot}$ is a concave function.
\end{proof}
A high probability version for the $\ell_2$-swap-multicalibration via the AMF approach can be obtained by appealing to the high probability version of the AMF approach guarantee (Theorem A.2 in \cite{noarov2021online}): this will incur additional additive error of $O\left(\sqrt{\frac{\ln(|\cF|)}{T}}\right)$, so the overall rate remains the same. 

Observe that this bound depends only logarithmically on $|\cF|$, and so although it applies only to finite classes $\cF$, we could inefficiently apply it to the class of all appropriately discretized linear functions $\cF_B$, which would obtain a bound that was comparable to the oracle efficient bound we obtain in Corollary \ref{cor:linearomniprediction} in terms of its dependence on $B$ and $d$, but with an improved dependence on $T$.

\section{Tightness of Our Results: A Separation Between Swap Omniprediction and Omniprediction}
In this section we interrogate the extent to which our rates can be improved. Can we hope for $O(\sqrt{T})$ rates for online omniprediction? Can we hope for them using our current set of techniques? 

First, we show that $O(\sqrt{T})$ rates \emph{are} possible for (non)-swap omniprediction --- at least for finite, binary hypothesis classes. Next, we show a barrier to obtaining $O(\sqrt{T})$ using our current techniques, which actually give swap omniprediction: it is \emph{not} possible to obtain $O(\sqrt{T})$ rates in the online setting for \emph{swap} omniprediction. This establishes a formal separation between omniprection and swap omniprediction in the online setting. The algorithm we give corresponding to our omniprediction upper bound has running time that is polynomial in $|\cF|$ however: we leave the problem of finding an oracle efficient algorithm obtaining these optimal omniprediction rates as our main open question.
\subsection{An $O(\sqrt{T})$ Upper Bound for Omniprediction}\label{sec:vcall}
In this section, we give a (non-oracle efficient) $O(\sqrt{T})$ upper bound for omniprediction for finite binary classes $\cF$ that does not go through multicalibration, which together with our lower bound in Section \ref{sec:lb} gives a separation between omniprediction and swap omniprediction. Our proof combines the ``AMF'' framework of \cite{noarov2021online} with a recent characterization of binary scoring rules by \cite{kleinberg2023u}.

First we recall the Online Minimax Optimization framework of \cite{noarov2021online}:

\begin{definition}[Appendix A.3 of \cite{noarov2021online}]
    A Learner plays against an adversary over rounds $t \in [T]$. In each of the rounds, the Learner accumulates a $d$-dimensional loss vector for some $d \ge 1$. Each round's loss vector lies in $[-C, C]^d$ for some constant $C > 0$. In each round $t \in [T]$, the interaction between the Learner ad the Adverary proceeds as follows:
    \begin{enumerate}
        \item Before round $t$, the Adversary selects and reveals to the Learner an \emph{environment} comprising: 
        \begin{enumerate}
            \item The Learner's strategy space is a finite set $\Theta_t$ and the Adversary's strategy space is a convex compact action sets $\cZ_t$ 
            \item  A continuous vector loss function $\ell_t(\cdot, \cdot): \Theta_t \times \cZ_t \to [-C, C]^d$ where $\ell_t^j: \Theta_t \times \cZ_t \to [-C, C]$ is concave in the 2nd argument for each $j \in [d]$. Note that because of the finiteness of $\Theta_t$ and the Learner plays a mixed strategy over $\Theta_t$, the loss function is already linearized in terms of the Learner. 
        \end{enumerate}
        \item The Learner selects a randomized strategy $\tilde{\theta}_t \in \Delta(\Theta_t)$
        \item The Adversary observes the Learner's selection $\tilde{\theta}_t$ and responds with some $z_t \in \cZ_t$.
        \item The Learner plays $\theta_t \sim \tilde{\theta}_t$ and suffers (and observes) the loss vector $\ell_t(\theta_t, z_t)$.
    \end{enumerate}

    The Learner's objective is to minimize the value of the maximum dimension of the accumulated loss vector after $T$ rounds: i.e. $\sum_{j \in [d]} \sum_{t=1}^T \ell^j_t(\theta_t, z_t)$.
\end{definition}

\begin{definition}[\cite{noarov2021online}]
The Adversary-Moves-First (AMF) value of the game defined by the environment $(\Theta_t, \cY_t, \ell_t)$ at round $t$ is 
\[
    w^A_t = \sup_{z_t \in \cZ_t} \min_{\tilde{\theta}_t \in \Delta\Theta_t} \left(\max_{j \in [d]} \E_{\theta_t \sim \tilde{\theta}_t}[\ell^j_d(\theta_t, z_t)]\right).
\]
\end{definition}

\begin{definition}[\cite{noarov2021online}]
    Given $\{(\Theta_\tau, \cZ_\tau, \ell_\tau), \theta_\tau, z_\tau\}_{\tau=1}^t$, we define the Learner's Adversary-Moves-First (AMF) regret for the $j$th dimension at round $t \in [T]$ as 
    \[
        R^j_t(\{(\Theta_\tau, \cZ_\tau, \ell_\tau), \theta_\tau, z_\tau\}_{\tau=1}^t)  = \sum_{\tau=1}^t \left(\ell^j_\tau(\theta_\tau, z_\tau) - w^A_\tau\right).
    \]
    The overall AMF regret is defined as $R_t(\{(\Theta_\tau, \cZ_\tau, \ell_\tau), x_\tau, y_\tau\}_{\tau=1}^t) = \max_{j \in [d]} R^j_t(\{(\Theta_\tau, \cZ_\tau, \ell_\tau), \theta_\tau, z_\tau\}_{\tau=1}^t)$.
    When it is obvious from the context, we write $R^j_t$ and $R_t$.

\end{definition}

\begin{algorithm}[H]
\begin{algorithmic}[1]
\FOR{$t=1, \dots, T$}
    \STATE AMF Learner observes the adversarially chosen $\Theta_t$, $\cZ_t$, and $\ell_t$. 
    \STATE Let \[
        \chi^j_t := \frac{\exp(\eta \sum_{s=1}^{t-1} \ell^j_s(\theta_t, z_t))}{\sum_{i \in [d]}\exp(\eta \sum_{s=1}^{t-1} \ell^i_s(\theta_t, z_t))}
    \]
    \STATE AMF Learner selects a mixed strategy $\tilde{\theta}_t$ where 
    \begin{align}
        \tilde{\theta}_t \in \argmin_{\tilde{\theta} \in \Delta(\Theta_t)} \max_{z \in \cZ}\sum_{i \in [d]}\chi^i_t \E_{\theta \sim \tilde{\theta}}[\ell^i_t(\theta, z)]\label{eqn:amf-distr}
    \end{align}
    \STATE AMF Adversary chooses $z_t \in \cZ_t$.
    \STATE AMF Learner plays $\theta_t \sim \tilde{\theta}_t$. 
\ENDFOR
\end{algorithmic}
\caption{Algorithm 2 from \cite{noarov2021online}}
\label{alg:amf-alg}
\end{algorithm}

\begin{theorem}[Theorem A.1 and A.2 of \cite{noarov2021online}]
\label{thm:minmax-game-reget}
    For any $T \ge \ln(d)$, the learner that plays according to Algorithm 2 of \cite{noarov2021online} against any Adversary achieves with appropriately chosen learning rate obtains AMF regret bounded as follows
    \[
        \E_{\{\theta_t\}_{t=1}^T}\left[R_T\right] \le 4C\sqrt{T \ln d}.
    \] 
    And with probability $1-\rho$, it achieves
    \[
        R_T \le 8C\sqrt{T \ln\left(\frac{d}{\rho}\right)}.
    \]
\end{theorem}

Next, we construct a multiobjective optimization problem within this framework that will imply $O(\sqrt{T})$ omniprediction bounds by using a structural result about proper loss functions recently proven by \cite{kleinberg2023u}.

For any $v \in [0,1]$, let $\ell_v$ be a loss function defined as follows:
\[
    \ell_v(y, \hp)=(v-y) \cdot \sign(\hp-v)
\]
where $\sign(x) = 1$ if $x \ge 0$ and $-1$ otherwise. Let us also write $\cLv = \{\ell_v: v \in [0,1]\}$ to denote a collection of such loss functions over $v \in [0,1]$ and $\tcLv^{m'} = \{\ell_v: v \in [\frac{1}{m'}]\}$ to denote the discretized version of it where we describe how to set $m'$ later.

First, we note that $\ell_v(y, q)$ is a proper-scoring rule:
\begin{lemma}
For any $v \in [0,1]$ and $p$, 
\[
    p \in \argmin_{a \in [0,1]} \E_{y \sim \Ber(p)}[\ell_v(y, a)].
\]
\end{lemma}
\begin{proof}
Note that
    \begin{align*}
        \E_{y \sim \Ber(p)}[\ell_v(y, a)] = (v - p) \cdot \sign(a-v).
    \end{align*}
A prediction $a$ minimizes the expected loss if the sign of $a-v$ is opposite that of $v-p$. $a = p$ satisfies this criterion.
\end{proof}

As a result of $\ell_v$ being a proper scoring rule, the forecaster's omniprediction regret simplifies to
\begin{align*}
    \cO(\pi_{1:T}, \cLv, \cF) &= \max_{f \in \cF} \max_{\ell_v \in \cLv} \frac{1}{T} \sum_{t=1}^T \ell_v(y_t, k_{\ell_v}(\hp_t)) - \ell_v(y_t, f(x_t))\\
    &= \max_{f \in \cF} \max_{\ell_v \in \cLv} \frac{1}{T} \sum_{t=1}^T \ell_v(y_t, \hp_t) - \ell_v(y_t, f(x_t)).
\end{align*}

We now argue that omniprediction with respect to $\tilde{\cLv}^m$ implies omniprediction with respect to $\cLv$.
\begin{lemma}
Suppose $m' > m$. Fix some family of boolean predictors $\cF: \cX \to \{0,1\}$. We have that for all transcripts $\pi_{1:T}$
    \[
        \cO(\pi_{1:T},\cLv, \cF) \le \cO(\pi_{1:T},\tcLv^m, \cF) + \frac{2T}{m'}.
    \]
\end{lemma}

\proof
Fix $f \in \cF$, $v \in [0,1]$, and $\pi_{1:T}$. Suppose $v \in (\frac{i}{m}, \frac{i+1}{m}]$. Then among $[\frac{1}{m'}]$, choose $v'$ so that it falls within $(\frac{i}{m}, \frac{i+1}{m}]$ and is closest to $v$:
\[
    v' = \argmin_{a \in [\frac{1}{m'}], a \in (\frac{i}{m}, \frac{i+1}{m}]} |v- a|.
\]

This guarantees that there doesn't exist any $\hp \in [\frac{1}{m}]$ that falls between $v$ and $v'$, meaning for any $\hp \in [\frac{1}{m}]$, $v$ and $v'$ always fall on the same side relative to $\hp$. Hence, we have for any $\hp \in [\frac{1}{m}]$
\[
    \sign(\hp - v) = \sign(\hp - v').
\]

Therefore, we have for any $\hp \in [1/m]$
\begin{align*}
    |\ell_v(y, \hp) - \ell_{v'}(y, \hp)| \le |\sign(\hp-v) (v-v')| \le \frac{1}{m'}.
\end{align*}

Writing $S(p, b) = \{ t \in[T]: \hp_t = p, f(x_t) = b \}$, we can show that
    \begin{align*}
        &\sum_{t=1}^T \ell_v(y_t, \hp_t) - \ell_v(y_t, f(x_t)) \\
        &=\sum_{p \in \cP} \sum_{b \in \{0, 1\}}\sum_{t \in S(p, b)} \ell_v(y_t, p) - \ell_v(y_t, b) \\
        &\le \sum_{p \in \cP} \sum_{b \in \{0, 1\}}\sum_{t \in S(p, b)} \ell_{v'}(y_t, p) - \ell_{v'}(y_t, b) + \frac{2T}{m'}.&\qed
    \end{align*}

Therefore, it suffices to bound the omniprediction regret with respect to $\tcLv^{m'}$. 

Now, define the following loss function $\ell_t$ for the AMF approach in the following manner: we have a coordinate for each $v \in [\frac{1}{m'}]$, $f \in \cF$, $a \in [\frac{1}{m}]$ and $b \in \{0,1\}$ where 
\[
    \ell^{f, a, b, v}_t(\theta_t, z_t) = \ind[f(x_t) = b] \cdot \left(\E_{y_t \sim z_t}[\ell_v(y_t, \theta_t) - \ell_v(y_t, a)]\right).
\]

\begin{algorithm}[H]
\begin{algorithmic}[1]
\FOR{$t=1, \dots, T$}
    \STATE Forecaster $\forecast_{\cV}$ observes the feature vector $x_t \in \cX$.
    \STATE Construct the AMF-Learner's environment as $\Theta_t = \cP$, $\cZ_t = \Delta(\cY)$, and
    $\ell_t$ as defined in \eqref{eqn:amf-loss}.
    \STATE After observing the environment $(\Theta_t, \cZ_t, \ell_t)$, AMF-Learner chooses a random strategy $\tilde{\theta}_t$ \\ according to Algorithm 2 of \cite{noarov2021online}.
    \STATE Adversary $\adv$ chooses $p_t \in [0,1]$ which is the the probability with which $y_t=1$.
    \STATE AMF Learner chooses $\theta_t \in \cP$ by sampling from $\tilde{\theta}_t \in \Delta(\cP)$.
    \STATE Forecaster $\forecast_{AMF}$ makes the forecast $\hp_t=\theta_t$ and observes $y_t \sim p_t$.
    \STATE Set the action played by the AMF Adversary's strategy as $z_t = y_t$ and have AMF Learner suffer $\ell_t(\theta_t, z_t)$.
\ENDFOR
\end{algorithmic}
\caption{Forecaster $\forecast_{\mathcal{V}}(\cF)$}
\label{alg:amf-v-forecaster}
\end{algorithm}

We first show how to bound the Adversary-Moves-First value of the game for the loss function defined above.
\begin{lemma}
\label{lem:v-forecast-amf-value} 
\[
   w^A_t=\max_{z \in \Delta(\cY)} \min_{\hp \in [\frac{1}{m}]} \max_{f\in \cF, a \in [\frac{1}{m}], b \in \{0,1\}, v \in [\frac{1}{m'}]}\ind[f(x_t) = b] \cdot \left(\E_{y_t \sim z_t}[\ell_v(y_t, \hp) - \ell_v(y_t, a)]\right) \le \frac{2}{m}.
\]
\end{lemma}

\proof
    We need to bound
    \[
        \max_{z \in \Delta(\cY)} \min_{\hp \in [\frac{1}{m}]} \max_{f\in \cF, a \in [\frac{1}{m}], b \in \{0,1\}, v \in [\frac{1}{m'}]}\ind[f(x_t) = b] \cdot \left(\E_{y_t \sim z_t}[\ell_v(y_t, \hp) - \ell_v(y_t, a)]\right).
    \]

    Given any $z \in \Delta(\cY)$, consider \[
        \hp = \argmin_{a \in [\frac{1}{m}]} |a - \E_{y \sim z}[y]|.
    \]
    First, we show that for any $v \in [1/m']$ and $a \in [1/m]$,
    \begin{align*}
        \E_{y \sim z}[\ell_v(y, \hp) - \ell(y, a)] \le \frac{2}{m}.
    \end{align*}

    For any $v \in [\frac{1}{m'}]$ such that $|v-\E_{y \sim z}[y]| \le \frac{1}{m}$. It is immediate that
    \[|(
     v-\E_{y \sim z}[y]) \cdot \sign(\hp - v)| \le \frac{1}{m}\quad\text{and}\quad |(v-\E_{y \sim z}[y]) \cdot \sign(a - v)| \le \frac{1}{m}.
     \]
    Let us now focus on $v \in [\frac{1}{m'}]$ such that $|v-\E_{y \sim z}[y]| > \frac{1}{m}$. Because $\ell_v$ is a proper scoring rule, we have
    \begin{align*}
        \E_{y \sim z}[\ell_v(y, \hp) - \ell(y, a)] \le \E_{y \sim z}\left[\ell_v(y, \hp) - \ell\left(y, \E_{y\sim z}[y]\right)\right].
    \end{align*}
    
    Note that $\hp$ and $\E_{y\sim z}[y]$ both fall on the same side with respect to $v$ as $|v-\E_{y\sim z}[y]| > \frac{1}{m}$ but $|\hp - \E_{y \sim z}[y]| \le \frac{1}{m}$. Therefore,
    \[\sign(\hp-v)=\sign\left(\E_{y \sim z}[y] - v\right), \text{ implying } \ell_v(y, \hp) = \ell\left(y, \E_{y\sim z}[y]\right).
    \]

    Therefore, given any $z \in \Delta(\cY)$, $\hp$ constructed as above satisfies 
  \begin{align*}
        \max_{f\in \cF, a \in [\frac{1}{m}], b \in \{0,1\}, v \in [\frac{1}{m'}]}\ind[f(x_t) = b] \cdot \left(\E_{y_t \sim z_t}[\ell_v(y_t, \hp) - \ell_v(y_t, a)]\right) \le \frac{2}{m}.&\qed
\end{align*}

Invoking Theorem~\ref{thm:minmax-game-reget} with Lemma~\ref{lem:v-forecast-amf-value} that bounds the AMF value yields the following:
\begin{theorem}
\label{thm:conditional-sub-sequence-guarantee}
Suppose $m' > m$. Forecaster $\cF_\cV(\cF)$ as described in Algorithm~\ref{alg:amf-v-forecaster} achieves
    \[
        \E_{\{\hp_t\}_{t=1}^T}\left[\max_{v \in [\frac{1}{m'}], f \in \cF, a \in [\frac{1}{m}], b \in \{0, 1\}}  \sum_{t=1}^T \ind[f(x_t) = b] \cdot (\ell_v(y_t, \hp_t) - \ell_v(y_t, a))\right] \le 8\sqrt{T \ln(2|\cF|{m'}^2)} + \frac{2T}{m}.
    \]
    and with probability $1-\rho$, 
    \[
        \max_{v \in [\frac{1}{m'}], f \in \cF, a \in [\frac{1}{m}], b \in \{0, 1\}} \sum_{t=1}^T \ind[f(x_t) = b] \cdot (\ell_v(y_t, \hp_t) - \ell_v(y_t, a)) \le 16\sqrt{T \ln\left(\frac{2|\cF|{m'}^2}{\rho}\right) } + \frac{2T}{m}.
    \]
\end{theorem}
\begin{proof}
    Note that for each coordinate $f \in \cF, v \in [\frac{1}{m'}], a \in [\frac{1}{m}], b \in \{0,1\}$, the loss function $\ell_t$ is always bounded
    \[
        \left|\ind[f(x_t) = b] \cdot \left(\E_{y_t \sim z_t}[\ell_v(y_t, \theta_t) - \ell_v(y_t, a)]\right) \right| \le 2.
    \]

    Theorem~\ref{thm:minmax-game-reget} and Lemma~\ref{lem:v-forecast-amf-value} together yield
    \[
        \max_{v \in [\frac{1}{m'}], f \in \cF, b \in \{0, 1\}} \E_{\{\hp_t\}_{t=1}^T}\left[\sum_{t=1}^T \ind[f(x_t) = b] \cdot (\ell_v(y_t, \hp_t) - \ell_v(y_t, a))\right] \le 8\sqrt{T \ln(2|\cF|{m'}^2)} + \frac{2T}{m}.
    \]

    The high probability bound follows the same way.
\end{proof}

One immediate corollary is that we can achieve $O(\sqrt{T})$ omniprediction for the set of loss functions $\cLv$ that serve as a lower bound of $\Omega(T^{0.528})$ for swap-omniprediction --- See Section~\ref{sec:lb}. 
\begin{corollary}
Suppose $m' > m$. Fix some family of boolean functions $\cF: \cX \to \{0,1\}$. Then, Forecaster $\cF_\cV(\cF)$ as described in Algorithm~\ref{alg:amf-v-forecaster} achieves with probability $1-\rho$ over the randomness of the transcript
    \[
        \cO(\pi_{1:T},\cLv, \cF) \le \frac{1}{T} \left(16\sqrt{T \ln\left(\frac{2|\cF|{m'}^2}{\rho}\right) } + \frac{2T}{m} + \frac{2T}{m'}\right).
    \]
Setting $m = \sqrt{T}$ and $m' = 2m$, we have
\[
    \cO(\pi_{1:T}, \cLv, \cF) = \tilde{O}(T^{-0.5}).
\]
\end{corollary}
\proof For a given transcript, 
    \begin{align*}
        &\max_{\ell_v \in \tcLv^{m'}, f \in \cF} \sum_{t=1}^T \ell_v(y_t, \hp_t) - \ell_v(y_t, f(x_t)) \\ 
        &=\max_{\ell_v \in \tcLv^{m'}, f \in \cF} \sum_{b \in \{0, 1\}} \sum_{t \in [T]: f(x_t) = b} \ell_v(y_t, \hp_t) - \ell_v(y_t, b) \\
        &\le 2 \max_{\ell_v \in \tcLv^{m'}, f \in \cF, b \in \{0, 1\}} \sum_{t \in [T]: f(x_t) = b} \ell_v(y_t, \hp_t) - \ell_v(y_t, b) . &\qed
    \end{align*}

In fact, we can further strengthen the above result with the following theorem from \cite{kleinberg2023u}.
\begin{theorem}[Theorem 8 of \cite{kleinberg2023u}]
\label{thm:u-forecast-regret}
For any proper loss function $\ell$ bounded in $[-1,1]$, 
    \[
        \sum_{t=1}^T \ell(y_t, \hp_t) - \ell(y_t, \beta) \le 2 \max_{\ell_v \in \cLv}  \left(\sum_{t=1}^T \ell_v(y_t, \hp_t) - \ell_v\left(y_t, \beta\right) \right).
    \]
where $\beta = \frac{1}{T} \sum_{t=1}^T y_t.$
\end{theorem}

Because our regret guarantee in Theorem~\ref{thm:conditional-sub-sequence-guarantee} holds conditionally on the value of $f(x_t)$ values, we can use Theorem~\ref{thm:u-forecast-regret} from \cite{kleinberg2023u} to show that we can in fact guarantee omniprediction with respect to proper losses: i.e. $p \in \argmin_{\hp}\E_{y \sim \Ber(p)}[\ell(y, \hp)]$. In fact, in Corollary~\ref{cor:monotone-losses}, we argue how to generalize to a more general set of losses. 

\begin{theorem}
\label{thm:omni-proper}
Suppose $m' > m$. Fix some finite family of boolean functions $\cF: \cX \to \{0,1\}$. Forecaster $\cF_\cV(\cF)$ as described in Algorithm~\ref{alg:amf-v-forecaster} achieves the following with probability $1-\rho$: for any proper loss $\ell$ (bounded in $[-1,1]$)
    \begin{align*}
        \sum_{t=1}^T \ell(y_t, k_\ell(\hp_t)) - \ell(y_t, f(x_t))\le 32\sqrt{T \ln\left(\frac{2|\cF|{m'}^2}{\rho}\right) } + \frac{4T}{m} 
    \end{align*}
    Setting $m=\sqrt{T}$ and $m'=2m$ guarantees
    \[
        \sum_{t=1}^T \ell(y_t, k_\ell(\hp_t)) - \ell(y_t, f(x_t)) = \tilde{O}(\sqrt{T}).
    \]
\end{theorem}
\begin{proof}
Fix any proper scoring rule $\ell$ and $f \in \cF$.  Write $R(b) = \{t \in [T]: f(x_t) = b\}$. For each $t \in [T]$, we also write $\beta_b = \frac{1}{|R(b)|}\sum_{t \in R(b)} y_t$ and $\tilde{\beta}_b = \argmin_{a \in [\frac{1}{m}]} |a - \beta_b|$.

As noted in the proof of Lemma~\ref{lem:v-forecast-amf-value}, we can show that
\begin{align}
    \left|\E_{y \sim \Ber(\beta_b)}[\ell_v(y, \beta_b)] - \E_{y \sim \Ber(\beta_b)}[\ell_v(y, \tilde{\beta}_b)]\right| \le \frac{2}{m}\label{eqn:ell_v-closeness}
\end{align}
If $|v - \E_{y \sim \Ber(\beta_b)}| \le \frac{1}{m}$, we have $|\E_{y \sim \Ber(\beta_b)}[\ell_v(y, \beta_b)]| \le \frac{1}{m}$ and $|\E_{y \sim \Ber(\beta_b)}[\ell_v(y, \tilde{\beta}_b)]| \le \frac{1}{m}$. Otherwise, $\E_{y \sim \Ber(\beta_b)}[\ell_v(y, \beta_b)] = \E_{y \sim \Ber(\beta_b)}[\ell_v(y, \tilde{\beta}_b)]$ as $\sign(\beta_b - v) = \sign(\tilde{\beta}_b -v)$; note that $|v - \beta_b| > \frac{1}{m}$ but $|\beta_b - \tilde{\beta}_b| \le \frac{1}{m}$.

Then, we can show that with probability $1-\rho$,
    \begin{align*}
        &\sum_{t=1}^T \ell(y_t, k_\ell(\hp_t)) - \ell(y_t, f(x_t))\\
        &=\sum_{b \in \{0,1\}} \sum_{t \in R(b)} \ell(y_t, \hp_t) - \ell(y_t, b)\\
        &\le \sum_{b \in \{0,1\}} \sum_{t \in R(b)} \ell(y_t, \hp_t) - \ell(y_t, \beta_b)&\text{($\ell$ is a proper scoring rule)}\\
        &\le 2\left(\sum_{b \in \{0,1\}} \sum_{t \in R(b)} \ell_v(y_t, \hp_t) - \ell_v(y_t, \beta_b)\right) &\text{(Theorem~\ref{thm:u-forecast-regret})}\\
        &\le 2\left(\sum_{b \in \{0,1\}} \sum_{t \in R(b)} \ell_v(y_t, \hp_t) - \ell_v(y_t, \tilde{\beta}_b)\right) + \frac{2T}{m}&\text{(Equation~\eqref{eqn:ell_v-closeness})}\\
        &\le 32\sqrt{T \ln\left(\frac{2|\cF|{m'}^2}{\rho}\right) } + \frac{2T}{m}  + \frac{2T}{m}.&\text{(Theorem~\ref{thm:conditional-sub-sequence-guarantee})}
    \end{align*}
\end{proof}

Now, we show how to generalize the above omniprediction to a more general family of losses:  all bi-monotone losses $\ell$. We say a loss function $\ell$ is bi-monotone if
$\ell(1, \cdot)$ is monotonically increasing and $\ell(0, \cdot)$ is monotonically decreasing. 
\begin{corollary}
\label{cor:monotone-losses}    
With $m=\sqrt{T}$ and $m' = 2m$, Forecaster $\cF_\cV(\cF)$ as described in Algorithm~\ref{alg:amf-v-forecaster} achieves the following with probability $1-\rho$ over the randomness of the transcript: for any bi-monotone loss $\ell$ bounded in $[-1,1]$
    \begin{align*}
        \sum_{t=1}^T \ell(y_t, k_\ell(\hp_t)) - \ell(y_t, f(x_t))\le \tilde{O}(\sqrt{T}).
    \end{align*}
\end{corollary}
\proof
Note that for any monotone loss function $\ell$, we have $\ell(y, k_\ell(0)) = \ell(y, 0)$ and $\ell(y, k_\ell(1)) = \ell(y, 1)$. 

Writing $\tilde{\ell}(y, \hp) = \ell(y, k_\ell(\hp))$, we have
    \begin{align*}
        &\ell(y, k_\ell(\hp)) - \ell(y, f(x))\\
        &=\ell(y, k_\ell(\hp)) - \ell(y, k_\ell(f(x)))\\
        &=\tilde{\ell}(y, \hp) - \tilde{\ell}(y, f(x)).
    \end{align*}

Note that $\tilde{\ell}$ is a proper loss function as a result of the optimal processing function $k_\ell$
\[
    p \in \argmin_{\hp} \E_{y \sim \Ber(p)}[\tilde{\ell}(y, \hp)].
\]

Theorem~\ref{thm:omni-proper} tells us that Forecaster $\forecast_{\cV}$ gets $O(\sqrt{T})$ regret with respect to all proper loss functions, which includes $\tilde{\ell}$ as well. Therefore, we have
\begin{align*}
    \sum_{t=1}^T \ell(y_t, k_\ell(\hp_t)) - \ell(y_t, f(x_t))\le \tilde{O}(\sqrt{T}).&\qed
\end{align*}

\subsection{A Lower Bound for Swap Omniprediction}
\label{sec:lb}

Consider $\cL$, the class of convex loss functions that are $\lips$-Lipschitz in both the arguments, which we denote with $\cLcvx^\lips$. We show that swap omniprediction with respect to $(\cLcvx^\lips,\cF)$ implies $L_1$-multicalibration with respect to class $\cF$. In fact, our lower bound holds for $\cLv$, the restricted class of convex loss functions on which we based our omniprediction upper bound in Section \ref{sec:vcall}, and so our lower bound serves to formally separate the rates obtainable for omniprediction and for swap omniprediction. We assume that the benchmark class $\cF$ contains the constant $0$ and $1$ function. In particular, the assumption is satisfied by the class of real-valued linear functions with bounded norm. Along with the lower bound of \cite{qiao2021stronger} on $L_1$ calibration error, the following lemma implies impossibility of achieving  swap-omniprediction with $O(\sqrt{T})$ regret. In fact, it is impossible to achieve regret of $o(T^{0.528})$. 
\begin{lemma}\label{lem:lowerbound}
Assume that $\cF$ contains the constant $0$ and the constant $1$ function. Recall that $\cLcvx^\lips$ is the class of $\lips$-Lipschitz convex loss functions. Then, for $\lips\ge 2$,
\[ \E_{\pi_{1:T}}\left[ \scO(\pi_{1:T}, \cF, \cLcvx^\lips) \right] \ge \Omega\left(\E_{\pi_{1:T}}\left[\bK_1(\pi_{1:T}, 1)\right]-\frac{1}{T}\right).\]
Here, $\bK_1(\pi_{1:T}, 1)$ denotes the $L_1$-calibration error for the transcript $\pi_{1:T}$. 
\end{lemma}
\begin{remark}\label{rem:sum1}
In fact, Lemma \ref{lem:lowerbound} holds for the class of loss functions $\cLv$ defined in Section \ref{sec:vcall}. Recall that $\cLv = \{\ell_v: v \in [0,1]\}$ where $\ell_v$ is proper-scoring rule defined as follows:
\[
    \ell_v(y, q)=(v-y) \cdot \sign(q-v).
\]
That is,
\[ \E_{\pi_{1:T}}\left[ \scO(\pi_{1:T}, \cF, \cLv) \right] \ge \Omega\left(\E_{\pi_{1:T}}\left[\bK_1(\pi_{1:T}, 1)\right]-\frac{1}{T}\right).\]
This gives a strict separation between omniprediction and swap-omniprediction.
\end{remark}
\begin{proof}[Proof of Lemma \ref{lem:lowerbound}]
For all $v\in[0,1]$, let $\el_v:\{0,1\}\times [0,1]\rightarrow [-1,1]$ be a loss function defined as follows:
$$\el_v(y,q)=(v-y)(2q-1).$$
It is easy to see that $\forall v, \el_v\in \cLcvx^\lips$, which contains the class of $2$-Lipschitz convex loss functions. Next, we compute $k_{\el_v}(p)$, which is the optimal post-processing function for loss $\el_v$ when $y$ is drawn from $\Ber(p)$. 
\begin{align*}
k_{\el_v}(p)= \arg\min_{q \in [0, 1]} \E_{y \sim \Ber(p)}[\el_v(y, q)]=\arg\min_{q \in [0, 1]} (v-p)(2q-1)=\begin{cases}0\text{ if } p< v\\
1\text{ if } p\ge v.\end{cases}
\end{align*}
For $p=v$, $k_{\el_v}(p)$ can be arbitrarily chosen, but w.l.o.g. we can assume it to be $ 1$ for the proof of this lemma. We can compute $\el_{v}(y,k_{\el_v}(p))$ as follows:
\begin{equation}\label{eq:sum2}
\el_{v}(y,k_{\el_v}(p))=\begin{cases}
-(v-y)\text{ if } p< v\\
(v-y)\text{ if } p\ge v.
\end{cases}
\end{equation}
Therefore, $\el_v(y,k_{\el_v}(p))=(v-y)\cdot \sign(p-v)$, which is exactly equal to proper-scoring rule $\ell_v$ mentioned in Remark \ref{rem:sum1}. Hence, the rest of the proof follows for loss class $\cLv$.

For a given transcript $\pi_{1:T}$, $L_1$-multicalibration error with respect to the constant function $1$ can be written as
\[\bK_1(\pi_{1:T}, 1)=\sum_{p\in\cP}{\frac{n(\pi_{1:T},p)}{T}\left|\bK(\pi_{1:T},p,1)\right|},\]
where $\bK(\pi_{1:T},p,1)=\frac{1}{n(\pi_{1:T},p)}\left(\sum_{t\in S(\pi_{1:T},p)}(y_t-\hp_t)\right)$ is the calibration error for the time-steps when forecaster's prediction value was $p$, that is, $\hp_t=p$.

Next, we will show that 
\begin{equation}\label{eq:sum1}
\scO(\pi_{1:T}, \cF, \cLcvx^\lips)\ge 2\bK_1(\pi_{1:T},1)-\frac{2}{T},
\end{equation}
and the lemma statement follows by taking expectation over the transcripts. 
\begin{align*}
\scO(\pi_{1:T}, \cF, \cLcvx^\lips)&=\max_{\{\ell_p\in \cLcvx^\lips\}_{p \in \cP}, \{f_p\in \cF\}_{p \in \cP}} \scO(\pi_{1:T}, \{\ell_p\}_{p \in \cP}, \{f_p\}_{p \in \cP})\\
&\ge\max_{\forall p\in\cP, \;\ell_p\in\{\el_{(p+1/T)}, \el_{(p-1/T)}\}, f_p\in \{0,1\}}  \scO(\pi_{1:T}, \{\ell_p\}_{p \in \cP}, \{f_p\}_{p \in \cP}).
\end{align*}
For all $p\in\cP$, we compute the swap-omniprediction regret with respect to the following values of $\ell_p$ and $f_p$.
\begin{enumerate}
\item If $\bK(\pi_{1:T},p,1)>0$, let $\ell_p=\el_{(p+1/T)}$ and $f_p=1$.
\item Else if $\bK(\pi_{1:T},p,1)<0$, let $\ell_p=\el_{(p-1/T)}$ and $f_p=0$. 
\item Else $\ell_p=\el_{p}$ and $f_p=1$.
\end{enumerate}
To summarize, we take $\ell_p=\el_{\left(p+\frac{1}{T}\cdot\frac{\bK(\pi_{1:T},p,1)}{|\bK(\pi_{1:T},p,1)|}\right)}$ and $f_p=1-\frac{|\bK(\pi_{1:T},p,1)|-\bK(\pi_{1:T},p,1)}{2|\bK(\pi_{1:T},p,1)|}$.

Recall that $\scO(\pi_{1:T}, \{\ell_p\}_{p \in \cP}, \{f_p\}_{p \in \cP})$ is equal to
\begin{align*}
\sum_{p \in \cP} \frac{n(\pi_{1:T}, p)}{T} \left(\frac{1}{n(\pi_{1:T}, p)}\sum_{t \in S(p)} \ell_p(y_t, k_{\ell_p}(\hp_t)) - \frac{1}{n(\pi_{1:T}, p)}\sum_{t \in S(p)} \ell_p(y_t, f_p(x_t)) \right).
\end{align*}
Next, we show that for all $p\in \cP$, 
\[\left(\frac{1}{n(\pi_{1:T}, p)}\sum_{t \in S(p)} \ell_p(y_t, k_{\ell_p}(\hp_t)) - \frac{1}{n(\pi_{1:T}, p)}\sum_{t \in S(p)} \ell_p(y_t, f_p(x_t))\right)\;\ge\; 2\left(|\bK(\pi_{1:T},p,1)|-\frac{1}{T}\right),\] 
which implies Equation~\eqref{eq:sum1}: 
\begin{align*}
&\sum_{t \in S(p)} \ell_p(y_t, k_{\ell_p}(\hp_t)) - \sum_{t \in S(p)} \ell_p(y_t, f_p(x_t))\\
&=\sum_{t \in S(p)} \ell_p(y_t, k_{\ell_p}(p)) - \sum_{t \in S(p)} \ell_p(y_t, f_p(x_t))\\
&=\sum_{t \in S(p)} \el_{\left(p+\frac{1}{T}\cdot\frac{\bK(\pi_{1:T},p,1)}{|\bK(\pi_{1:T},p,1)|}\right)}\left(y_t, k_{\el_{\left(p+\frac{1}{T}\cdot\frac{\bK(\pi_{1:T},p,1)}{|\bK(\pi_{1:T},p,1)|}\right)}}(p)\right) \\
&\;\;\;\;\;- \sum_{t \in S(p)} \el_{\left(p+\frac{1}{T}\cdot\frac{\bK(\pi_{1:T},p,1)}{|\bK(\pi_{1:T},p,1)|}\right)}\left(y_t, 1-\frac{|\bK(\pi_{1:T},p,1)|-\bK(\pi_{1:T},p,1)}{2|\bK(\pi_{1:T},p,1)|}\right)\\
&=\sum_{t \in S(p)}{\left(p+\frac{1}{T}\cdot\frac{\bK(\pi_{1:T},p,1)}{|\bK(\pi_{1:T},p,1)|}-y_t\right)\cdot \sign\left(-\frac{1}{T}\cdot\frac{\bK(\pi_{1:T},p,1)}{|\bK(\pi_{1:T},p,1)|}\right)}\tag{using Equation~\eqref{eq:sum2}}\\
&\;\;\;\;\;- \sum_{t \in S(p)}{ \left(p+\frac{1}{T}\cdot\frac{\bK(\pi_{1:T},p,1)}{|\bK(\pi_{1:T},p,1)|}-y_t\right)\cdot  \left(1-\frac{|\bK(\pi_{1:T},p,1)|-\bK(\pi_{1:T},p,1)}{|\bK(\pi_{1:T},p,1)|}\right)}\\
&=\sum_{t \in S(p)}{\left(p+\frac{1}{T}\cdot\frac{\bK(\pi_{1:T},p,1)}{|\bK(\pi_{1:T},p,1)|}-y_t\right)\cdot \left(1-\frac{|\bK(\pi_{1:T},p,1)|+\bK(\pi_{1:T},p,1)}{|\bK(\pi_{1:T},p,1)|}\right)}\\
&\;\;\;\;\;- \sum_{t \in S(p)}{ \left(p+\frac{1}{T}\cdot\frac{\bK(\pi_{1:T},p,1)}{|\bK(\pi_{1:T},p,1)|}-y_t\right)\cdot  \left(1-\frac{|\bK(\pi_{1:T},p,1)|-\bK(\pi_{1:T},p,1)}{|\bK(\pi_{1:T},p,1)|}\right)}\\
&=\sum_{t\in S(p)}{\left(p+\frac{1}{T}\cdot\frac{\bK(\pi_{1:T},p,1)}{|\bK(\pi_{1:T},p,1)|}-y_t\right)\cdot \frac{-2\bK(\pi_{1:T},p,1)}{|\bK(\pi_{1:T},p,1)|}}\\
&\ge \sum_{t\in S(p)}{\left(y_t-p\right)\cdot \frac{2\bK(\pi_{1:T},p,1)}{|\bK(\pi_{1:T},p,1)|}} -\sum_{t\in S(p)}\frac{2}{T}\\
&=2\left|\sum_{t\in S(p)}{\left(y_t-p\right)}\right| -\frac{2n(\pi_{1:T},p)}{T}.&
\end{align*}
\end{proof}

Now, we state \cite{qiao2021stronger}'s $L_1$-calibration lower bound result.
\begin{theorem}[\cite{qiao2021stronger}]
\label{thm:l1-cal-lb}
There exists an adversary $\adv$ such that  for every Forecaster $\forecast$, we have
    \[
        \E_{\pi_{1:T}}[\bK_1(\pi_{1:T}, 1)] = \Omega\left(\frac{T^{0.528}}{T}\right)
    \]
\end{theorem}

Lemma~\ref{lem:lowerbound} and Theorem~\ref{thm:l1-cal-lb} together result in a lower bound for the swap omniprediction:
\begin{corollary}
\label{cor:swap-lb}
There exists an adversary $\adv$ such that  for every Forecaster $\forecast$, we have 
    \[ \E_{\pi_{1:T}}\left[ \scO(\pi_{1:T}, \cF, \cLcvx^\lips) \right] =\Omega\left(\frac{T^{0.528}}{T}\right).\]
\end{corollary}

\section{An Extension: Online Oracle-efficient Quantile Multicalibration and  Multivalid Conformal Prediction}
\label{sec:conformal-extension}

In this section, we observe that the techniques that we develop in this paper are largely independent of the goal of \emph{mean} multicalibration and \emph{squared error} online regression. The connection between mean multicalibration and the squared error loss function comes from the fact that squared error is a proper scoring rule --- i.e. it is minimized by or \emph{elicits} the mean. More generally, there is a direct connection between multicalibration for a generic distributional property $\Gamma$ and the regression loss that elicits that property (i.e. is minimized at that property) \cite{noarov2023scope}. In this section, we show how to extend our results to give an online oracle efficient algorithm for \emph{quantile} multicalibration, as studied by \cite{gupta2021onlinevalid,bastani2022practical}. Informally, squared error is to mean multicalibration as pinball loss is to quantile multicalibration, a connection that was first formalized in \cite{jung2022batch}. At a high level, we can therefore similarly reduce the problem of online quantile multicalibration to the problem of online learning with respect to pinball loss. One primary reason to study online quantile multicalibration is that it has a direct application to the problem of online conformal prediction, which we define next.

We describe the online conformal prediction problem following  \cite{bastani2022practical}. Fix a bounded conformal score function $s_t: \cX \times \cY \to \R$ which can change in arbitrary ways between rounds $t \in [T]$. Without loss of generality, we assume that the scoring function takes values in the unit interval: 
$s_t(x, y) \in \cS$ where $\cS = [0,1]$ for any $x \in \cX, y \in \cY$, and $t\in [T]$. Fix some target coverage rate $q$. In each round $t \in [T]$, an interaction between a conformal learner and an adversary proceeds as follows:
\begin{enumerate}
    \item The conformal learner chooses a conformal score function $s_t:\cX \times \cY \rightarrow [0, 1]$, which may be observed by the adversary.
    \item The \emph{adversary} chooses a joint distribution over feature vectors $x_t \in \cX$ and labels $y_t \in \cY$. The learner receives $x_t$ (a realized feature vector), but no information about the label $y_t$. 
    \item The learner produces a conformity threshold $\hp_t \in \cPq$ where $\cPq = [\frac{1}{m}]$ as before. This corresponds to a prediction set which the learner outputs: \[
        \Tau_t(x_t) = \{y  \in \cY: s_t(x_t, y) \le \hp_t\}.
    \]
    \item The learner then learns the realized label $y_t$. 
\end{enumerate}
Ideally, the learner wants to produce prediction sets $\Tau_t(x_t)$ that cover the true label $y$ with probability $q$ over the randomness of the adversary's unknown label distribution:
    $\Pr_{y | x_t}[y \in \Tau_t(x_t)] \approx q$. Because of the structure of the prediction sets, this is equivalent to choosing a conformity threshold $q_t$ such that over the randomness of the adversary's unknown label distribution: $\Pr_{y | x_t}[s_t(x_t, y) \le q_t ] \approx q$.
Because the adversary may choose the label distribution with knowledge of the conformal score function, we will elide the particulars of the conformal score function and the distribution on labels $y_t$ in our derivation, and instead equivalently imagine the adversary  directly choosing a distribution over conformal scores $s_t$ conditional on $x_t$ (representing the distribution over conformal scores $s_t(x_t,y_t)$). We may thus view the interaction in the following simplified form:
\begin{enumerate}
    \item The \emph{adversary} chooses a joint distribution over feature vector $x_t \in \cX$ and conformal score $s_t \in  \cS$. The learner receives $x_t$ (a realized feature vector), but no information about $s_t$. 
    \item The learner produces a conformity threshold $\hp_t \in \cPq$.
    \item The learner observes the realized conformal score $s_t$. 
\end{enumerate}
We'll refer to the conditional conformal score distribution at round $t \in [T]$ as $\ts_t$ (or $\ts_t|x_t$ to make it obvious what the realized feature vector is) for convenience. Also, just as in \cite{jung2022batch}, we assume that $\ts_t$ is smooth:

\begin{definition}\cite{jung2022batch}
A conditional nonconformity score distribution $\ts \sim \Delta(\cS)$ is \emph{$\rho$-Lipschitz} if we have
\[
\Pr_{s \sim \ts}[s \leq \tau'] - \Pr_{s \sim \ts}[s \leq \tau] \leq \rho(\tau'-\tau) \quad \text{ for all $0\le \tau \le \tau' \le 1$}.
\]
We say distribution $\cD \in \Delta(\cX \times \cS)$ is $\rho$-Lipschitz if the conditional conformal distribution $\cD_\cS(x)$ is $\rho$-Lipschitz for every $x$ in the support of $\cD$ (i.e. $\cD_\cX(x) > 0$).
\end{definition}

As in Section~\ref{sec:prelim}, we write $\pi_{1:t} = \{(x_\tau, s_\tau, \hp_\tau)\}_{\tau=1}^t$ to denote the realized transcript of the interaction between the learner and the adversary. Similarly, we write $\tpi_{1:t} = \{(x_\tau, \ts_\tau, \hp_\tau)\}_{\tau=1}^t$ to denote the unrealized transcript where $x_t$ and $\hp_t$ are realized but $\ts_t$'s haven't been realized. As before, we write
\begin{align*}
    S(\pi_{1:t}, p) &= S(\tpi_{1:t}, p ) = \{\tau \in [t]: \hp_t = p\} \\
    n(\pi_{1:t}, p) &= n(\tpi_{1:t}, p) = |\{\tau \in [t]: \hp_t = p\}|.
\end{align*}

We write $\cD(\tpi_{1:T})$ to denote the uniform distribution over $\{(x_t, \ts_t)\}_{t=1}^T$ and
$\cD^p(\tpi_{1:T})$ to denote the uniform distribution over $\{(x_t, \ts_t)\}_{t \in S(\pi_{1:T}, p)}$. When it is obvious from the context, we just write $\cD$ and $\cD^p$. Given any distribution $\cD \in \Delta(\cX \times \cS)$, to refer to its marginal distribution over just the feature vectors $\cX$, we write $\cD_{\cX}$. And we write $\cD_\cS(x) \in \Delta(\cS)$ to denote the conformal score distribution of $\cD$ conditioning on the feature vector $x$. When there are multiple arguments, we sometimes curry the arguments: e.g. given $\cD(\tpi_{1:T})$, we write $\cD(\tpi_{1:T}, x) = \cD(\tpi_{1:T})(x)$ to denote the conformal score distribution on conditioning on $x$. 

\begin{definition}
Given $\pi_{1:T}$ that is generated by the conformal learner and the adversary, its quantile-multivalidity error with respect to target quantile $q$, $f \in \cF$ and $p \in \cPq$ is defined as 
\[
    Q(\pi_{1:T}, p, f) = \frac{1}{n(\pi_{1:T}, p)}\sum_{t \in S(\pi_{1:T}, p)} f(x_t) \cdot (q- \ind[s_t \le p]).
\]

Similarly given unrealized transcript $\tpi_{1:T}$, we write its quantile-multivalidity error with respect to $f \in \cF$ and $p \in \cPq$ as: 
\begin{align*}
    Q(\tpi_{1:T}, p, f) &= \frac{1}{n(\pi_{1:T}, p)} \sum_{t \in S(\tpi_{1:T}, p)} f(x_t) \cdot \left(q - \Pr_{s_t \sim \ts_t}[s_t \le p]\right)\\
    &=\E_{x \sim \cD^p_{\cX}(\tpi_{1:T})}\left[f(x) \cdot \left(q - \Pr_{s_t \sim \cD_\cS(\tpi_{1:T}, x)}[s_t \le p]\right)\right]
\end{align*}
As before, we write the $L_2$-swap-quantile-multivalidity error with respect to $\{f_p\}_{p \in \cPq}$ and $L_2$-quantile-multivalidity error with respect to $f$ as 
\begin{align*}
    \overline{sQ}_2(\pi_{1:T}, \{f_p\}_{p \in \cPq}) &= \sum_{p \in \cP} \frac{n(\pi_{1:T}, p)}{T} \left(Q(\pi_{1:T}, p, f_p)\right)^2\\
    \overline{Q}_2(\pi_{1:T}, f) &= \sum_{p \in \cP} \frac{n(\pi_{1:T}, p)}{T} \left(Q(\pi_{1:T}, p, f)\right)^2.
\end{align*}
Also, similarly, we write
\begin{align*}
    \overline{sQ}_2(\tpi_{1:T}, \{f_p\}_{p \in \cPq}) &= \sum_{p \in \cP} \frac{n(\pi_{1:T}, p)}{T} \left(Q(\tpi_{1:T}, p, f_p)\right)^2\\
    \overline{Q}_2(\tpi_{1:T}, f) &= \sum_{p \in \cP} \frac{n(\pi_{1:T}, p)}{T} \left(Q(\tpi_{1:T}, p, f)\right)^2.
\end{align*}
\end{definition}

\subsection{Sketch: Minimizing $L_2$-Swap-Quantile-Multivalidity Error}
Here we sketch out how to minimize $L_2$-swap-quantile-multivalidity error $\overline{sQ}_2$ with techniques that mirror the techniques presented in Section~\ref{sec:oracle-efficient-multical-alg}. Details can be found in Appendix \ref{app:conformal-extension}.

First, just as in Section~\ref{sec:oracle-efficient-multical-alg} where we leverage an online regression oracle for squared loss, we need to rely on an online quantile regression oracle for pinball loss where pinball loss is defined below. Note that just as squared loss is a proper scoring rule for eliciting means, pinball loss is a proper scoring rule for eliciting quantiles. 
\begin{definition}[Pinball Loss]
The pinball loss at target quantile $q$ for prediction $p \in \cPq$ and score $s \in \cS$ is
    \begin{align*}
        \pb_q(p, s) = (s-p)q \cdot 1[s > p] + (p-s)(1-q) \cdot 1[s \leq p].
    \end{align*}
Given a distribution $\ts$ over $\Delta([0,1])$, we similarly write
\[
    \pb_q(p, \ts) = \E_{s \sim \ts}[\pb_q(p,s)].
\]
\end{definition}

We write $\cA: \Psi^* \times \cX \to \cPq$ to denote the online quantile regression oracle. Suppose against any adversarially and adaptively chosen sequence of $\{(x_t, s_t)\}_{t=1}^T$, its pinball loss regret guarantee with respect to any $f \in \cF$ can be bounded as 
\[
    \sum_{t=1}^T \pb_q(\hp_t, s_t) - \sum_{t=1}^T \pb_q(f(x_t), s_t) \le r_\cA(T, \cF).
\]

Because the pinball loss with target quantile $q$ is $q'=\max(q, 1-q)$-Lipschitz, we can round $\cA$'s output to $[\frac{1}{m}]$ and suffer at most an additional $\frac{q'}{m}$ loss from rounding. As before, we also choose a prediction from $[\frac{1}{m}]$ at random with probability $1-\frac{1}{\sqrt{T}}$. In other words, this rounded and randomized  oracle, which we denote by $\tilde{\cA}^m$  as before, has an expected regret guarantee of $r_\cA(T, \cF) + \frac{q'T}{m} + \sqrt{T}$.

Finally, similarly as in Section~\ref{subsec:oracle}, we can borrow ideas from \cite{blum2007external, ito2020tight} to construct a conformal learner such that its contextual swap regret is bounded with high probability against any $\{f_p\}_{p \in \cPq}$ in terms of $r_{\tilde{\cA}^m}(T, \cF)$:
\begin{align*}
    & \sum_{t=1}^T
     \E_{\hp_t \sim \thp_t}\left[
     \pb_q(p, s_t) - \pb_q(f_{\hp_t}(x_t), s_t) | \pi_{1:t-1}\right] \\
     &\;\;\;\;\;\;\;\;\;\;\;\le 
    \left(m r_\cA\left(\frac{T}{m}, \cF\right) + \frac{q'T}{m} + 1\right) + \max(8B, 2\sqrt{B})m C_{\cF_B} \sqrt{\frac{\log(\frac{4m}{\rho})}{T}}.
\end{align*}
Here we simply observe that our algorithmic construction in Section~\ref{subsec:oracle} was agnostic as to the form of the loss function, and so it carries over unchanged when we replace squared loss with pinball loss.

The final piece to our argument is the connection between contextual swap regret with respect to pinball loss and $L_2$-swap-quantile-multivalidity error. Using similar ideas as in Lemma~\ref{lem:cal-error-implies-swap-error} and borrowing some ideas from \cite{jung2022batch}, we can show that if there exists $f \in \cF$ and $p \in \cPq$ such that $Q(\pi_{1:T}, p, f) \ge \alpha$, then there must exist $f' \in \cF$ such that 
\[
    \sum_{t \in S(\pi_{1:T}, p)} \pb_q(p, s_t) - \pb_q(f'(p), s_t)
\]
also scales linearly with $\alpha$. 

While deferring the actual details to Appendix~\ref{app:conformal-extension}, here we go over the overall argument that shows their connection while highlighting the additional arguments needed compared to Lemma~\ref{lem:cal-error-implies-swap-error}. 

We first need to argue that if the multivalidity error with respect to $f$ and $p$ under the uniform distribution over realized (feature, conformal score) pairs $\{(x_t, s_t)\}_{t=1}^T$ is big, then the multivalidity error with respect to $\{(x_t, \ts_t)\}_{t=1}^T$ must be big as well; we can appeal to \cite{block2021majorizing} to show this concentration over all $f \in \cF$. By taking the randomness over $s_t$ into account through $\ts_t$, we now have the conformal score distribution over $\{(x_t, \ts_t)\}_{t=1}^T$ is now continuous and smooth.

Now, we need to modify Lemma 3.1 of \cite{jung2022batch} --- the lemma states that fixing the multivalidity error with respect to some group function $g: \cX \to \{0,1\}$ and level set $p$ by swapping $p$ with $p + \eta $ for the points that belong to group $g$ and are given the prediction $p$ decreases the pinball loss over these points by an amount that depends on the multivalidity error and the smoothness of the distribution. 

We note that they choose $\eta$ such that $p + \eta$ is the exact $q$-quantile for these points, which requires the conformal score distribution to be continuous, and they show the decrease in pinball loss depends on the smoothness of the distribution over the conformal score distribution. Both of these issues of continuity and smoothness of the conformal distribution are taken care of by dealing with the uniform distribution over $\{(x_t, \ts_t)\}_{t=1}^T$ as opposed to uniform distribution over the realized (feature, conformal score) pairs $\{(x_t, s_t)\}_{t=1}^T$.

Note that swapping above corresponds to predicting with $p + \eta \cdot g(x)$ for the points who are given prediction $p$. However, their argument only works for Boolean functions $g: \cX \to \{0,1\}$, but we have soft-membership determined by some $f:\cX \to \R$ which simply reweights the points according to $f(x_t)$: 
\[
    \E_{s_t \sim \ts_t \forall t \in [T]}\left[\frac{1}{n(p)} \sum_{t \in S(p)} f(x_t) \cdot \left(q - 1[s_t \le p] \right)\right].
\]
Even under this re-weighting under $f$, it's easy to see that there exists some $b \in \R$ such that 
\begin{align*}
    \E_{s_t \sim \ts_t \forall t \in [T]}\left[\frac{1}{n(p)} \sum_{t \in S(p)} f(x_t) \cdot \left(q - 1[s_t \le p + b \cdot f(x_t)] \right)\right] 
\end{align*}
is exactly 0 as the above value is monotonic in $b$. Swapping $p$ to $p + b \cdot f(x)$ will decrease the pinball loss with respect to this $f$-reweighted distribution over the points whose prediction was $p$ via the same argument as in Lemma 3.1 of \cite{jung2022batch}. The actual argument presented in the appendix doesn't choose $b$ to set the quantile error to be exactly 0 so as to avoid having $b$ be too big as in Lemma~\ref{lem:cal-error-implies-swap-error}.

By our assumption that $\cF$ is closed under affine transformation, we have $p + b \cdot f \in \cF$. Therefore, if the multivalidity error under uniform distribution over $\{(x_t, \ts_t)\}_{t=1}^T$ with respect to $f$ and $p$ is big, then $f' = p + b \cdot f$ witnesses to the fact that contextual swap regret under the same distribution must be big as well. Finally, we can once again appeal to \cite{block2021majorizing} to argue that contextual swap regret under to the empirical distribution over $\{(x_t, s_t)\}_{t=1}^T$ must be close to that under the uniform distribution over $\{x_t, \ts_t\}_{t=1}^T$ to conclude that contextual swap regret with respect to $\{(x_t, s_t)\}_{t=1}^T$ must be big as well. This concludes the connection between multivalidity error with respect to $f$ and $p$ and contextual swap regret for the pinball loss.

\section{Discussion and Conclusion}
We have given the first oracle efficient algorithms for online multicalibration, online omniprediction, and online multivalid conformal prediction. Our algorithms do not, however, obtain the rates that we would ideally like: $O(T^{-1/2})$. Are these rates obtainable in an oracle efficient way? We leave this as our main open question.

Achieving these rates seemingly requires new techniques: our current techniques give algorithms that obtain the stronger guarantee of \emph{swap}-omniprediction for which we have proven that obtaining an  $O(T^{-1/2})$ rate is impossible. On the other hand, at least for finite binary benchmark classes, we have shown that using different techniques, $O(T^{-1/2})$  rates are obtainable for (vanilla) omniprediction --- so the lower bound is not inherent to the goal of omniprediction; but the techniques we use to obtain these stronger omniprediction bounds seem to require enumerating over the benchmark class $\cF$, and it is not clear how one would implement this same algorithm by reducing to solving an (online) learning problem over $\cF$.

\section{Acknowledgements}
We thank Jon Schneider for helpful discussions and pointing us to \cite{kleinberg2023u}.

\bibliographystyle{alpha}
\bibliography{refs}
\appendix

\begin{appendices}

\section{Missing Details from Section~\ref{sec:prelim}}
\lemmulticalrelationship*
\begin{proof}
This results follows directly from Cauchy-Schwartz inequality: $\E[X Y]^2 \le \E[X^2] \E[Y^2]$. For any $\{f_p\}_{p \in \cP}$,
\begin{align*}
    \left(\sum_{p \in \cP} \frac{n(\pi_{1:T}, p)}{T} \cdot 1 \cdot \left|K(\pi_{1:T}, p, f_p)\right| \right)^2 &\le \sum_{p \in \cP} \left(\frac{n(\pi_{1:T}, p)}{T} \cdot 1^2\right) \left(\sum_{p \in \cP} \frac{n(\pi_{1:T}, p)}{T} \left|K(\pi_{1:T}, p, f_p)\right|^2 \right)\\
    &\le \bsK_2(\pi_{1:T}, \cF).
\end{align*}

The same argument can be used to show $\bK_1(\pi_{1:T}, \cF) \le \sqrt{\bK_2(\pi_{1:T}, \cF)}$.
\end{proof}

\section{Missing Details from Section~\ref{sec:oracle-efficient-multical-alg}}
\label{app:miss-sec-oracle-efficient}
\lemregretbyrounding*
\begin{proof}
For any $y,y' \in [0,1]$ and $m \ge 1$, we have
\begin{align*}
    \left(\round\left(y', \frac{1}{m}\right)-y\right)^2 - (y'-y)^2 &\le \left(\left|y-y'\right| + \frac{1}{m}\right)^2 - |y-y'|^2\\
    &= \left| \frac{2|y-y'|}{m} + \frac{1}{m^2} \right| \\
    &\le \frac{3}{m}.
\end{align*}
 We can then show that for any sequence $\{(x_t, y_t)\}_{t=1}^T$ that is chosen adversarially and adaptively against $\cA^m$,
\begin{align*}
    &\regret\left(\{(x_t, y_t, \hy_t)\}_{t=1}^T, f\right)\\ &=\sum_{t=1}^T \left(\hy_t - y_t\right)^2 - \sum_{t=1}^T (f(x_t) - y_t)^2 \\
    &=\sum_{t=1}^T \left(\round\left(\cA(\{x_\tau, y_\tau\}_{\tau=1}^{t-1}, x_t), \frac{1}{m}\right)- y_t\right)^2 - \sum_{t=1}^T (f(x_t) - y_t)^2\\
    &\le \sum_{t=1}^T (\text{Project}_{[0,1]}(\cA(\{x_\tau, y_\tau\}_{\tau=1}^{t-1}, x_t))- y_t)^2 + \frac{3}{m} - (f(x_t) - y_t)^2 \\
    &\le \sum_{t=1}^T (\cA(\{x_\tau, y_\tau\}_{\tau=1}^{t-1}, x_t)- y_t)^2 + \frac{3}{m} - (f(x_t) - y_t)^2 \\
    &\le r_\cA(T) + \frac{3T}{m}.
\end{align*}

\end{proof}

\lemroundedrandomoracle*
\begin{proof}
Fix any $f \in \cF$. For convenience, write $s_t = \{(x_\tau, y_\tau)\}_{\tau=1}^{t-1}$

\begin{align*}
    &\sum_{t=1}^T \E_{\hy'_t \sim \tcA^m(s_t, x_t)}[(y'_t - \hy_t) - (y'_t - f(x_t))^2 | \pi_{1:t-1}] \\
    &\le \sum_{t=1}^T \Pr[\tcA^m(s_t, x_t) = \cA^m(s_t, x_t) | \pi_{1:t-1}]\cdot \left((\tcA^m(s_t, x_t) - y_t)^2 - (y_t - f(x_t))^2\right) \\
    &\;\;\;\;\;\;+ \Pr[\tcA^m(s_t, x_t) \neq \cA^m(s_t, x_t) | \pi_{1:t-1}] \\
    &\le \sum_{t=1}^T (\tcA^m(s_t, x_t) - y_t)^2 - (y_t - f(x_t))^2 + \frac{1}{T} \\
    &\le \regret_\cA(T, \cF) + \frac{3T}{m} + 1. 
\end{align*}
where the last inequality comes from Lemma~\ref{lem:regret-by-rounding}.
\end{proof}

\section{Missing Details from Section~\ref{sec:omnipredictor}}
\label{app:miss-omnipredictor}
\lemconditionylittlechange*

\begin{proof}
Fix $\pi_{1:T}$, $y \in \cY$, and $\{f_p\}_{p\in\cP}$. We can show that
\begin{align*}
     &\sum_{p \in \cP} \frac{n(p)}{T}\left|\frac{1}{n(p)}\sum_{t \in S(p)} f_p(x_t) y_t - \left(\frac{1}{n(p)}\sum_{t \in S(p)} y_t \right) \left(\frac{1}{n(p)}\sum_{t \in S(p)} f_p(x_t)\right)\right|\\
     &\le \sum_{p \in \cP} \frac{n(p)}{T} \left(\left|\frac{1}{n(p)}\sum_{t \in S(p)} f_p(x_t) y_t - \left(\frac{1}{n(p)}\sum_{t \in S(p)} f_p(x_t) p \right)\right| + \left|p - \frac{\sum_{t \in S(p)} y_t}{n(p)}\right|\right)\\
     &= \sum_{p \in \cP} \frac{n(p)}{T}\left|\frac{\sum_{t \in S(p)} f_p(x_t)( y_t - \hp_t)}{n(p)}\right| + \sum_{p \in \cP}  \frac{n(p)}{T} \left|p - \frac{\sum_{t \in S(p)} y_t}{n(p)}\right|\\
     &\le 2\alpha
 \end{align*}
where the last inequality follows from $\bsK_1(\pi_{1:T}, \cF) \le \alpha$ and Assumption~\ref{ass:all-one}. 

Note that for any $p \in \cP$
\begin{align*}
    &\left|\frac{1}{n(p)}\sum_{t \in S(p)} f_p(x_t) y_t - \left(\frac{1}{n(p)}\sum_{t \in S(p)} y_t \right) \left(\frac{1}{n(p)}\sum_{t \in S(p)} f_p(x_t)\right)\right|\\
    &= \left|\frac{|S(p, 1)|}{n(p)} \left(\frac{1}{|S(p,1)|}\sum_{t \in S(p,1)} f_p(x_t) -  \frac{1}{n(p)}\sum_{t \in S(p)}f_p(x_t) \right)\right|\\
    &= \left|\frac{|S(p,0)|}{n(p)} \left(\frac{1}{|S(p,0)|}\sum_{t \in S(p, 0)} f_p(x_t) -  \frac{1}{n(p)}\sum_{t \in S(p)}f_p(x_t) \right)\right|
\end{align*}
where the second equality follows from $\Pr[y=1](\E[z | y=1] - \E[z]) = \Pr[y=0](\E[z] - \E[z|y=0])$.

In other words, we have 
\begin{align*}
    \sum_{p \in \cP} \frac{n(p)}{T} \left|\frac{|S(p, 1)|}{n(p)} \left(\barf_p(p, 1) -  \barf_p(p) \right)\right|= \sum_{p \in \cP} \frac{n(p)}{T} \left|\frac{|S(p,0)|}{n(p)} \left(\barf_p(p, 0) - \barf_p(p) \right)\right| \le 2\alpha.
\end{align*}

The same argument works for the non-swap version as well.
\end{proof}

\lemoptimalitypostprocess*
\begin{proof}
Fix $\pi_{1:T}$ and loss function $\ell$. For any post-processing function $k: \cP \to [0,1]$, we have
\begin{align*}
    &\frac{1}{n(p)}\sum_{t \in S(p)} \ell(y_t, k(p)) \\
    &=\left(1-\frac{|S(p, 1)|}{n(p)}\right) \ell(0, k(p)) + \frac{|S(p,1)|}{n(p)} \ell(1, k(p)) \\
    &=\left(1-\frac{\sum_{t \in S(p)} y_t}{|S(p)|}\right) \ell(0, k(p)) + \frac{\sum_{t \in S(p)} y_t}{|S(p)|} \ell(1, k(p)) \\
    &\ge \left(1-p \right) \ell(0, k(p)) + p \ell(1, k(p))  - |K(I, p)|\left| \ell(0, k(p)) -  \ell(1, k(p))\right|
\end{align*}
where the last inequality follows from 
\begin{align*}
    \frac{1}{n(p)}\left|\sum_{t \in S(p)} y_t - p \right| = |K(I, p)|
\end{align*}
and that the image of loss function $\ell$ is a non-negative real number.

Recalling the definition of $k^\ell$, we have
\begin{align*}
    &\frac{1}{n(p)}\sum_{t \in S(p)} \ell(y_t, k(p))\\
    &\ge \left(1-p \right) \ell(0, k(p)) + p \ell(1, k(p))  - |K(I, p)|\left| \ell(0, k(p)) - \ell(1, k(p))\right|\\
    &\ge \frac{1}{|S(p)|}\sum_{t \in S(p)} \ell(y_t, k^\ell(p)) - |K(I, p)|\left| \ell(0, k(p)) -  \ell(1, k(p))\right|.
\end{align*}
\end{proof}

\thmmulticaltoomni*
\begin{proof}
Fix $\pi_{1:T}$, $\{f_p\}_{p \in \cP}$ and $\{\ell_p\}_{p \in \cP} \in \cL_{convex}^m$. For now, fix $p \in \cP$. Consider the following post-processing function:
\[
    \hat{k}^{f_p}(p) = \barf(p) = \frac{1}{n(p)}\sum_{t \in S(p)}  f_p(x_t). 
\]

Using the convexity of $\ell_p(y, \cdot)$ and Jensen's inequality, we get
\begin{align*}
    &\frac{1}{n(p)}\sum_{t \in S(p)} \ell(y_t, k^{\ell_p}(\hp_t)) - \frac{1}{n(p)}\sum_{t \in S(p)} {\ell_p}(y_t, f_p(x_t)) \\
    &\le \frac{1}{n(p)}\sum_{t \in S(p)} {\ell_p}\left(y_t, \hat{k}^{f_p}(p) \right) - \frac{1}{n(p)}\sum_{t \in S(p)} \ell_p(y_t, f_p(x_t)) + C_{\ell_p} |K(I, p)|&\text{(Lemma~\ref{lem:optimality-post-process})}\\
    &= C_{\ell_p} |K(I,p)| +  \sum_{y \in \{0,1\}} \frac{|S(p,y)|}{n(p)}\left( \frac{1}{|S(p,y)|}\sum_{t \in S(p): y_t = y} \ell_p(y, \hat{k}^{f_p}(p))- \ell_p(y, f_p(x_t)) \right) \\
    &\le C_{\ell_p} |K(I, p)| + \sum_{y \in \{0,1\}} \frac{|S(p, y)|}{n(p)} \left( \ell_p(y, \hat{k}^{f_p}(p)) - \ell_p\left(y, \frac{\sum_{t \in S(p, y)} f_p(x_t)}{|\{t \in S(p): y_t = 0\}|}\right) \right) &\text{(Jensen's inequality)}\\
    &\le C_{\ell_p} |K(I,p)| + D_{\ell_p} \sum_{y \in \{0,1\}}  \frac{|S(p,y)|}{n(p)}\left(\barf_p(p) - \barf_p(p,y)\right) &\text{($D_{\ell_p}$-Lipschitzness of $\ell$)}.
\end{align*}

Averaging over each forecast $p \in \cP$, we have the final result:
\begin{align*}
    &\sum_{p \in \cP} \frac{n(p)}{T} \left(\frac{1}{|S(p)|}\sum_{t \in S(p)} \ell_p(y_t, k_{\ell_p}(\hp_t)) - \frac{1}{n(p)}\sum_{t \in S(p)} \ell_p(y_t, f_p(x_t)) \right)\\
    &\le \sum_{p \in \cP} \frac{n(p)}{T} \left(C_{\ell_p} |K(I, p)| + D_{\ell_p} \sum_{y \in \{0,1\}}  \frac{|S(p,y)|}{n(p)}\left(\barf_p(p) - \barf_p(p,y)\right)\right)\\
    &\le \left(\max_{p \in \cP} C_{\ell_p}\right) \sum_{p \in \cP} \frac{n(p)}{T} |K(I, p)| + \left(\max_{p \in \cP} D_{\ell_p}\right) \sum_{y \in \{0,1\}} \sum_{p \in \cP} \frac{|S(p,y)|}{T}\left(\barf_p(p) - \barf_p(p,y)\right) \\
    &\le \left(\max_{p \in \cP} C_{\ell_p}\right) \bsK_1(I) + 4\left(\max_{p \in \cP} D_{\ell_p}\right) \bsK_1(\{f_{p}\}_{p \in \cP})\\
    &\le \left(\max_{p \in \cP}C_{\ell_p} + 4D_{\ell_p} \right)\bsK_1(\cF).
\end{align*}
where the second to last inequality comes from Lemma~\ref{lem:condition-y-little-change}. 

The same argument works for the non-swap version as well.
\end{proof}

\section{Concentration Lemmas}
\label{app:concentration}
As we need to borrow some tools from \cite{block2021majorizing} and modify some of them, let us quickly go through some notations. Given any set $\cZ$, we write $z \in \cZ$ to denote its element and capital $Z$ to denote a complete binary tree. Given a path $\epsilon = (\epsilon_1, \dots, \epsilon_n) \in \{\pm 1\}^n$ where $-1$ means `left' and $+1$ `right', we write $Z(\epsilon)$ to denote the path along the tree $Z$ and $Z_t(\epsilon)$ or $Z(\epsilon_1, \dots, \epsilon_{t-1})$ to denote the $t$th element along the path $Z(\epsilon)$ --- note that $Z_1(\epsilon)$ is the root node that is the same no matter which path $\epsilon$ is given. Given a function $f: \cZ \to \R$ and $\cZ$-tree $Z$, we write $f(Z)$ to refer to $\R$-valued tree induced by $Z$ and $f$.

\begin{definition}[Sequential Fat Shattering Dimension]
    A tree $Z$ of depth $d$ is $\alpha$-shattered by a function class $\cG$ if there exists a tree $V$ such that
    \[
        \forall\ \epsilon\in\{\pm 1\}^d, \exists g\in \cG\quad \text{s.t. } \epsilon_t(f(Z_t(\epsilon)) - V_t(\epsilon)) \ge \alpha/2.
    \]
    In words, given any path $\epsilon \in \{\pm 1\}$ along the tree $Z$, there must exist some $f \in \cF$ that matches the sign of the path in terms of overestimating or underestimating the example $Z_t(\epsilon) = Z(\epsilon_1, \dots, \epsilon_{t-1})$.
    
    Tree $V$ is a witness to $\alpha$-shattering. The sequential fat shattering dimesnion $\fat_{\alpha}(\cG, \cZ)$ at scale $\alpha$ for the function class $\cG$ is the largest $d$ such that $\cG$ $\alpha$-shatters a $\cZ$-valued tree of depth $d$. 
\end{definition}

\begin{definition}[Sequential Majorizing Measure \cite{block2021majorizing}]
\[
    I^{\alpha}_{\cG, Z} =  \inf_{supp\ \mu \subseteq \cG(Z)} \sup_{v \in V, \epsilon \in \{\pm 1\}^n} \alpha + \frac{1}{\sqrt{n}} \int_{\alpha}^1 \sqrt{\log\left(\frac{1}{\mu(B_\delta(V, \epsilon))}\right)}.
\]
where $B_\delta(V, \epsilon)$ is the set of all trees $V'$ in the support of measure $\mu$ such that \[
||V'(\epsilon) - V(\epsilon)||^2 := \sum_{t=1}^n (V'_t(\epsilon) - V_t(\epsilon))^2 \le n \delta.\] We write $I^{\alpha}_{\cG} = \sup_{Z} I^{\alpha}_{\cG, Z}$. 
\end{definition}

Now, let us re-write the main tools that we will borrow from \cite{block2021majorizing}.
\begin{lemma}[Corollary 11 of \cite{block2021majorizing}]
\label{lem:sequential-concentrate}
Let $z_1, \dots, z_t, \dots$ be a sequence of $\cZ$-valued random variables adapted to the filtration $\cA_t$ and let $\cG$ be a $[0,1]$-valued function class on $\cZ$. Then with probability at least $1-4\rho$ over the randomness of $\{\ts_t\}_{t \in n(p)}$, 
\begin{align*}
    \sup_{g \in \cG} \left|\frac{1}{n}\sum_{t=1}^n g(z_t) -\E[g(z_t) | \cA_t]  \right| &\le C_1 \left(I^{\alpha}_{\cG} + \sqrt{\frac{\log(\frac{1}{\rho})}{n}}\right) \\
    &\le C_2 \cdot \left( \alpha + \frac{1}{\sqrt{n}}\int_{\alpha}^1 \sqrt{\log(N'(\cF, \delta))} d\delta +\sqrt{\frac{\log(\frac{1}{\rho})}{n}} \right)
\end{align*}
for any $\alpha \in (0,1)$, some universal constants $(C_1, C_2)$, and $N'(\cF, \delta)$ is a fractional covering number defined in \cite{block2021majorizing}.
\end{lemma}

\begin{theorem}[Theorem 13 of \cite{block2021majorizing}]
\label{thm:fat-frac-cover}
    Let $\cG: \cZ \to [0,1]$ be a function class. There exists universal constants $c, C_3$ such that for all $\delta > 0$, and all trees $Z$
    \[
        N'(\cG, \delta, Z) \le \left(\frac{C_3}{\delta}\right)^{3\fat_{c\delta}(\cG)}
    \].
\end{theorem}

\subsection{Mean Case}
\lemswapregretconcentrationtwo*
The proof for Lemma~\ref{lem:swap-regret-concentration2} is essentially identical to the proof for Lemma~\ref{lem:swap-regret-concentration}, so we only present the proof for Lemma~\ref{lem:swap-regret-concentration} below.

\lemswapregretconcentration*
\begin{proof}
Consider a sequence  \[
        \tz_t = (x_t, y_t, \thp_t)
    \]
adapted to the filtration $\pi_{1:t}$ for all $t \in [T]$. We write the domain of $\tz$ as $\cZ=\cX \times \cY \times \cP$. For each $p\in \cP$, define 
\[
    \cG_p = \left\{g((x,y,\hp)) = \frac{1}{2}\left(D_1f(x) +1\right): f \in \cF_B \right\}
\]
where $D_1 = \min(1, \frac{1}{\sqrt{B}})$.
Also, we define a post-processing function\[
    d_p(v, \hp, y) = D_2 \cdot \ind[\hp = p] \cdot \left((p - y)^2 - \left(\frac{1}{D_1}(2v-1) - y\right)^2\right)
\]
where $D_2=\max(\frac{D^2_1}{8}, \frac{D_1}{2})$ and we assume $v \in [0, 1]$. Note that for any $g \in \cG_p$ and $z=(x,y,\hp) \in \cZ$
\[
    d_p(g(x), \hp, y) = D_2\cdot \ind[\hp = p] \cdot \left((p - y)^2 - (f(x) - y)^2\right) \in [0,1]
\]
and $d(v, \hp, y)$ is 1-Lipschitz in $v \in [0,1]$:
\begin{align*}
    |d_p(v, \hp, y) - d_p(v', \hp, y)| &= \ind[\hp = p]D_2 \left|\left(\frac{1}{D_1}(2v-1)-y\right)^2 - \left(\frac{1}{D_1}(2v'-1)-y\right)^2 \right|\\
    &= D_2 \left|\left(\frac{4}{D^2_1}(v+v')- \frac{2y}{D_1}\right)(v-v') \right|\\
    &\le D_2 \left|\left(\frac{8}{D^2_1}- \frac{2y}{D_1}\right)(v-v') \right|\\
    &\le |v-v'|.
\end{align*}

Now, we can re-write the contextual swap regret's deviation from its expectation in each round as
\begin{align*}
     &\frac{D_2}{T} \sup_{\{f_p\}_{p \in \cP}}\sum_{t=1}^T (\hp_t- y_t)^2 - (f_{\hp_t}(x_t) - y_t)^2 - \E_{\hp'_t}\left[(\hp'_t- y'_t)^2 - (f_{\hp'_t}(x'_t) - y'_t)^2 | \pi_{1:t-1}\right] \\
     &=\frac{D_2}{T} \sum_{p \in \cP} \sup_{f_p \in \cF_B} \sum_{t=1}^T \ind[\hp_t = p] \left((\hp - y)^2 - (f_p(x) - y)^2\right) - \left((p - y)^2 - (f_p(x) - y)^2\right) \Pr_{\hp' \sim \thp_t}[\hp' = p|\pi_{1:t-1}]\\
     &=\sum_{p \in \cP} \sup_{g_p \in \cG_p} \frac{1}{T} \sum_{t=1}^T d_p(g_p(x_t), \hp_t, y_t) - \E_{\hp' \sim \thp_t}[d(g_p(x_t), \hp_t, y_t)|\pi_{1:t-1}].
\end{align*}

In other words, it suffices to bound the deviation from its expectation for every $p$. 
\begin{lemma}
\label{lem:mean-frac-contraction}
    \[
        I^\alpha_{d_p \circ \cG_p, Z} \le I^{\alpha}_{\cF_B, X^Z}.
    \]
    In other words, we have $I^{\alpha}_{d \circ \cG_p} \le I^{\alpha}_{\cF_B}$
\end{lemma}
\begin{proof}
\item
\paragraph{1. $I^{\alpha}_{d_p \circ \cG, Z} \le  I^{\alpha}_{\cG, X^Z}$:}
Fix any tree $Z$. Similarly to the proof of Lemma~\ref{lem:fat-dim-bound}, we write $X^Z$, $Y^Z$, $P^Z$ to denote the $\cX$, $\cY$, and $\cP$-trees induced by $Z$: for any $\epsilon \in \{\pm 1\}^n$ and $\tau \in [n]$, $Z_\tau(\epsilon) = (X^Z_\tau(\epsilon), Y^Z_\tau(\epsilon), P^Z_\tau(\epsilon))$.

Note that it is sufficient to show that for any measure $\mu$ over the $\R$-trees whose support is over $g(X^Z)$ for every $g \in \cG_p$, there exists a measure $\tilde{\mu}$ over $\R$-trees such that the following is true: for every $g \in \cG_p$, we have
\[
   \log\left(\frac{1}{\tilde{\mu}(B_\delta (d_p(g(X^Z), P^Z, Y^Z), \epsilon))}\right) \le \log\left(\frac{1}{\mu(B_\delta(g(X^Z), \epsilon))}\right). 
\]

We can construct such $\tilde{\mu}$ over $\R$-valued trees via pushforward operation with the following transformation function $Q$: given a $\R$-valued-tree $R$, $Q$ outputs a new $\R$-valued tree $R' = Q(R)$ such that for every $\epsilon \in \{ \pm 1\}^n$
\[
    R'(\epsilon) := d_p(R(\epsilon), P^Z(\epsilon), Y^Z(\epsilon)) 
\]
and $\tilde{\mu} = Q \# \mu$ is the resulting pushforward measure via $Q$. Then, because of the $1$-Lipschitzness of $d_p(\cdot, \hp, y)$, we have for any tree $R$,
\begin{align*}
    ||g(X^Z(\epsilon)) - R(\epsilon)|| \le n \delta^2
    \implies ||d_p(g(X^Z(\epsilon)), P^Z(\epsilon), Y^Z(\epsilon)) - R'(\epsilon)|| \le n \delta^2
\end{align*}
where $R' = Q(R)$. Hence, we have
\[
    \left\{Q(R): R \in B_\delta(g(X^Z), \epsilon)\right\} \subseteq B_\delta(d(g(X^Z), P^Z, Y^Z), \epsilon).
\]

Therefore, we can show that for any $\epsilon \in \{\pm 1\}^n$ \[
    \mu(B_\delta(g(X^Z), \epsilon))  = \tilde{\mu}(\left\{Q(R): R \in B_\delta(g(X^Z), \epsilon)\right\}) \le \tilde{\mu}\left(B_\delta(d_p(g(X^Z), P^Z, Y^Z), \epsilon)\right).
\]
\paragraph{2. $I^\alpha_{\cG_p, X} = I^{\alpha}_{\cF, X}$:}
Note that $\cG = h \circ \cF_B$ where $h(v) = \frac{1}{2}(D \cdot v + 1)$, which is $1$-Lipschitz as $\frac{D}{2} \le 1$. Therefore, by Proposition 5 of \cite{block2021majorizing}, we have $I^\alpha_{\cG_p, X} = I^{\alpha}_{\cF_B, X}$.

\end{proof}

Invoking Lemma~\ref{lem:sequential-concentrate} gives us that with probability $1-4m$ over the randomness of $\{\hp_t\}_{t=1}^T$,
\begin{align*}
    &\sup_{g_p \in \cG_p} \sum_{t=1}^T d_p(g_p(x_t), \hp_t, y_t) - \E_{\hp' \sim \thp_t}[d_p(g_p(x_t), \hp_t, y_t)|\pi_{1:t-1}]\\
    &\le C_1 \left(I^{\alpha}_{d_p \circ \cG_p} + \sqrt{\frac{\log(\frac{1}{\rho})}{T}}\right) \\
    &\le C_1 \left(I^{\alpha}_{\cF_B} + \sqrt{\frac{\log(\frac{1}{\rho})}{T}}\right) \\
    &\le 2C_1 \cdot \left(\inf_{\alpha \in (0,1)} \alpha + \frac{1}{\sqrt{T}}\int_{\alpha}^1 \sqrt{\log(N'(\cF_B, \delta))} d\delta +\sqrt{\frac{\log(\frac{1}{\rho})}{T}} \right)\\
    &\le  2C_1 \cdot \left(\inf_{\alpha \in (0,1)} \alpha + \frac{1}{\sqrt{T}}\int_{\alpha}^1 \sqrt{\log\left(\left(\frac{C_3}{\delta}\right)^{3\fat_{c\delta}(\cF_B)}\right)} d\delta +\sqrt{\frac{\log(\frac{1}{\rho})}{T}} \right)\\
    &\le  2C_1 \cdot \left(\inf_{\alpha \in (0,1)} \alpha + \frac{1}{\sqrt{T}}\int_{\alpha}^1 \sqrt{ 3\fat_{c\delta}(\cF_B) \log\left(\left(\frac{C_3}{\delta}\right)\right)} d\delta +\sqrt{\frac{\log(\frac{1}{\rho})}{T}} \right)\\
    &\le  C_{\cF_B} \sqrt{\frac{\log(\frac{1}{\rho})}{T}}
\end{align*}
for some finite constant $C_{\cF_B}$ that depends on the complexity of $\cF_B$ via its fat shattering dimension. The third inequality from Lemma~\ref{lem:mean-frac-contraction}, and the fourth inequality follows from Theorem~\ref{thm:fat-frac-cover}, and the very last inequality follows from plugging in $\alpha = \frac{1}{\sqrt{T}}$ and noting that the sequential fat dimension of $\cF_B$ is finite at any scale $\delta >0 $. 

Summing over all $p \in \cP$, we have with probability $1-4m\rho$,
\begin{align*}
    &\frac{D_2}{T} \sup_{\{f_p\}_{p \in \cP}}\sum_{t=1}^T (\hp_t- y_t)^2 - (f_{\hp_t}(x_t) - y_t)^2 - \E_{\hp'_t}\left[(\hp'_t- y'_t)^2 - (f_{\hp'_t}(x'_t) - y'_t)^2 | \pi_{1:t-1}\right] \\
     &=\sum_{p \in \cP} \sup_{g_p \in \cG_p} \frac{1}{T} \sum_{t=1}^T d_p(g_p(x_t), \hp_t, y_t) - \E_{\hp' \sim \thp_t}[d_p(g_p(x_t), \hp_t, y_t)|\pi_{1:t-1}]\\
     &\le m C_{\cF_B} \sqrt{\frac{\log(\frac{1}{\rho})}{T}}
\end{align*}
\end{proof}

\subsection{Quantile Case}
\begin{lemma}[Concentration of Quantile Error]
\label{lem:quant-error-concentrate}
Fix any $\tpi_{1:T}$, $p \in \cPq$, and $\cF_B$. Suppose the sequential shattering dimension of $\cF_B$ is finite at any scale $\delta$: $\fat_{\delta}(\cF_B) < \infty$. With probability $1-4\rho$,
\[
    \sup_{f \in \cF_B} \left|\frac{1}{n(p)} \sum_{t \in S(p)}f(x_t)\ind[s_t \le p] - f(x_t) \Pr_{s_t \sim \ts_t}[s_t \le p]\right| \le \max\left(1, \sqrt{B}\right) C_{\cF_B}  \sqrt{\frac{\log(\frac{1}{\rho})}{n(p)}}
\]
where $C_{\cF_B}$ is some finite constant that depends on the complexity of $\cF_B$. 
\end{lemma}
\begin{proof}
For each round $\tau \in [n]$, define $\cZ = [0,1] \times \cX$ and $D = \min(1, 1/\sqrt{B})$. Now, for each $f \in \cF_B$ and any $z=(x,s) \in \cZ$, define 
\[
        g_f(z) = \frac{1}{2}\left(D \cdot f(x) \cdot M_p(s) + 1\right)
\]
where $M_p(s) = \ind[s \le p]$. We denote the induced family of functions $\cG = \{g_f: f \in \cF_B\}$. Note that $\cG \subseteq [0,1]^\cZ$ as $D \cdot f(x) \le [-1, 1]$ for any $f \in \cF_B$.

\begin{lemma}
\label{lem:fat-dim-bound}
The fat shattering dimension of $\cG$ matches that of $\cF_B$:
    \[
        \fat_\alpha(\cG, \cZ) \le \fat_{\alpha'}(\cF_B, \cX)
    \]
    where $\alpha' = 2\sqrt{B} \alpha$.
\end{lemma}
\begin{proof}
For convenience, write $\fat_{\alpha}(\cF_{\cA}, \cX)=d$ and $\fat_\alpha(\cG, \cZ) = d'$. For the sake of contradiction, suppose the fat shattering dimension of $\cG$ is strictly greater than $d$, meaning there exists a $\cZ$-valued tree $Z$ and $\R$-valued tree $R$ that witness to its fat shattering dimension being $d' > d$. We write $X^Z$ and $S^Z$ to denote the $\cX$-tree and $\R$-tree induced by $Z$: i.e. $Z_\tau(\epsilon) = (X^Z_\tau(\epsilon), S^Z_\tau(\epsilon))$.

First, we argue that there must exist $\epsilon \in \{\pm 1\}^{d'}$ and $\tau \in [d']$ such that \[
M_p(S^Z_{\tau}(\epsilon)) = 0.
\]
If not, then we have for every $\epsilon \in \{\pm 1\}^{d'}$ 
\[
    g_f(Z_\tau(\epsilon)) = \frac{1}{2}\left(D f(X^Z_{\tau}(\epsilon)) + 1 \right).
\]
Then, we have a $\cX$-tree $X^Z$ with depth $d'$ and an $\R$-tree $R'$ where $R'=1/D \cdot (2 \cdot R^Z - 1)$ which together serve as a witness for $\alpha'$-shattering where $\alpha'=2\sqrt{B}\alpha$: for every $\epsilon \in \{\pm 1\}^{d'}$ and $\tau \in [d']$,
\[
    \epsilon_\tau \cdot (f(X_{\tau}^Z(\epsilon)) - R'_\tau(\epsilon)) \ge \alpha' / 2:
\] i.e. we would have $\fat_{\alpha'}(\cF_B, \cX) = d' > d$, which leads to a contradiction that $d$ is the maximum depth for which $\cF_B$ is shattered at scale $\alpha'$.

Hence, suppose there exists some $\tau^* \in [d']$ and $(\epsilon_1, \dots, \epsilon_{\tau^*-1}) \in \{\pm 1\}^{\tau^*}$  such that 
\[
  M_p(S^Z(\epsilon_1, \dots, \epsilon_{\tau^*-1})) = 0.  
\]

Then, for any $g_f \in \cG$, we have 
\[
    g_f(Z(\epsilon_1, \dots, \epsilon_{\tau^*-1})) - R(\epsilon_1, \dots, \epsilon_{\tau^*-1}) = - R(\epsilon_1, \dots, \epsilon_{\tau^*-1}).
\]

Hence, either there doesn't exist any $g_f \in \cG$ such that
\[
    +1 \left(g_f(Z(\epsilon_1, \dots, \epsilon_{\tau^*-1})) - R(\epsilon_1, \dots, \epsilon_{\tau^*-1})\right) = - R(\epsilon_1, \dots, \epsilon_{\tau^*-1}) \ge \alpha/2
\]
or there's no $g_f \in \cG$ such that
\[
    -1 \left(g_f(Z(\epsilon_1, \dots, \epsilon_{\tau^*-1})) - R(\epsilon_1, \dots, \epsilon_{\tau^*-1})\right) = R(\epsilon_1, \dots, \epsilon_{\tau^*-1}) \ge \alpha/2.
\]
This contradicts the assumption that $Z$ and $R$ is a witness to $\alpha$-shattering for $\cG$. This concludes the proof that $d' \le d$.
\end{proof}

Now, fix $\tpi_{1:T}$. And instead of iterating the entire time horizon $[T]$, we only iterate over $S(\tpi_{1:T}, p)$. For convenience, we overload the notation and define the new rounds $\tau \in [n(\tpi_{1:T}, p)]$ such that the corresponding random variable $z_\tau$ matches the $(x_t, s_t)$: define \[
    z_\tau = (s_{t_\tau}, x_{t_\tau})
\]
where $t_\tau = S(\tpi_{1:T}, p)[\tau]$ --- i.e. the round at which prediction $p$ is made for the $\tau$-th time.

Invoking Lemma~\ref{lem:sequential-concentrate} gives us that with probability $1-4\rho$ over the randomness of $\{\ts_t\}_{t \in S(\tpi_{1:T}, p)}$,
\begin{align*}
    & D \cdot \sup_{f \in \cF_B} \left|\frac{1}{n(p) } \sum_{t \in S(p)}f(x_t)\ind[s_t \le p] - f(x_t) \Pr_{s_t \sim \ts_t}[s_t \le p]\right| \\
    &=2\sup_{g \in \cG} \left|\frac{1}{n(p)} \sum_{\tau=1}^{n(p)} g_f(z_\tau) - \E_{\tau-1}[g(z_\tau)]\right| \\
    &\le 2C_1 \cdot \left(\inf_{\alpha \in (0,1)} \alpha + \frac{1}{\sqrt{n(p)}}\int_{\alpha}^1 \sqrt{\log(N'(\cG, \delta))} d\delta +\sqrt{\frac{\log(\frac{1}{\rho})}{n(p)}} \right)\\
    &\le  2C_1 \cdot \left(\inf_{\alpha \in (0,1)} \alpha + \frac{1}{\sqrt{n(p)}}\int_{\alpha}^1 \sqrt{\log\left(\left(\frac{C_3}{\delta}\right)^{3\fat_{c\delta}(\cG)}\right)} d\delta +\sqrt{\frac{\log(\frac{1}{\rho})}{n(p)}} \right)\\
    &\le  2C_1 \cdot \left(\inf_{\alpha \in (0,1)} \alpha + \frac{1}{\sqrt{n(p)}}\int_{\alpha}^1 \sqrt{ 3\fat_{2c\sqrt{B}\delta}(\cF_B) \log\left(\left(\frac{C_3}{\delta}\right)\right)} d\delta +\sqrt{\frac{\log(\frac{1}{\rho})}{n(p)}} \right)\\
    &\le  C_{\cF_B} \sqrt{\frac{\log(\frac{1}{\rho})}{n(p)}}
\end{align*}
for some finite constant $C_{\cF_B}$ that depends on the complexity of $\cF_B$ via its fat shattering dimension. The second inequality follows from Theorem~\ref{thm:fat-frac-cover}, the third inequality from Lemma~\ref{lem:fat-dim-bound}, and the fourth inequality follows by plugging in $\alpha = \frac{1}{\sqrt{n(p)}}$ and noting that the sequential fat dimension of $\cF_B$ is finite at any scale $\delta >0 $. 
\end{proof}

\begin{lemma}[Concentration of Pinball Loss]
\label{lem:pinball-concentrate}
Fix $\tpi_{1:T}$, $p \in \cPq$, and $\cF_B$. Suppose the sequential shattering dimension of $\cF_B$ is finite at any scale $\delta$: $\fat_{\delta}(\cF_B) < \infty$. We have that with probability $1-4\rho$ over the randomness of $\{\ts_{t}\}_{t \in S(p)}$,
    \begin{align*}
    \sup_{f \in \cF_B} \left|\frac{1}{n(p)} \sum_{t \in S(p)} \pb_q(s_t, p) - \pb_q(s_t, f(x_t)) - \E_{s \sim \ts_t}[\pb_q(s, p) - \pb_q(s, f(x_t)]\right| \le \max\left(1, \sqrt{B}\right) C'_{\cF_B} \sqrt{\frac{\log(\frac{1}{\rho})}{n(p)}}
    \end{align*}
\end{lemma}
\begin{proof}
We proceed with a similar argument as the proof of Proposition 5 of \cite{block2021majorizing}.
First, define $\cZ = \cX \times [0,1]$ and $D = \min(1, \frac{1}{\sqrt{B}})$. In order to enforce the range of the function values to be $[0,1]$, we define the function class \[
\cG = \left\{g_f = \frac{1}{2}\left(D \cdot f(\cdot)+ 1\right): f \in \cF_B \right\}.
\]

Define $d: [0,1] \times [0, 1] \to [0,1]$ as
\[
    d(s, v) = \frac{1}{2}\left(D \cdot \left(\pb_q(s, p) - \pb_q\left(s, \frac{2v-1}{D}\right)  \right) +1\right).
\]
Note that $d(s, \cdot)$ is $1$-Lipschitz:
\begin{align*}
    &|d(s, v) - d(s, v')| \\
    &= \frac{D}{2} \left|\pb_q\left(s, \frac{2v-1}{D}\right) - \pb_q\left(s, \frac{2v'-1}{D}\right)  \right|\\
    &\le \frac{D}{2} \left|\frac{2v-1}{D} - \frac{2v'-1}{D} \right|\\
    &\le |v - v'|
\end{align*}
as $PB_q(s, \cdot)$ is $\max(q,1-q)$-Lipschitz and $\max(q, 1-q) \le 1$.

Also, note that $d(s, v) \in [0,1]$ and for any $g_f \in \cG$,
\[
    d(s, g_f(x)) = \frac{1}{2}\left(D \cdot \left(\pb_q\left(s, f(x)\right) - \pb_q(s, p) \right) +1\right) 
\]
where $f\in \cF_B$ is such that $g_f(x) = \frac{1}{2}\left(D \cdot f(x)+ 1\right)$.

Using the same argument as in Proposition 5 of \cite{block2021majorizing}, we can show the following claim.
\begin{lemma}
\label{lem:lipschitz-pinball}
For any tree $Z$,
    \begin{align*}
        I^{\alpha}_{d \circ \cG, Z} \le  I^{\alpha}_{\cG, X^Z} = I^{\alpha}_{\cF_B, X^Z}
    \end{align*}
where  $X^Z$ is the tree induced by the $Z$: i.e. for every $\epsilon \in \{\pm 1\}^n$, $Z_t(\epsilon) = (X^Z_t(\epsilon), s_t)$. In other words, $I^{\alpha}_{d \circ \cG} \le  I^{\alpha}_{\cF_B}$.
\end{lemma}
\begin{proof}
\item
\paragraph{1. $I^{\alpha}_{d \circ \cG, Z} \le  I^{\alpha}_{\cG, X^Z}$:}
Fix any tree $Z$. As in the proof of Lemma~\ref{lem:fat-dim-bound}, we write $X^Z$ and $S^Z$ to denote the $\cX$ and $\R$-trees induced by $Z$: for any $\epsilon \in \{\pm 1\}^n$ and $\tau \in [n]$, $Z_\tau(\epsilon) = (S^Z_\tau(\epsilon), X^Z_\tau(\epsilon))$.

Note that it is sufficient to show that for any measure $\mu$ over the $\R$-trees whose support is over $g_f(X^Z)$ for every $g_f \in \cG$, there exists a measure $\tilde{\mu}$ over $\R$-trees such that the following is true:
\[
   \log\left(\frac{1}{\tilde{\mu}(B_\delta (d(S^Z, g_f(Z)), \epsilon))}\right) \le \log\left(\frac{1}{\mu(B_\delta(g_f(X^Z), \epsilon))}\right). 
\]

We can construct such $\tilde{\mu}$ over $\R$-valued trees via pushforward operation with the following transformation function $Q$: given a $\R$-valued-tree $R$, $Q$ outputs a new $\R$-valued tree $R' = Q(R)$ such that for every $\epsilon \in \{ \pm 1\}^n$
\[
    R'(\epsilon) := d(S^Z(\epsilon), R(\epsilon)) 
\]
and $\tilde{\mu} = Q \# \mu$ is the resulting pushforward measure via $Q$. Then, because of the $1$-Lipschitzness of $d(s, \cdot)$, we have for any tree $R$,
\begin{align*}
    ||g_f(X^Z(\epsilon)) - R(\epsilon)|| \le n \delta^2
    \implies ||d(S^Z(\epsilon), g_f(X^Z(\epsilon))) - R'(\epsilon)|| \le n \delta^2
\end{align*}
where $R' = Q(R)$. Hence, we have
\[
    \left\{Q(R): R \in B_\delta(g_f(X^Z), \epsilon)\right\} \subseteq B_\delta(d(S^Z, g_f(X^Z)), \epsilon)
\]

Therefore, we can show that for any $\epsilon \in \{\pm 1\}^n$ \[
    \mu(B_\delta(g_f(X^Z), \epsilon))  = \tilde{\mu}(\left\{Q(R): R \in B_\delta(g_f(X^Z), \epsilon)\right\}) \le \tilde{\mu}\left(B_\delta(d(S^Z, g_f(X^Z)), \epsilon)\right).
\]
\paragraph{2. $I^\alpha_{\cG, X} = I^{\alpha}_{\cF, X}$:}
Note that $\cG = h \circ \cF_B$ where $h(v) = \frac{1}{2}(D \cdot v + 1)$, which is $1$-Lipschitz. Therefore, by Proposition 5 of \cite{block2021majorizing}, we have $I^\alpha_{\cG, X} = I^{\alpha}_{\cF, X}$.

\end{proof}

Fix any $\tpi_{1:T}$. Combining Lemma~\ref{lem:lipschitz-pinball} and Lemma~\ref{lem:sequential-concentrate}, we have with probability $1-4\rho$ over the randomness of $\{\ts_t\}_{t \in n(p)}$, 
\begin{align*}
    &D \sup_{f \in \cF_B} \left|\frac{1}{n(p)} \sum_{t \in S(p)} \pb_q(s_t, p) - \pb_q(s_t, f(x_t)) - \E_{s \sim \ts_t}[\pb_q(s, p) - \pb_q(s, f(x_t)]\right| \\
    &= 2\sup_{g \in \cG} \left|\frac{1}{n(p)} \sum_{t \in S(p)} d(s_t, g_f(x_t)) - \E_{s_t \sim \ts_t}[d(s_t, g_f(x_t))|\pi_{1:t-1}] \right|\\
    &\le  2C_1 \left(I^{\alpha}_{d \circ \cG} + \sqrt{\frac{\log(\frac{1}{\rho})}{n(p)}}\right) \\
    &\le  2C_1 \left(I^{\alpha}_{\cF_B} + \sqrt{\frac{\log(\frac{1}{\rho})}{n(p)}}\right)\\
    &\le  2C_1 \cdot \left(\inf_{\alpha \in (0,1)} \alpha + \frac{1}{\sqrt{n(p)}}\int_{\alpha}^1 \sqrt{\log\left(\left(\frac{C_3}{\delta}\right)^{3\fat_{c\delta}(\cF_B)}\right)} d\delta +\sqrt{\frac{\log(\frac{1}{\rho})}{n(p)}} \right)\\
    &\le  2C_1 \cdot \left(\inf_{\alpha \in (0,1)} \alpha + \frac{1}{\sqrt{n(p)}}\int_{\alpha}^1 \sqrt{ 3\fat_{c \delta}(\cF_B) \log\left(\left(\frac{C_3}{\delta}\right)\right)} d\delta +\sqrt{\frac{\log(\frac{1}{\rho})}{n(p)}} \right)\\
    &\le  C'_{\cF_B} \sqrt{\frac{\log(\frac{1}{\rho})}{n(p)}}
\end{align*}
where $C'_{\cF_B}$ is a finite constant that depends on the complexity of $\cF_B$.
\end{proof}

\section{An Extension: Online Oracle-efficient Quantile Multicalibration and  Multivalid Conformal Prediction}
\label{app:conformal-extension}

\subsection{Connection between Multivalidity and Quantile Errors}

Here we leverage a similar technique as in \cite{jung2022batch}. We show that if there exists some $f$ and level set $p$ that witnesses to the quantile error being $\alpha$ with respect to the conformal score distribution $\ts_t$'s, an affine transformation of $f$ also serves as a witness to the pinball loss regret being large with respect to $\ts_t$'s. For the same reasons discussed in \cite{jung2022batch}, this argument needs to happen with respect to the conformal score distributions $\ts_t$'s not the realized scores $s_t$'s. However, because our contextual swap regret algorithm's guarantee is with respect to realized scores $s_t$'s, we will use concentration arguments proved in Appendix~\ref{app:concentration} to show that multivalidity error and pinball loss regret under $\ts_t$'s and $s_t$'s must be very close with high probability.

\begin{lemma}
Fix some distribution of (features, conformal score) $\cD \sim \Delta(\cX \times [0,1])$ and write its corresponding marginal distribution over $\cX$ as $\cD_\cX$. The expected value of the derivative of the pinball loss with respect to the shift scaled by $f(x)$ is
\begin{align*}
    \E_{(x,s) \sim \cD}\left[\frac{d\pb_q(p+\tau \cdot f(x), s)}{d\tau} \Big| \tau = b \right] = \E_{x \sim \cD_\cX}\left[f(x) \cdot \left(\Pr_{s \sim \cD_\cS(x)}[s \le p + b \cdot f(x)] - q\right)\right]
\end{align*}
\end{lemma}
\begin{proof}
Define the function $H_{(x,s)}(b) = \pb_q(p+ b \cdot f(x), s)$ and its derivative 
\begin{align*}
    H'_{(x,s)}(b) &= \frac{d\pb_q(p+\tau \cdot f(x), s)}{d\tau}\Big| \tau=b\\
    &= \begin{cases} f(x) \cdot (1-q) \quad\text{if $s \le p + b \cdot f(x)$}\\ -f(x) \cdot q \quad\text{otherwise}\end{cases}.
\end{align*}
Taking the expectation of $H'_{(x,s)}(b)$ with respect to $\cD$ yields
    \begin{align*}
        &\E_{(x,s) \sim \cD}\left[\frac{d\pb_q(p+\tau \cdot f(x), s)}{d\tau} \Big| \tau = b \right]\\
        &=\E_{(x,s) \sim \cD}[H'_{(x,s)}(b)] \\ 
        &=\E_{x \sim \cD_\cX}\left[\int_{0}^b H'_{(x,s)}(b) \cD_\cS(x)(ds) \right]\\
        &=\E_{x \sim \cD_\cX}\left[\int_{0}^{p+ b \cdot f(x)} H'_{(x,s)}(b) \cD_\cS(x)(ds) + \int_{p+ b \cdot f(x)}^{1} H'_{(x,s)}(b) \cD_\cS(x)(ds) \right]\\
        &=\E_{x \sim \cD_\cX}\left[f(x) \cdot \left(\int_{0}^{p+ b \cdot f(x)} (1-q) \cD_\cS(x)(ds) - \int_{p+ b \cdot f(x)}^{1} q \cD_\cS(x)(ds)\right) \right]\\
        &=\E_{x \sim \cD_\cX}\left[f(x) \cdot \left(\Pr_{s \sim \cD_\cS(x)}[s \le p + b \cdot f(x)] - q\right) \right].
    \end{align*}
\end{proof}

\begin{lemma}
\label{lem:quantile-error-implies-pinball-loss-error}
Fix $\tpi_{1:T}$. Suppose for every $t \in [T]$, $\ts_t$ is $\rho$-Lipschitz and continuous. If there exists $p \in \cPq$ and $f \in \cF_B$ such that
\[
    Q(\tpi_{1:T}, p, f) \ge \alpha
\]
then there exists $f' \in \cF_{(1+\sqrt{B})^2}$ such that
    \[
        \sum_{t \in S(\tpi_{1:T}, p)} \E_{s_t \sim \ts_t}\left[PB(\hp_t, s_t) - PB(f'(x_t), s_t) \right] \ge \frac{\alpha^2}{\max(2, 8B\rho)}
    \]
\end{lemma}
\begin{proof}
Fix $p \in \cPq$ and $\tpi_{1:T}=\{(x_t, \ts_t, \hp_t)\}_{t=1}^T$. Write \[H(\tau) = \frac{1}{n(p)} \sum_{t \in S(p)} f(x_t) \cdot (q - \Pr_{s_t \sim \ts_t}[s_t \le p + \tau \cdot f(x_t)])\] and $H^{-1}$ is the corresponding inverse function. Define $f'(x) = p + b \cdot f(x)$ where $b$ is chosen to be 
\[
b = \min\left(1, H^{-1}\left(\frac{\alpha}{2}\right)\right).
\]
Note because $b \le 1$, $f' \in \cF_{(1+\sqrt{B})^2}$.

We can re-write the difference in pinball loss under $p$ and $f'(x_t)$ in terms of $H$:
\begin{align*}
    &\E_{(x,s) \sim \cD^p}[\pb_q(p, s) - \pb_q(p + b\cdot f(x), s)]\\
    &= - \int_{\cX \times [0,1]} \left(\int_{0}^b \frac{d\pb_q(p+\tau \cdot f(x), s)}{d\tau} d\tau\right) \cD^p(dx, ds)\\
    &=-\int_{0}^b  \left(\int_{\cX \times [0,1]} \frac{d\pb_q(p+\tau \cdot f(x), s)}{d\tau} \cD^p(dx, ds)\right)d\tau \\
    &= \int_{0}^b \E_{x \sim \cD^p_\cX}\left[f(x) \cdot \left(q-\Pr_{s \sim \cD^p_{\cS}(x)}[s \le p + \tau \cdot f(x)] \right)\right] d\tau\\
    &= \int_{0}^b H(\tau) d\tau.
 \end{align*}

\begin{lemma}
\label{lem:quantile-error-integral-bound}
Fix $\tpi_{1:T}$. Suppose $\cD(\tpi_{1:T})$ is $\rho$-Lipschitz.
If there exists $p \in \cPq$ and $f \in \cF_B$ such that
\[
H(0) = \E_{x \sim \cD^p_\cX(\tpi_{1:T})}\left[f(x) \cdot \left(q - \Pr_{s \sim \cD^p_{\cS}(\tpi_{1:T}, x)}[s \le p] \right)\right] \ge \alpha.
\]
Then we have
    \[
        \int_{0}^b H(\tau) d\tau = \int_{0}^b \E_{x \sim \cD^p_\cX(\tpi_{1:T})}\left[f(x) \cdot \left(q - \Pr_{s \sim \cD^p_{\cS}(\tpi_{1:T}, x)}[s \le p + \tau \cdot f(x)] \right)\right] d\tau  \ge \frac{\alpha^2}{2B\rho}.
    \]
\end{lemma}
\begin{proof}
Note that for any $\tau' > \tau$,
\begin{align*}
    &H(\tau') - H(\tau) \\
    &= \E_{x \sim \cD^p_\cX}\left[f(x) \cdot \left(\Pr_{s \sim \ts | x}[s \le p + \tau \cdot f(x)] - \Pr_{s \sim \ts |x}[s \le p + \tau' \cdot f(x)]\right) \right]  \\
    &\ge \rho (\tau - \tau') \E_{x \sim \cD^p_{\cX}}[f^2(x)]\\
    &\ge \rho B (\tau - \tau')
\end{align*}
because $f \le \cF_B$. Using the same argument as in \cite{jung2022batch}, we can show that the area under the curve is at least
\begin{align*}
    \int_{0}^b H(\tau) d\tau & \ge b H(b) + \frac{(H(0) - H(b))^2}{2\rho B}
\end{align*}
where $b H(b)$ is the area of the rectangle and $\frac{(H(0) - H(b))^2}{2\rho B}$ the area of the triangle is where the slope cannot be steeper than $\rho B$ due to $H$'s $(\rho B)$-Lipschitzness. 

\paragraph{Case (i) $1 \le H^{-1}(\alpha/2)$:}
Because $H(\tau)$ is a non-increasing function in $\tau$, $H(b) = H(1) \ge \alpha/2$. Therefore, we can lower bound the integral as 
\begin{align*}
    \int_{0}^b H(\tau) d\tau &\ge b H(b) + \frac{(H(0) - H(b))^2}{2\rho B}\\
    &\ge b H(b) \\
    &\ge \frac{\alpha}{2} \ge \frac{\alpha^2}{2}.
\end{align*}

\paragraph{Case (ii) $1 > H^{-1}(\alpha/2)$:}
Simply plugging in $b = H^{-1}(\alpha/2)$ yields
\begin{align*}
    \int_{0}^b H(\tau) d\tau &\ge \frac{(H(0) - H(b))^2}{2\rho B}\\
    &\ge \frac{1}{2\rho B}\left(\frac{\alpha}{2}\right)^2 = \frac{\alpha^2}{8 \rho B}.
\end{align*}
\end{proof}

Lemma~\ref{lem:quantile-error-integral-bound} gives us the final result:
\begin{align*}
    &\E_{(x,s) \sim \cD^p}[\pb_q(p, s) - \pb_q(p + b\cdot f(x), s)] = \int_{0}^b H(\tau) d\tau \ge \frac{\alpha^2}{\max(2, 8\rho B)}.
 \end{align*}
\end{proof}

\begin{corollary}
\label{cor:swap-regret-multival-error-expect}
Fix any $B > 0$, $\tpi_{1:T}$, and $\cF$ that is closed under affine transformation. Suppose for every $t \in [T]$, $\ts_t$ is $\rho$-Lipschitz and continuous, and assume the conformal learner's contextual swap regret for pinball loss with respect to any $\{f_p\}_{p \in \cPq} \in \cF_{(1+\sqrt{B})^2}^m$ is bounded in expectation: for any $\{f_p\}_{p \in \cPq} \in \cF_{(1+\sqrt{B})^2}^m$,
\[
    \frac{1}{T} \sum_{p \in \cPq} \sum_{t \in S(\tpi_{1:T}, p)} \E_{s_t \sim \ts_t}[PB_q(\hp_t, s_t) - PB_q(f_p(x_t), s_t)] \le \frac{\alpha}{\max(2B, 8\rho B^2)}.
\] Then its $L_2$-swap-multivalidity error with respect to $\cF_B$ is bounded as 
\[
    \overline{sQ}_2(\tpi_{1:T}, \cF_B) \le \alpha.
\]
\end{corollary}
\begin{proof}
For the sake of contradiction, suppose there exists $\{f_p\}_{p \in \cPq} \in \cF_B^m$ such that its $L_2$-multivalidity error is more than $\alpha$. For each $p \in \cPq$, write
    \[
        \alpha_ p =  \frac{n(p)}{T} Q(\tpi_{1:T}, p , f_p)^2.
    \]
We have
\[
    Q(\tpi_{1:T}, p, f^*_p) = \sqrt{\frac{T}{n(p)} \alpha_p}
\]
where $f^*_p = f_p$ or $-f_p$. Note that 
\begin{align*}
    Q(\tpi_{1:T}, p, f^*_p) = \frac{1}{n(p)}\sum_{t \in S(p)} f(x_t) \cdot (q - \Pr_{s \sim \ts_t}[s \le p]) \le \sqrt{B}
\end{align*}
In other words, $Q(\tpi_{1:T}, p, f^*_p) \le \frac{1}{\sqrt{B}}\sqrt{\frac{T}{n(p)} \alpha_p}$ and $\frac{1}{\sqrt{B}}\sqrt{\frac{T}{n(p)} \alpha_p} \le 1$.

Lemma~\ref{lem:quantile-error-implies-pinball-loss-error} then tells us that for each $p \in \cPq$, there exists $f'_p \in \cF_{(1+\sqrt{B})^2}$ such that 
\[
    \frac{1}{n(p)}\sum_{t \in S(\tpi_{1:T}, p)} \E_{s_t \sim \ts_t}\left[PB(\hp_t, s_t) - PB(f'_p(x_t), s_t)\right] \ge \frac{1}{\max(2, 8\rho B)} \frac{1}{B} \frac{T}{n(p)} \alpha_p
\]
Averaging over each $p \in \cPq$ gives us 
\begin{align*}
    &\sum_{p \in \cPq} \sum_{t \in S(\tpi_{1:T}, p)} \E_{s_t \sim \ts_t}\left[PB(\hp_t, s_t) - PB(f'_p(x_t), s_t)\right]\\
    &\ge \sum_{p \in \cPq} \frac{1}{\max(2, 8 \rho B)} \frac{1}{B} \frac{T}{n(p)} \alpha_p\\
    &> \frac{ \alpha T}{\max(2B, 8 \rho B^2)} 
\end{align*}
This is a contradiction to our contextual-swap-regret guarantee with respect to pinball loss.
\end{proof}

As noted in Section~\ref{sec:conformal-extension}, there was nothing specific about the squared loss in Theorem~\ref{thm:forecast-swap-context-regret} and it can be generalized to any other loss as well. We state this more formally in the following claim:
\begin{claim}
\label{claim:conformal-context-swap-oracle}
Fix some $B>0$. Suppose the sequential shattering dimension of $\cF_B$ is finite at any scale $\delta$: $\fat_{\delta}(\cF_B) < \infty$. Suppose oracle $\cA$'s regret bound $r_\cA(T, \cF_B)$ is concave in time horizon $T$. Fix any adversary $\adv$ that forms $(x_t, y_t)$ as a function of $\psi_{1:t-1}= \{(x_\tau, s_\tau, \hp_\tau)\}_{\tau=1}^{t-1}$. With probability $1-\delta$ over the randomness of $\{ i_t \sim \ti_t \}_{t=1}^T$, Forecaster $\forecastSwap(\cF_B, \cA, m, T)$ results in $\pi_{1:T} = \{(x_t, s_t, \hp_t, (i_t, j_t))\}_{t=1}^T$ such that
\begin{align*}
    &\sup_{\{f_p\}_{p \in \cP} \in \cF_B^m}  \sum_{t=1}^T \E_{\hp'_t}\left[PB_q(\hp'_t/m, s_t) -PB_q(f_{\hp'_t}(x_t), s_t)|\pi_{1:t-1} \right] \\
    &\le \left(m r_\cA\left(\frac{T}{m}, \cF_B \right) + \frac{T}{m} + m\right) + \max(8B, 2\sqrt{B})m C_{\cF_B} \sqrt{\frac{\log(\frac{4m}{\delta})}{T}}.
\end{align*}
\end{claim}
For the sake of clarity and brevity, we skip the proof for the above claim or have a separate meta theorem from which both Theorem~\ref{thm:forecast-swap-context-regret} and the above claim can be obtained as corollaries. 

\subsection{Final Bounds}
Combining concentration bounds in Appendix~\ref{app:concentration} and Corollary~\ref{cor:swap-regret-multival-error-expect}, we can finally show that the connection between the contextual swap regret with pinball loss and the quantile $\ell_2$-multivalidity error:
\begin{corollary}
Fix any $B > 0$, $\tpi_{1:T}$, and $\cF$ that is closed under affine transformation. Suppose for every $t \in [T]$, $\ts_t$ is $\rho$-Lipschitz and continuous and that the sequential shattering dimension of $\cF_B$ is finite at any scale $\delta$: $\fat_{\delta}(\cF_B) < \infty$. With probability $1-\rho$, the forecaster's $L_2$-swap-multivalidity error with respect to $\cF_B$ is bounded as 
\begin{align*}
        &\overline{sQ}_2(\pi_{1:T}, \cF_B) \\
        &\le \max(2B, 8\rho B^2)\left(\frac{1}{T}\sum_{p \in \cPq} \sum_{t \in S(p)} \pb_q(\hp_t, s_t) - \pb_q(f_p(x_t), s_t) 
        + \frac{m B' C'_{\cF_B}\sqrt{\log(8m/\rho)}}{\sqrt{T}}\right)\\
        & +  \frac{2 m \sqrt{B} B' C_{\cF_B}\sqrt{\log(\frac{8m}{\rho})}}{\sqrt{T}} + \frac{m(B' C_{\cF_B})^2 \log(\frac{8m}{\rho})}{T}
\end{align*}
where $(C_{\cF_B}, C'_{\cF_B})$ are the same as in Lemma~\ref{lem:quant-error-concentrate} and \ref{lem:pinball-concentrate}. 
\end{corollary}
\begin{proof}
Write $B' = \max(1, \sqrt{B})$. By union-bounding Lemma~\ref{lem:pinball-concentrate} over all $p \in \cPq$, we have that with probability $1-4m\rho$, for any $\{f_p\}_{p \in \cPq} \in \cF_{(1+\sqrt{B})^2}^m$, 
    \begin{align*}
        &\frac{1}{T} \sum_{p \in \cPq} \sum_{t \in S(p)} \E_{s_t \sim \ts_t}\left[\pb_q(\hp_t, s_t) - \pb_q(f_p(x_t), s_t)\right]\\
        &= \sum_{p \in \cPq} \frac{n(p)}{T} \left(\frac{1}{n(p)}\sum_{t \in S(p)} \E_{s_t \sim \ts_t}\left[\pb_q(\hp_t, s_t) - \pb_q(f_p(x_t), s_t)\right]\right)\\
        &\le \sum_{p \in \cPq} \frac{n(p)}{T} \left(\frac{1}{n(p)}\sum_{t \in S(p)} \pb_q(\hp_t, s_t) - \pb_q(f_p(x_t), s_t) + B' C'_{\cF_B} \sqrt{\frac{\log(\frac{1}{\rho})}{n(p)}}\right) \\
        &= \frac{1}{T}\sum_{p \in \cPq} \sum_{t \in S(p)} \pb_q(\hp_t, s_t) - \pb_q(f_p(x_t), s_t) + \frac{B' C'_{\cF_B}\sqrt{\log(1/\rho)}}{T} \sum_{p \in \cPq}  \sqrt{n(p)} \\
        &\le \frac{1}{T}\sum_{p \in \cPq} \sum_{t \in S(p)} \pb_q(\hp_t, s_t) - \pb_q(f_p(x_t), s_t) + \frac{m B' C'_{\cF_B}\sqrt{\log(1/\rho)}}{\sqrt{T}}. 
    \end{align*}

    Similarly, we have with probability $1-4m\rho$
    \begin{align*}
        &\overline{sQ}_2(\pi_{1:T}, \cF_B)\\
        &= \sup_{\{f_p\}_{p \in \cPq} \in \cF_B^m} \sum_{p \in \cPq} \frac{n(p)}{T} (Q(\pi_{1:T}, p, f_p))^2\\
        &\le \sup_{\{f_p\}_{p \in \cPq} \in \cF_B^m} \sum_{p \in \cPq} \frac{n(p)}{T} \left(Q(\tpi_{1:T}, p, f_p) + B' C_{\cF_B}  \sqrt{\frac{\log(\frac{1}{\rho})}{n(p)}} \right)^2\\
        &= \sup_{\{f_p\}_{p \in \cPq} \in \cF_B^m} \sum_{p \in \cPq} \frac{n(p)}{T} \left(Q(\tpi_{1:T}, p, f_p)^2 + 2 Q(\tpi_{1:T}, p, f_p) B' C_{\cF_B}  \sqrt{\frac{\log(\frac{1}{\rho})}{n(p)}} + (B' C_{\cF_B})^2  \frac{\log(\frac{1}{\rho})}{n(p)}\right) \\
        &\le \sup_{\{f_p\}_{p \in \cPq} \in \cF_B^m} \sum_{p \in \cPq} \frac{n(p)}{T} \left(Q(\tpi_{1:T}, p, f_p)^2 + 2 \sqrt{B'} B' C_{\cF_B}  \sqrt{\frac{\log(\frac{1}{\rho})}{n(p)}} + (B' C_{\cF_B})^2  \frac{\log(\frac{1}{\rho})}{n(p)}\right) \\
        &\le \overline{sQ}_2(\tpi_{1:T}, \cF_B) + \frac{2 m \sqrt{B'} B' C_{\cF_B}\sqrt{\log(\frac{1}{\rho})}}{\sqrt{T}} + \frac{m(B' C_{\cF_B})^2 \log(\frac{1}{\rho})}{T}.
    \end{align*}
    The second to last inequality follows from $f_p \in \cF_B$, so $Q(\tpi_{1:T}, p, f_p) \le \sqrt{B}$. The last inequality follows from $\sum_{p \in \cPq} \sqrt{n(p)} \le m\sqrt{T}$.

    Combining them together via Corollary~\ref{cor:swap-regret-multival-error-expect}, we have with probability $1-8m\rho$,
    \begin{align*}
        &\overline{sQ}_2(\pi_{1:T}, \cF_B) \\
        &\le \max(2B, 8\rho B^2)\left(\frac{1}{T}\sum_{p \in \cPq} \sum_{t \in S(p)} \pb_q(\hp_t, s_t) - \pb_q(f_p(x_t), s_t) 
        + \frac{m B' C'_{\cF_B}\sqrt{\log(1/\rho)}}{\sqrt{T}}\right)\\
        & +  \frac{2 m \sqrt{B} B' C_{\cF_B}\sqrt{\log(\frac{1}{\rho})}}{\sqrt{T}} + \frac{m(B' C_{\cF_B})^2 \log(\frac{1}{\rho})}{T}
    \end{align*}
\end{proof}

\cite{rakhlin2015online} shows that any function that is online learnable has a finite sequential fat shattering dimension at any $\alpha$. As we are already assuming an online quantile regression oracle for $\cA$, we can assume that $\cF$'s sequential fat shattering dimension is finite at any $\alpha$. 

When we plug the above bound to the contextual swap regret guarantee of $\forecastSwap$ in Claim~\ref{claim:conformal-context-swap-oracle}, we get the following $\ell_2$-multivalidity error rate:
\begin{corollary}
\label{cor:pinball-multivalid-reduction}
Fix $\cF$ and suppose its fat shattering dimension is finite at any scale $\alpha$ (e.g. it is online learnable).  Suppose for every $t \in [T]$, $\ts_t$ is $\rho$-Lipschitz and continuous. Then, the conformal learner that makes predictions according to $\forecastSwap(\cF_{(1+\sqrt{B})^2}, \cA, m, T)$ will guarantee that with probability $1-\delta$ 
\begin{align*}
        &\overline{sQ}_2(\pi_{1:T}, \cF_B) \\
        &\le \max(2B, 8\rho B^2)\bigg(\frac{1}{T}\left(m r_\cA\left(\frac{T}{m}, \cF_B \right) + \frac{T}{m} + m\right) \\
        &\;\;\;\;\;\;\;\;\;\;\;\;\;\;\;\;\;\;\;\;\;\;\;\;\;\;\;\;\;\;\;\;\;+ \max(8B, 2\sqrt{B})m C_{\cF_B} \sqrt{\frac{\log(\frac{16m}{\delta})}{T}} 
        + \frac{m B' C'_{\cF_B}\sqrt{\log(16m/\delta)}}{\sqrt{T}}\bigg)\\
        &\;\;\;\; +  \frac{2 m \sqrt{B} B' C_{\cF_B}\sqrt{\log(\frac{16m}{\delta})}}{\sqrt{T}} + \frac{m(B' C_{\cF_B})^2 \log(\frac{16m}{\delta})}{T}.
    \end{align*}
\end{corollary}

In other words, for oracles whose regret is $\sqrt{T}$, setting $m=\frac{1}{T^{1/4}}$ yields $\overline{sQ}_2(\pi_{1:T}, \cF_B) = O(T^{-1/4})$.

\subsection{Conformal Learner's Multivalidity Error Against Linear Functions}
\begin{theorem}
\label{thm:pgd-pinball}
Suppose $||x_t||_2 \le 1$ for every round $t \in [T]$. Online projected gradient descent  $\cA_{PGD}$ is defined as
\[\theta_{t+1} = \text{Project}_{||\theta_t|| \le A}\left(\theta_{t} - \eta \cdot \nabla \ell_t(\theta_t) \right)\] where
$\ell_t(\theta) = \pb(s_t, \theta \cdot x_t)$. $\cA_{PGD}$ guarantees that 
\[
\sum_{t=1}^T \pb_q(s_t, \theta_t \cdot x_t) - \min_{\theta^*: ||\theta^*||\le B}\sum_{t=1}^T \pb_q(s_t, \theta^* \cdot x_t) \le \frac{B^2}{2\eta} + \frac{\eta T}{2}.
\]
Setting $\eta=\frac{B}{\sqrt{T}}$ gives us
\[
    \sum_{t=1}^T \pb_q(s_t, \theta_t \cdot x_t) - \min_{\theta^*: ||\theta^*||\le B}\sum_{t=1}^T \pb_q(s_t, \theta^* \cdot x_t) \le  B \sqrt{T}.
\]
\end{theorem}
\begin{proof}
It's easy to see that the loss function $\ell_t(\theta)$ is convex in $\theta$. Then, following \cite{zinkevich2003online}, we get
\[
\sum_{t=1}^T \pb(s_t, \theta_t \cdot x_t) - \min_{\theta^*: ||\theta^*||\le A}\sum_{t=1}^T \pb(s_t, \theta^* \cdot x_t) \le \frac{A}{2\eta} + \frac{\eta}{2} \sum_{t=1}^T ||\nabla \ell_t(\theta_t)||^2_2.
\]  

So it suffices to bound $||\nabla \ell_t(\theta_t)||^2_2$ for every $t \in [T]$:
\begin{align*}
    ||\nabla \ell_t(\theta_t)||^2_2 &= || \nabla \pb_q(s_t, \theta \cdot x_t)||^2_2\\
    &\le \max(q, 1-q) ||x_t||^2_2\\
    &\le 1
\end{align*}
\end{proof}

\begin{corollary}
Suppose the adversary creates feature vectors such that $||x_t||_2 \le 1$ and . Let $\cF^{lin} = \{\langle \theta, x\rangle: \theta \in \R^d \}$ be all linear functions and for every $t \in [T]$, $\ts_t$ is $\rho$-Lipschitz and continuous. Then, $\forecastSwap(\cF^{lin}_{(1+\sqrt{B})^2}, \cA_{PGD}, m, T)$ and setting $m=T^{1/4}$ guarantees that 
\[
    \overline{sQ}_2(\pi_{1:T}, \cF_B) \le  O\left(\sqrt{\ln\left(\frac{T^{1/4}}{\delta}\right)} T^{-1/4}\right).
\]
\end{corollary}
\begin{proof}
Because $||x_t||_2 \le 1$, all linear functions whose range is bounded corresponds exactly to all linear functions whose norm of the slope is bounded:
\[
    \cF^{lin}_{(1+\sqrt{B}^2)} = \{ \langle \theta, x \rangle: ||\theta||^2_2 \le (1 + \sqrt{B})^2\}.
\]

 Theorem~\ref{thm:pgd-pinball} tells us that $\cA_{PGD}$ with the learning rate set to be $\eta=\frac{1+\sqrt{B}}{\sqrt{T}}$ achieves \[
    \sum_{t=1}^T \pb_q(s_t, \hp_t) - \min_{f \in \cF^{lin}_{(1+\sqrt{B})^2}}\sum_{t=1}^T \pb_q(s_t, f(x_t)) \le (1+\sqrt{B})\sqrt{T}
\]
where $\hp_t = \langle \theta_t, x_t \rangle$. In other words, plugging in $r_\cA(T) = (1+\sqrt{B})\sqrt{T}$ into Corollary~\ref{cor:pinball-multivalid-reduction} gives us that with probability $1-\delta$, \[
\overline{sQ}_2(\pi_{1:T}, \cF_B) \le  O\left(\rho \sqrt{\ln\left(\frac{T^{1/4}}{\delta}\right)} T^{-1/4}\right).
\]
\end{proof}

\section{Better Online $L_2$-multicalibration for Finite $\cF$}
\label{sec:AMF}
In this section, we show how to leverage the Online Minimax Optimization Framework introduced by \cite{noarov2021online} to obtain $L_2$-multicalibration bounds for finite classes $\cF$ that obtain multicalibration error tending to zero at the optimal $\tilde O\left(\frac{1}{\sqrt{T}}\right)$ rate. The algorithm has running time scaling polynomially with $|\cF|$, and so is not oracle efficient. %

Recall that we introduced the AMF framework in Section \ref{sec:vcall}; we do not repeat the preliminaries here.

Now, we describe how to cast our problem into the AMF framework. We refer to the Learner and the Adversary in the AMF framework as AMF Learner and AMF adversary.

Fix some finite set of predictors $\cF$. At each round $t \in [T]$, the AMF Learner's strategy space is $\Theta_t =\cP=[\frac{1}{m}]$, the AMF Adversary's strategy space is $\cZ_t = \Delta(\cY)$, and the dimension of the loss vectors we will construct is $d =  |\cF|$. For convenience, we equivalently write the adversary's strategy space as $\cZ_t = [0,1]$ and write $z_t$ for the probability with which $y \sim z_t$ is 1. Once the forecaster observes the feature vector $x_t$, we construct the loss vector at round $t$ as follows: For each $f \in \cF$, we assign $f$ a coordinate $i$ in the loss vector and define $\ell^i_t$ as: 
\begin{align}
        \ell^i_t(\theta_t, z_t) = \ell^{f}_t(\theta_t, z_t) = \E_{y_t \sim z_t}\left[ (f(x_t)\cdot(y_t - \theta_t)) K(\pi_{1:t-1}, \theta_t, f)\right]\label{eqn:amf-loss}
\end{align}

The loss function $\ell$ defined above in \eqref{eqn:amf-loss} is defined in terms of an upper bound on the increase in the unnormalized $L_2$-multicalibration error that we show in Lemma~\ref{lem:bound-increase-multcal-error} below, whose proof can be found in Appendix~\ref{app:miss-amf}. For convenience, to denote the unnormalized version of the $L2$-multicalibration error, we write.  
\[
    u\bK_2(\pi_{1:T}, f) = T \cdot \bK_2(\pi_{1:T}) 
\]
\begin{restatable}{lemma}{lemboundincreasemultcalerror}
\label{lem:bound-increase-multcal-error}
Suppose $B = \max_{f \in \cF, x \in \cX} (f(x))^2$. Fix $t \in [T]$, $\pi_{1:t-1}$, $(x_t, y_t, \hp)$ and $f \in \cF$.   We can bound the increase in the unnormalized $L_2$-multicalibration error for $f$ in the case in which the interaction between the Forecaster and the Adversary results in $(x_t, y_t, \hp)$ in round $t$ as follows: if $n(\pi_{1:t-1}, \hp) \ge 1$, then
    \begin{align*}
        u\bK_2(\pi_{1:t-1} \circ (x_t, y_t, \hp), f) - u\bK_2(\pi_{1:t-1}, f)\le \frac{B}{n(\pi_{1:t-1}, \hp)} + 2(f(x_t)\cdot(y_t - \hp)) K(\pi_{1:t-1}, \hp, f)
    \end{align*}
    Otherwise, \[
    u\bK_2(\pi_{1:t-1} \circ (x_t, y_t, \hp), f) - u\bK_2(\pi_{1:t-1}, f) \le B
    \]
\end{restatable}
\begin{proof}
Fix any $f \in \cF$. The latter case when $n(\pi_{1:t-1}) = 0$, it's easy to see that \[
    u\bK_2(\pi_{1:t-1} \circ (x_t, y_t, \hp), f) - u\bK_2(\pi_{1:t-1}, f) = (f(x_t) \cdot (y_t - \hp))^2 \le B.
\]

Now consider the other case. 
    \begin{align*}
        &u\bK_2(\pi_{1:t-1} \circ (x_t, y_t, \hp), f) - u\bK_2(\pi_{1:t-1}, f) \\
        &= (n(\pi_{1:t-1}, \hp) + 1) \cdot (K(\pi_{1:t-1} \circ (x_t, y_t, \hp), \hp, f))^2 - n(\pi_{1:t-1}, \hp) \cdot(K(\pi_{1:t-1}, \hp, f))^2 \\
        &= \frac{1}{n(\pi_{1:t-1}, \hp) + 1}\left(f(x_t) \cdot (y_t - \hp) + \sum_{\tau \in S(\pi_{1:t-1}, \hp)} f(x_\tau) \cdot (y_\tau - \hp_\tau) \right)^2 \\
        &- \frac{1}{n(\pi_{1:t-1}, \hp)}\left(\sum_{\tau \in S(\pi_{1:t-1}, \hp)} f(x_\tau) \cdot (y_\tau - \hp_\tau) \right)^2\\
        &\le \frac{1}{n(\pi_{1:t-1}, \hp)}\left(f(x_t) \cdot (y_t - \hp) + \sum_{\tau \in S(\pi_{1:t-1}, \hp)} f(x_\tau) \cdot (y_\tau - \hp_\tau) \right)^2 \\
        &- \frac{1}{n(\pi_{1:t-1}, \hp)}\left(\sum_{\tau \in S(\pi_{1:t-1}, \hp)} f(x_\tau) \cdot (y_\tau - \hp_\tau) \right)^2\\ 
        &= \frac{1}{n(\pi_{1:t-1}, \hp)}\left((f(x_t)\cdot(y_t - \hp))^2 + 2 (f(x_t) \cdot(y_t - \hp)) \sum_{t \in S(\pi_{1:t-1},\hp)   } f_\tau(x_\tau) \cdot (y_\tau - \hp_\tau) \right)\\
        &\le \frac{B}{n(\pi_{1:t-1}, \hp)} + 2(f(x_t)\cdot(y_t - \hp)) K(\pi_{1:t-1}, \hp, f)
    \end{align*}
where the last inequality follows from the fact that $f(x_t)^2 \le B$.
\end{proof}

\begin{algorithm}[H]
\begin{algorithmic}[1]
\FOR{$t=1, \dots, T$}
    \STATE Forecaster $\forecast_{AMF}$ observes the feature vector $x_t \in \cX$.
    \STATE Construct the AMF-Learner's environment as $\Theta_t = \cP$, $\cZ_t = \Delta(\cY)s$, and
    $\ell_t$ as defined in \eqref{eqn:amf-loss}.
    \STATE After observing the environment $(\Theta_t, \cZ_t, \ell_t)$, AMF-Learner chooses a random strategy $\tilde{\theta}_t$ according to Algorithm 2 of \cite{noarov2021online}.
    \STATE Adversary $\adv$ chooses $p_t \in [0,1]$ which is the the probability with which $y_t=1$.
    \STATE AMF Learner chooses $\theta_t \in \cP$ by sampling from $\tilde{\theta}_t \in \Delta(\cP)$.
    \STATE Forecaster $\forecast_{AMF}$ makes the forecast $\hp_t=\theta_t$ and observes $y_t \sim p_t$.
    \STATE Set the action played by the AMF Adversary's strategy as $z_t = y_t$ and have AMF Learner suffer $\ell_t(\theta_t, z_t)$.
\ENDFOR
\end{algorithmic}
\caption{Forecaster $\forecast_{AMF}(\cF)$}
\label{alg:amf-reduction}
\end{algorithm}

\begin{theorem}
\label{thm:amfmultical}
Suppose $\max_{f \in \cF, x \in \cX}(f(x))^2 \le B$ and $|\cF| < \infty$. Against any adaptive adversary, the $L_2$-multicalibration error of Forecaster $\cF_{AMF}(\cF)$ is always bounded in expectation over the forecaster's randomness as:
\[
    \E_{\pi_{1:T}}[\bK_2(\pi_{1:T}, \cF_B)] \le \frac{1}{T} \left(3Bm\log(T) + 4\sqrt{BT \ln(|\cF|)} + \frac{\sqrt{B}T}{m}\right).
\]
Setting $m = \sqrt{T}$, we have
\[
    \E_{\pi_{1:T}}[\bK_2(\pi_{1:T}, \cF_B)] \le \frac{3B \log(T) + 4\sqrt{B \ln(|\cF|)} + \sqrt{B}}{\sqrt{T}} = \tilde{O}\left(\frac{1}{T^{1/2}}\right).
\]
\end{theorem}
\begin{proof}
    First, note that the constructed loss function $\ell^{f}_t$ is linear in terms of the adversary's strategy $z_t$ because of the linearity of expectation. Also, note that \begin{align*}
        |\ell^f_t(\theta_t, z_t)| &\le |(f(x_t)\cdot(y_t - \hp)) K(\pi_{1:t-1}, \hp_t, f) |\\
        &\le |f(x_t)| |y_t - \hp_t| |K(\pi_{1:t-1}, \hp_t, f)|\\
        &\le \sqrt{B} 
    \end{align*}
    as $f(x)^2 \le B$.

    We can calculate the Adversary-Moves-First value. Whenever the adversary chooses $z_t \in [0,1]$ which corresponds to the probability of $y_t$ is 1, we can always deterministically set $\theta_t = \round(z_t, \frac{1}{m})$ and suffer $\frac{1}{m}$ discretization error.
    \begin{align*}
        w^A_t &= \sup_{z_t \in \Delta(\cY)}\min_{\tilde{\theta}_t \in \Delta(\cP)} \max_{f \in \cF} \E_{\theta_t \sim \tilde{\theta}_t, y_t \sim z_t}\left[\ell^f_t(\theta_t, y_t)\right]\\
        &\le \frac{\sqrt{B}}{m}.
    \end{align*}

Finally, it remains to upper bound our cumulative calibration error using Lemma \ref{lem:bound-increase-multcal-error} to telescope over our per-round \emph{increase} in multicalibration error, and use the upper bound on the algorithm's AMF regret from Theorem \ref{thm:minmax-game-reget} to bound the sum of these terms in the worst case over $f \in \cF_B$.
    
    \begin{align*}
        &\E_{\pi_{1:T}}\left[\max_{f \in \cF} u\bK_2(\pi_{1:T}, f)\right] \\
        &= \E_{\pi_{1:T}}\left[\max_{f \in \cF}\sum_{t=1}^T u\bK_2(\pi_{1:t}, f) - u\bK_2(\pi_{1:t-1}, f) \right]\\
        &\le \E_{\pi_{1:T}}\left[\max_{f \in \cF}\sum_{t \in [T]: n(\pi_{1:t-1}, \hp_t) = 0} B + \sum_{t \in [T]: n(\pi_{1:t-1}, \hp_t) \ge 1} \frac{B}{n(\pi_{1:t-1}, \hp_t)} + 2\ell^f_t(\theta_t, z_t)\right] \quad&\text{Lemma~\ref{lem:bound-increase-multcal-error}}\\
        &\le \E_{\pi_{1:T}}\left[Bm + \max_{f \in \cF} \sum_{t \in [T]: n(\pi_{1:t-1}, \hp_t) \ge 1} \frac{B}{n(\pi_{1:t-1}, \hp_t)} + 2\ell^f_t(\theta_t, z_t)\right]\\
        &\le \E_{\pi_{1:T}}\left[ Bm + \sum_{t \in [T]: n(\pi_{1:t-1}, \hp_t) \ge 1} \frac{B}{n(\pi_{1:t-1}, \hp_t)} +  2\max_{f \in \cF}\sum_{t=1}^T\ell^{f}_t(\theta_t, z_t)\right] \\
        &\le Bm + 2Bm\log(T) + \E_{\pi_{1:T}}\left[R_T + \sum_{t=1}^T w^A_t \right]  \\
        &\le 3Bm\log(T) + 4\sqrt{B T \ln(|\cF|)} + \frac{T\sqrt{B}}{m} \quad&\text{Theorem~\ref{thm:minmax-game-reget}}
    \end{align*}
    The third inequality follows because when $n(\pi_{1:t-1, \hp_t})=0$ for round $t \in [T]$, then $\ell^f_t(\theta_t, z_t) = 0$ as $K(\pi_{1:t-1}, p, f) = 0$. The fourth inequality follows because for any $\pi_{1:T}$
    \begin{align*}
        \sum_{t \in [T]: n(\pi_{1:t-1}, \hp_t) \ge 1} \frac{1}{n(\pi_{1:t-1}, \hp_t)} &= \sum_{p \in \cP} \sum_{\tau=1}^{|n(\pi_{1:T}, p)|} \frac{1}{\tau}\\
        &\le \sum_{p \in \cP} \sum_{\tau=1}^T \frac{1}{\tau}\\
        &\le 2m \log(T).
    \end{align*} 
\end{proof}

We remark that this bound is instantiated by Algorithm \ref{alg:amf-alg}, which needs at each round to compute a numeric value $\chi_t^j$ for each $f \in \cF$ (each of which is a sum over at most $T$ terms) and to solve a minmax problem which can be expressed as a linear program with $m=\sqrt{T}$ variables and $|\cF|$ many constraints. Hence the running time of the algorithm is polynomial in $T$ and $|\cF|$.  

\end{appendices}
\end{document}